\documentclass[jmlr]{article}

\usepackage{microtype}
\usepackage{graphicx}
\usepackage{subcaption}
\usepackage{booktabs} %
\usepackage{multicol}
\usepackage[dvipsnames]{xcolor}
\usepackage{listings}
\usepackage{algorithm}
\usepackage{algorithmic}

\usepackage[export]{adjustbox}
\usepackage{jmlr2e}
\usepackage[numbers]{natbib}

\usepackage{xcolor}
\usepackage[bottom=0.85in,top=0.85in,left=0.85in,right=0.85in,footskip=.25in]{geometry}

\definecolor{codegreen}{rgb}{0,0.6,0}
\definecolor{codegray}{rgb}{0.5,0.5,0.5}
\definecolor{codepurple}{rgb}{0.58,0,0.82}
\definecolor{backcolour}{rgb}{0.95,0.95,0.92}

\lstdefinestyle{mystyle}{
    backgroundcolor=\color{backcolour},   
    commentstyle=\color{codegreen},
    keywordstyle=\color{magenta},
    numberstyle=\tiny\color{codegray},
    stringstyle=\color{codepurple},
    basicstyle=\ttfamily\small,
    breakatwhitespace=false,         
    breaklines=true,                 
    captionpos=b,                    
    keepspaces=true,                 
    numbers=left,                    
    numbersep=5pt,                  
    showspaces=false,                
    showstringspaces=false,
    showtabs=false,                  
    tabsize=2
}

\usepackage{hyperref}
\usepackage{amsmath}
\usepackage{amssymb}
\usepackage{mathtools}
\usepackage{natbib}
\usepackage{enumerate}
\usepackage{tikz}
\usepackage{graphicx}

\newcommand{\textdiff}[1]{\textcolor{red}{#1}}

\newcommand{\defeq}{\mathrel{\stackrel{\makebox[0pt]{\mbox{\normalfont\tiny def}}}{=}}}

\newcommand{\argmin}[1]{\underset{#1}{\textrm{argmin}}\ }

\newcommand{\Projmu}{\Pi_\mu}
\newcommand{\trans}{T}

\newcommand{\Qclass}{\mathcal{Q}}

\newcommand{\valuefeedback}{\mathrm{V_{CF}}}

\newcommand{\bellmanopt}{\mathcal{B}^*}
\newcommand{\expec}{\mathbb{E}}

\newcommand{\kldiv}{\mathrm{D_{KL}}}
\newcommand{\tvd}{\mathrm{D_{TV}}}
\newcommand{\diag}{\mathrm{Diag}}

\usepackage{semicrunch}
\usepackage{wrapfig,lipsum,booktabs}

\newtheorem{theorem}{Theorem}[section]
\newtheorem{assumption}{Assumption}[section]

\newtheorem{definition}{Definition}[section]
\newtheorem{corollary}{Corollary}[theorem]
\newtheorem{lemma}{Lemma}[theorem]

\usepackage{sidecap}
\usepackage{minitoc}

\usepackage{tikz}
\usetikzlibrary{arrows}

\tikzset{
  treenode/.style = {align=center, inner sep=0pt, text centered,
    font=\sffamily},
  arn_n/.style = {treenode, circle, white, font=\sffamily\bfseries, draw=black,
    fill=black, text width=1.5em},%
  arn_r/.style = {treenode, circle, grey, draw=grey, 
    text width=1.5em, very thick},%
  arn_x/.style = {treenode, rectangle, draw=black,
    minimum width=0.8em, minimum height=0.8em}%
}

\firstpageno{1}

\author{\name Aviral Kumar\thanks{Corresponding Author; Preprint. Under Review.} \email {aviralk@berkeley.edu}\\
\name Abhishek Gupta \email {abhigupta@berkeley.edu} \\
\name Sergey Levine \email {svlevine@eecs.berkeley.edu} \\
\addr Electrial Engineering and Computer Sciences, University of California, Berkeley}

\begin{document}
\setcitestyle{square}

\title{\textbf{DisCor: Corrective Feedback in Reinforcement Learning via Distribution Correction}}
\editor{}
\maketitle

\begin{abstract}
Deep reinforcement learning can learn effective policies for a wide range of tasks, but is notoriously difficult to use due to instability and sensitivity to hyperparameters. The reasons for this remain unclear. When using standard supervised methods (e.g., for bandits), on-policy data collection provides ``hard negatives'' that correct the model in precisely those states and actions that the policy is likely to visit. We call this phenomenon ``corrective feedback.'' We show that bootstrapping-based Q-learning algorithms do not necessarily benefit from this corrective feedback, and training on the experience collected by the algorithm is not sufficient to correct errors in the Q-function.
In fact, Q-learning and related methods can exhibit pathological interactions between the distribution of experience collected by the agent and the policy induced by training on that experience, leading to potential instability, sub-optimal convergence, and poor results when learning from noisy, sparse or delayed rewards. We demonstrate the existence of this problem, both theoretically and empirically. We then show that a specific correction to the data distribution can mitigate this issue. Based on these observations, we propose a new algorithm, DisCor, which computes an approximation to this optimal distribution and uses it to re-weight the transitions used for training, resulting in substantial improvements in a range of challenging RL settings, such as multi-task learning and learning from noisy reward signals. \textbf{Blog post presenting a summary of this work is available at:} \url{https://bair.berkeley.edu/blog/2020/03/16/discor/}.
\end{abstract}

\section{Introduction}

Reinforcement learning (RL) algorithms, when combined with high-capacity deep neural net function approximators, have shown promise in domains ranging from robotic manipulation~\citep{kalashnikov18} to recommender systems~\citep{shani2005recommender}. However, current deep RL methods can be difficult to use, due to sensitivity with respect to hyperparameters and inconsistent and unstable convergence. While a number of hypotheses have been proposed to understand these issues~\citep{hassalt10doubleq,Hasselt2018DeepRL,pmlr-v80-fujimoto18a,fu19diagnosing}, and gradual improvements have led to more powerful algorithms in recent years~\citep{Haarnoja18,rainbow}, an effective solution has proven elusive. We hypothesize that a major source of instability in reinforcement learning with function approximation and value function estimation, such as Q-learning~\citep{Watkins92,Riedmiller2005,Mnih2015} and actor-critic algorithms~\citep{Haarnoja2017,konda_ac}, is a pathological interaction between the data distribution induced by the latest policy, and the errors induced in the learned approximate value function as a consequence of training on this distribution, which then exacerbates the issues for the data distribution at the next iteration.

While a number of prior works have provided a theoretical examination of various approximation dynamic programming (ADP) methods -- which encompasses Q-learning and actor-critic algorithms --
prior work has not extensively studied the relationship between the data distribution induced by the latest value function, and the \textit{errors} in the future value functions obtained by training on this data.

When using supervised learning to train a model, as in the case of contextual bandits or model-based RL, using the resulting model to select the most optimal actions results in a kind of ``hard negative'' mining: the model collects precisely those transitions that lead to good outcomes according to the model (potentially erroneously). This results in collecting precisely the data needed to fix those errors and improve. We refer to this as ``corrective feedback.''
We argue that ADP algorithms (e.g., Q-learning and actor-critic), which use bootstrapped targets rather than ground-truth labels, often do not enjoy this sort of corrective feedback. Since they regress onto bootstrapped estimates of the current value function, rather than the true optimal value function (which is unknown), simply visiting states with high error and updating the value function at those states does not necessarily correct those errors, since errors in the target value might be due to upstream errors in other states that are visited less often. This absence of corrective feedback can result in severe detrimental consequences on performance.

We demonstrate that na\"ively training a value function on transitions collected either by the latest policy or a mixture of recent policies (i.e., with replay buffers) may not result in corrective feedback. In fact, in some cases, it can actually lead to increasing accumulation of errors, which can lead to poor performance even on simple tabular MDPs.
We then show how to address this issue by re-weighting the data buffer using a distribution that explicitly optimizes for corrective feedback, which gives rise to our proposed algorithm, \textbf{DisCor}.

The two main contributions of our work consist of an analysis showing that ADP methods may not benefit from corrective feedback, even with online data collection, as well as a proposed solution to this problem based on estimating target value error and resampling the replay buffer to mitigate error accumulation.
Our method, DisCor, is general and can be used in conjunction with most modern ADP-based deep RL algorithms, such as DQN~\citep{Mnih2015} and SAC~\citep{Haarnoja18}. Our experiments show that DisCor substantially improves performance of standard RL methods, especially in challenging settings, such as multi-task RL and learning from noisy rewards. We evaluate our approach on both continuous control tasks and discrete-action Atari games.
On the multi-task MT10 benchmark~\citep{yu2019meta} and several robotic manipulation tasks, our method learns policies with a final success rate that is 50\% higher than that of SAC.

\section{Preliminaries}
\label{sec:backrgound}
{The goal in reinforcement learning is to learn a policy that maximizes the expected cumulative discounted reward in a Markov decision process (MDP), which is defined by a tuple} $(\mathcal{S}, \mathcal{A}, \trans, R, \gamma)$. $\mathcal{S}, \mathcal{A}$ represent state and action spaces, $\trans(s' | s, a)$ and $r(s,a)$ represent the dynamics and reward function, and $\gamma \in (0,1)$ represents the discount factor. $\rho_0(s)$ is the initial state distribution. The infinite-horizon, {discounted marginal state distribution of the policy $\pi(a|s)$ is denoted as $d^{\pi}(s)$ and the  corresponding state-action marginal is $d^{\pi} (s, a) = d^{\pi}(s) \pi(a|s)$.}

Approximate dynamic programming (ADP) algorithms, such as Q-learning and actor-critic methods, aim to acquire the optimal policy by modeling the optimal state ($V^*(s)$) and state-action ($Q^*(s,a)$) value functions. These algorithms are based on recursively iterating the Bellman optimality operator, $\bellmanopt$, defined as
\begin{align*}
\vspace{-5pt}
(\bellmanopt Q)(s, a) &= r(s, a) + \gamma E_{s' \sim \trans}[\max_{a'} Q(s', a')] .
\end{align*}
The goal is to converge to the optimal value function, $Q^*$, by applying successive Bellman projections.
With function approximation, these algorithms project the values of the Bellman optimality operator $\bellmanopt$ onto a family of Q-function approximators $\Qclass$ (e.g., deep neural nets) under a sampling or data distribution $\mu$, such that $Q_{k+1} \leftarrow \Projmu(\bellmanopt Q_k)$ and
\begin{equation}
\vspace{-5pt}
\label{eqn:bellman_projection} 
\Projmu(Q) \defeq 
\argmin{Q' \in \Qclass} \expec_{s,a \sim \mu}[(Q'(s,a) - Q(s,a))^2].
\end{equation}
We refer to $\bellmanopt Q$ as the \textit{target value} for the projection step. 
Q-function fitting is usually interleaved with additional data collection, which typically uses a policy derived from the latest value function, augmented with either $\epsilon$-greedy~\citep{Mnih2015} or Boltzmann-style~\citep{Haarnoja18} exploration.
To simplify analysis, we mainly consider an underlying RL algorithm (Appendix~\ref{sec:omitted_proofs}, Algorithm~\ref{alg:fqi}) that alternates between fitting the action-value function, $Q(s, a)$, fully with the current data by minimizing Bellman error, and then collecting data with the policy derived from this value function. This corresponds to fitted Q-iteration.

For commonly used ADP methods, $\mu$ simply corresponds to the on-policy state-action marginal, $\mu_k = d^{\pi_k}$ (at iteration $k$) or else a ``replay buffer''~\citep{Haarnoja18,Mnih2015,Lillicrap2015} formed as a mixture distribution over all past policies, such that $\mu_k = \sum_{i=1}^k d^{\pi_i}$.
However, as we will show in this paper, the choice of the sampling distribution $\mu$ is of crucial importance for the stability and efficiency of ADP algorithms, and that many commonly-used choices of this distribution can lead to convergence to sub-optimal solutions, even with online data collection. We analyze this issue in Section~\ref{sec:problem_description},
and then discuss a potential solution to this problem in Section~\ref{sec:method_description}.

\paragraph{Experiment setup for analysis.} For the purposes of the analysis of corrective feedback in Section~\ref{sec:problem_description}, we use tabular MDPs from \cite{fu19diagnosing}, which provide the ability to measure oracle quantities, such as the error against $Q^*$, the ground truth optimal Q-function. We remove other sources of error, such as sampling error, by providing all transitions in the MDP to the ADP algorithm. 
We simulate different data distributions by re-weighting these transitions. We use a two hidden layer feed-forward network to represent the Q-function. More details on this setup are provided in Appendix~\ref{sec:app_exps_gridworld}.
\section{Corrective Feedback in Q-Learning}
\label{sec:problem_description}

In this paper, we study how the policies induced by value functions learned via ADP result in data distributions that can \emph{fail} to correct systematic errors in those value functions, in contrast to non-bootstrapped methods (e.g., supervised learning of models).
That is, ADP methods lack ``corrective feedback.'' To define corrective feedback intuitively and formally, we start with a contextual bandit example, where the goal is to learn the optimal state-action value function $Q^*(s, a)$ which, for a bandit, is equal to the reward $r(s, a)$ for performing action $a$ in state $s$. At iteration $k$, the algorithm minimizes the estimation error of the Q-function: \mbox{$\expec_{s \sim \beta(s), a \sim \pi_k(a|s)} [|Q_k(s, a) - Q^*(s, a)|]$}. Using a greedy or Boltzmann policy $\pi_k(a|s)$ to collect data for training $Q_{k+1}$ leads the agent to choose actions $a$ with over-optimistic $Q_k(s, a)$ values at a state $s$, and observe the corresponding true $Q^*(s, a)$ values as a result. Minimizing this estimation error then leads to the errors being \emph{corrected}, as $Q_k(s, a)$ is pushed closer to match the true $Q^*(s, a)$ for actions $a$ with incorrectly high Q-values. This constructive interaction between online data collection and error correction is what we refer to as \textbf{corrective feedback}.

In contrast, as we will observe shortly, ADP methods that employ bootstrapped target values do not necessarily benefit from such corrective feedback, and therefore can converge to suboptimal solutions. Because value functions are trained with target values computed by applying the Bellman backup on the previous \emph{learned} Q-function, rather than the true optimal $Q^*$, errors in the previous Q-function at the states that serve as backup targets can result in incorrect Q-value targets at the current state. In this case, no matter how often the current transition is observed, the error at this transition is not corrected. 
Since ADP algorithms typically use data collected using past Q-functions for training, thus coupling the data distribution to the learned value function, this issue can make them unable to correct target value errors.

As shown in Figure~\ref{fig:visitation_doesnt_correct_eror} (experiment setup described at the end of Section~\ref{sec:backrgound}), state visitation frequencies of the latest policy in an ADP algorithm can often correlate \emph{positively} with increasing error, suggesting that visiting a state can actually \textit{increase} error at that state, in contrast to supervised learning in bandit problems, where this correlation is either negative (i.e., the error decreases) or 0.

\begin{figure}[!h]
\centering
\begin{subfigure}{.47\columnwidth}
  \includegraphics[width=0.7\textwidth,right]{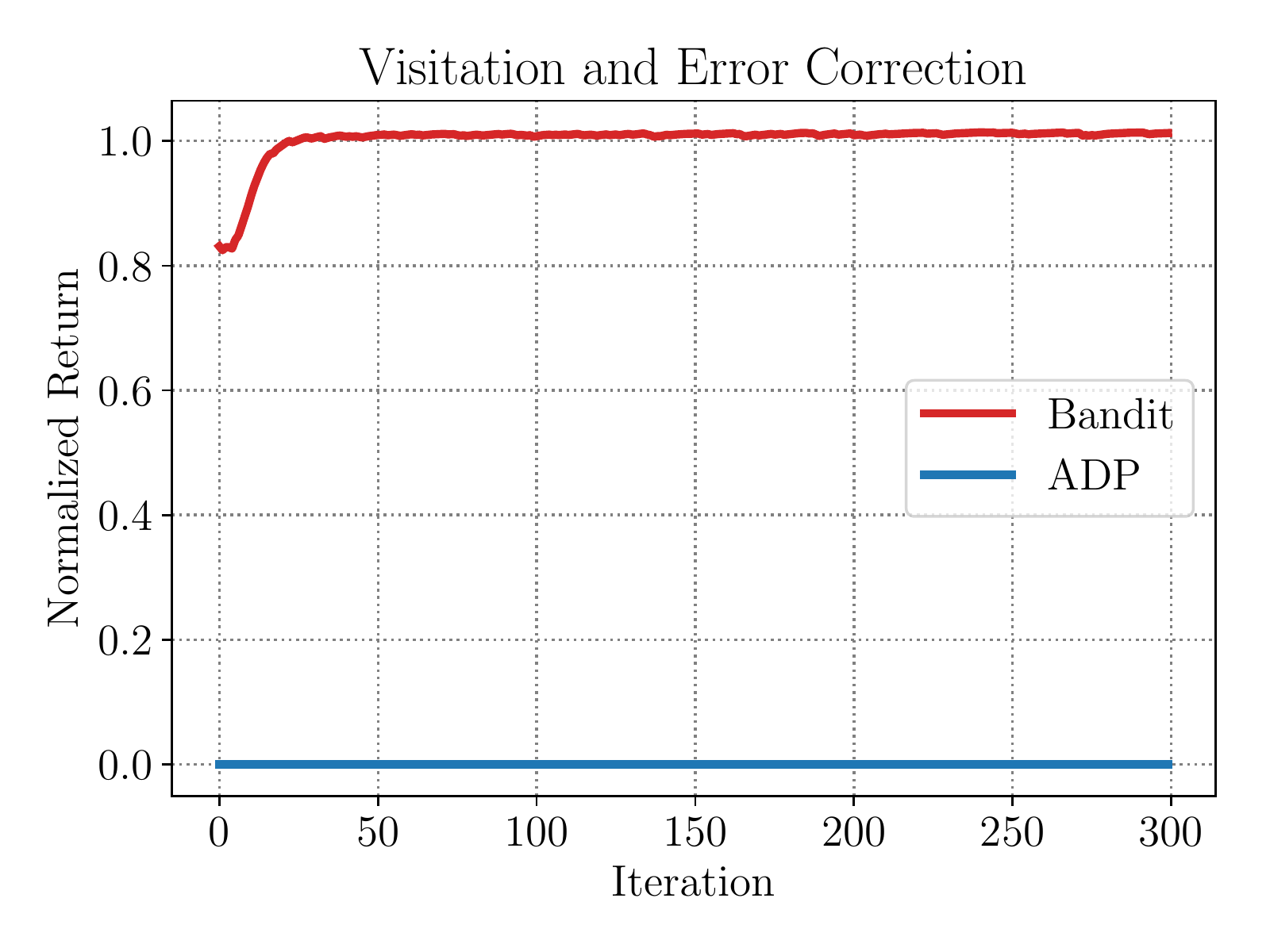}
\end{subfigure}%
~
\begin{subfigure}{.47\columnwidth}
  \includegraphics[width=0.7\textwidth,left]{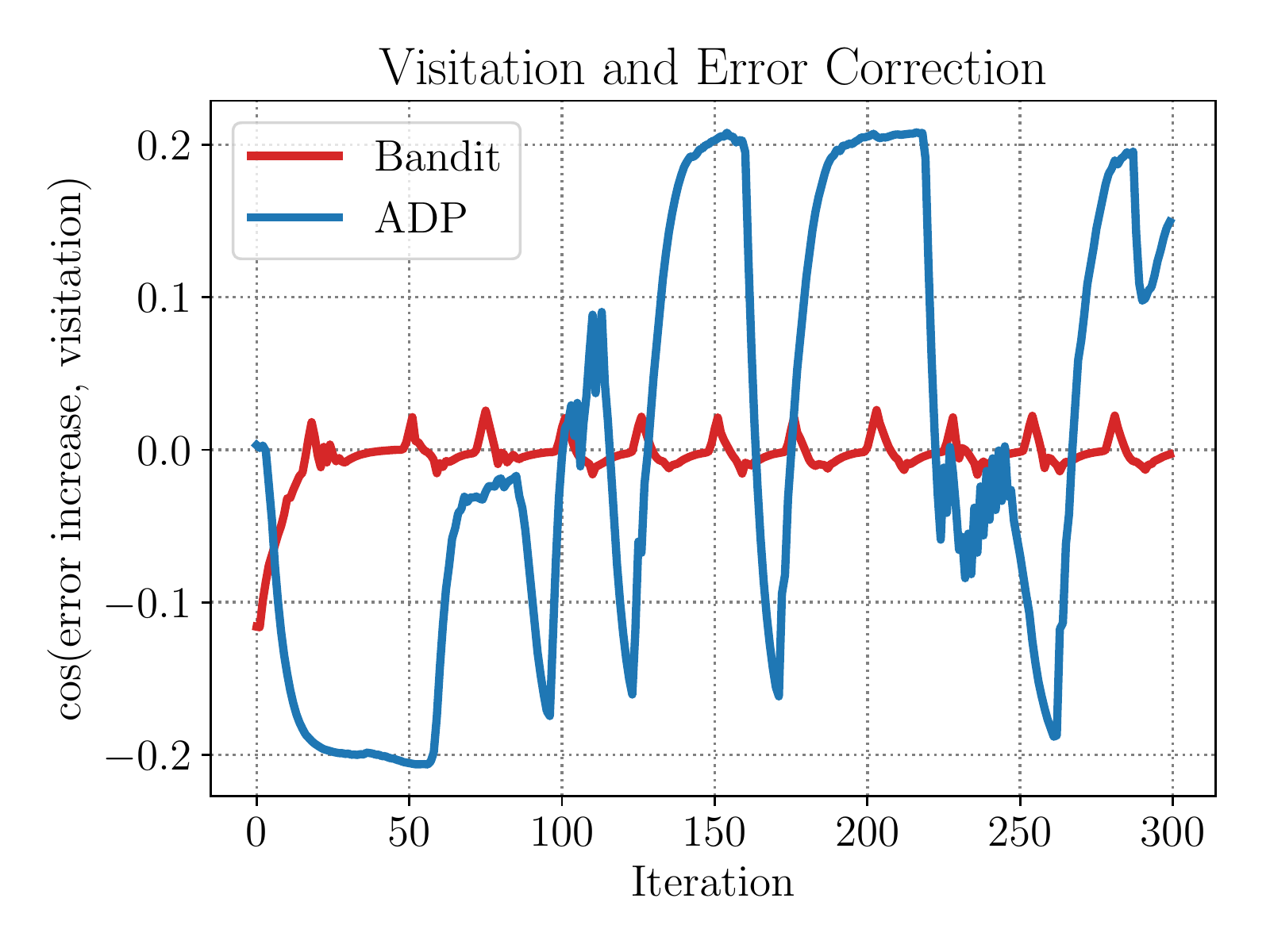}
\end{subfigure}
\caption{\footnotesize{\textbf{Left:} Return attained by with direct supervised learning of values in a bandit problem (``Bandit'') and with bootstrapped targets (``ADP'') on a simple tabular MDP. \textbf{Right:} Cosine similarity between per-iteration error increase $\mathcal{E}_{k+1} - \mathcal{E}_k$ and the policy's state visitation frequency $d^{\pi_k}$ for the two training runs. The ADP run performs poorly, and exhibits a \textit{positive} cosine similarity for prolonged periods during training, suggesting that the policy visitation at least at some states correlates with an \textbf{increase} in error. With ground truth targets (``Bandit''), this quantity is negative until convergence (around iteration 25), and then fluctuates around $0$, suggesting that this approach does benefit from corrective feedback.}}
\vspace{-15pt}
\label{fig:visitation_doesnt_correct_eror}
\end{figure}

Thus, na\"ively coupling the choice of data distribution $\mu$ and the $Q$ function being optimized, by sampling uniformly from a replay buffer collected by policies corresponding to prior Q-functions~\cite{Mnih2015}, or even just sampling from data collected by the latest policy, can cause an absence of corrective feedback. As we will show in Section~\ref{sec:consequences}, this can lead to several detrimental effects such as sub-optimal convergence, instability, and slow learning progress. To theoretically analyze the lack of corrective feedback, we first define error against the optimal value function as follows:

\begin{definition}
\label{eqn:value_of_feedback}
The value error is defined as the expected absolute error to the optimal Q-function $Q^*$ weighted under the corresponding on-policy ($\pi_k$) state-action marginal, $d^{\pi_k}$:
$$\mathcal{E}_k = \expec_{d^{\pi_{k}}}[|Q_k - Q^*|]. $$
\end{definition}

A smooth decrease in the value error $\mathcal{E}_k$ indicates that the algorithm is effectively correcting errors in the Q-function. If $\mathcal{E}_k$ fluctuates or increases, the algorithm may be making poor learning progress. When the value error $\mathcal{E}_k$ is roughly stagnant at a non-zero value, this may indicate premature, sub-optimal convergence, provided that the function approximator class is able to support lower-error solutions (which is typically the case for large deep networks).
Thus, having the learning process monotonically and quickly bring the value error $\mathcal{E}_k$ down to $0$ is desirable for effective learning.
By means of a simple didactic example, we now discuss some reasons why corrective feedback may be absent in ADP methods. We will then describe several examples that illustrate the negative consequences of absent corrective feedback on learning progress.

\subsection{A Didactic Example and Theoretical Analysis of an Absence of Corrective Feedback}
\label{sec:didactic_example_and_theory}

In this section, we first present a diadctic example which provides intuition for why corrective feedback may be absent in RL, and then we generalize this intuition to a more formal result in Section~\ref{sec:theoretical_analysis}.

\subsubsection{Didactic Example}
\label{sec:didactice_example}
{We first describe our didactic example which is a tree-structured deterministic MDP (Figures \ref{fig:on_policy_example} and \ref{fig:discor_sample}) with 7 states and 2 actions, $a_1$ and $a_2$, at each state. The MDP has deterministic transitions, where action $a_1$ transitions to a node's left child, and $a_2$ transitions to the right child. At each leaf node, the episode terminates and the agent receives a reward.}

{We illustrate the learning progress of Q-learning (Algorithm~\ref{alg:fqi}) on this tree-MDP in Figure~\ref{fig:on_policy_example} and the learning progress of a method with an optimal distribution in Figure~\ref{fig:discor_sample}. In the on-policy setting, Q-values at states are updated according to the visitation frequency of these states under the current policy.
Since the leaf nodes are the least likely in this distribution, the Bellman backup is slow to correct errors at the leaves. Using these incorrect leaf Q-values as target values for nodes higher in the tree then gives rise to incorrect
values, even if Bellman error is fully minimized for the sampled transitions. Thus, most of the Bellman updates do not actually bring the updated Q-values closer to $Q^*$.
If we carefully order the states, such that the lowest level nodes are updated first before proceeding upwards in the tree, as shown in Figure \ref{fig:discor_sample}, the errors are much lower. We will show in Section~\ref{sec:method_description} that our proposed approach in this paper aims at performing updates in the manner shown in Figure~\ref{fig:discor_sample} by re-weighting transitions collected by on-policy rollouts via importance sampling.}

\begin{figure*}
\centering
\includegraphics[width=0.9\linewidth]{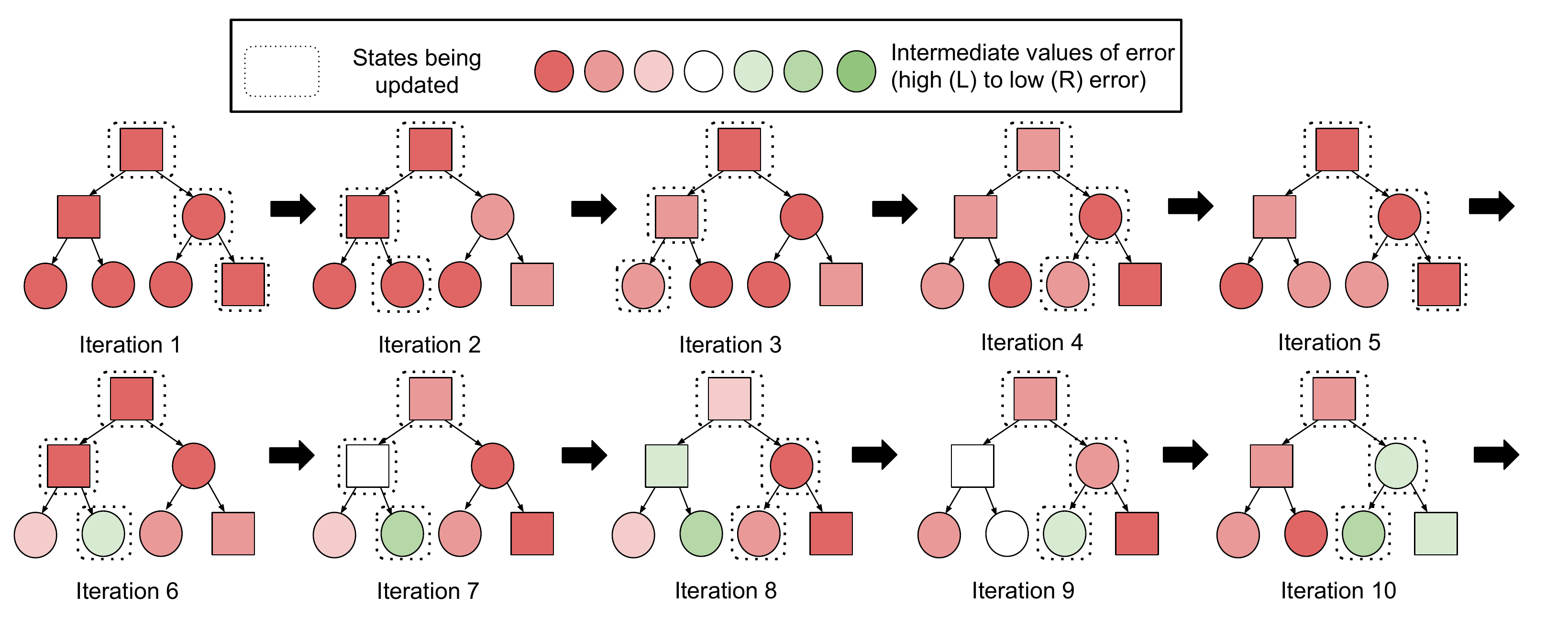}
    \caption{\footnotesize{Iterations of Q-learning on a tree-structured MDP. Trajectories are sampled using the current policy, with each trial (shown as dotted borders around states) starting from the root. Function approximation leads to aliasing of box-shaped nodes and circle-shaped nodes, such that updates to one circle-shaped node affect all other circles, and likewise for box-shaped nodes. Due to the training distribution and aliasing, this method often backs up incorrect target values. Due to aliasing, previously correct values at other states may become incorrect due to these erroneous backups, resulting in non-convergence. This issue is due to each of the leaf nodes (which cause the errors) being sampled less often than the nodes higher in the tree (which suffer from these errors), thus disabling  the algorithm from correct erroneous target values.}}
  \vspace{-15pt}
  \label{fig:on_policy_example}
\end{figure*}

\paragraph{Reasons for the absence of corrective feedback.} The above example provides an intuitive explanation for how on-policy or replay buffer distributions may not lead to error correction. Updates on states from such distributions may fail to correct Q-values at states that are the \textit{causes} of the errors. In general, Bellman backups rely on the correctness of the values at the states that are used as targets, which may mean relying on the correctness of states that occur least frequently in any distribution generated by online data collection. While this is sufficient to guarantee convergence in tabular settings, with function approximation this can lead to an absence of corrective feedback, as states with erroneous targets are visited more and more often, while the visitation frequency of the states that \emph{cause} these errors does not increase.

\begin{figure*}
\centering
\includegraphics[width=0.9\linewidth]{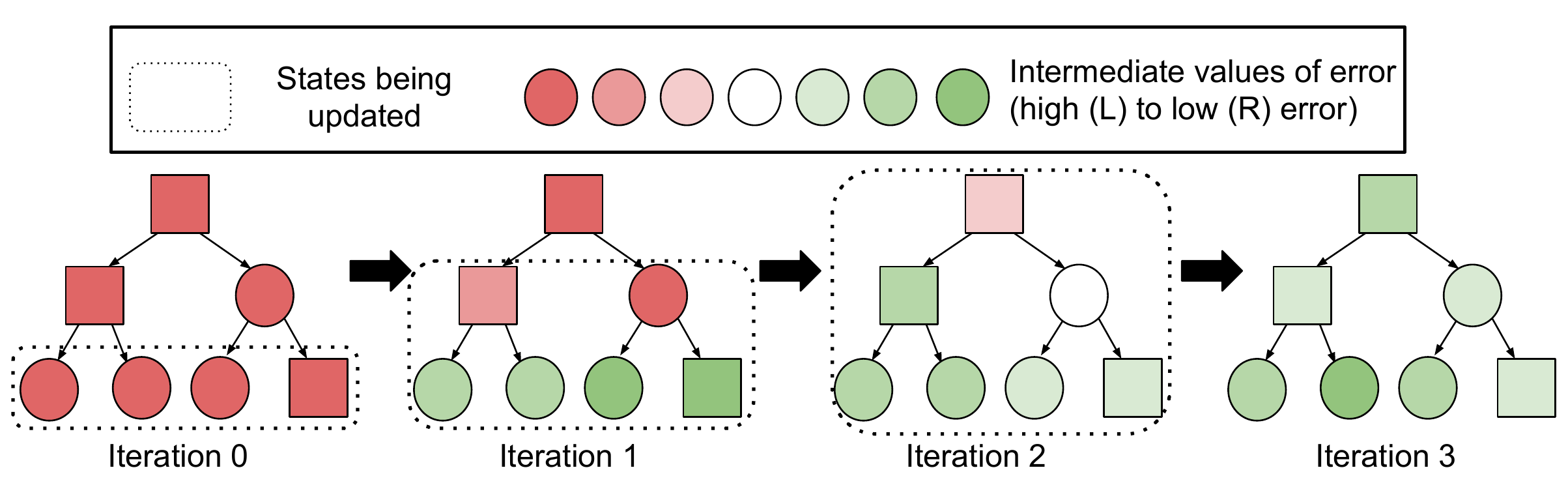}
    \caption{\footnotesize{Iterations of Q-learning on a tree-structured MDP with an optimal training distribution, where states are sampled starting from the leaf nodes, progressing upwards towards the root node in the tree. Note that this method backs up \textit{very few} incorrect values in any iterations, and takes only a few iterations to converge. Our aim will be to approximate such an optimal training distribution.}}
  \vspace{-10pt}
  \label{fig:discor_sample}
\end{figure*}

\subsubsection{Theoretical Analysis} 
\label{sec:theoretical_analysis}
We next present a theoretical result that generalizes the didactic example in a formal result. Proof for this result can be found in Appendix~\ref{sec:omitted_proofs}. 
Our result is a lower-bound on the iteration 
complexity of on-policy Q-learning. We show that, under the on-policy distribution, Q-learning may require exponentially many iterations to learn the optimal Q-function. This will also provide an explanation for the slow and unstable learning as described in Figure~\ref{fig:on_policy_example} and in Section~\ref{sec:consequences}, Figure~\ref{fig:sparse_reward}. We state the result next. An extension to the replay buffer setting is given in Appendix~\ref{sec:omitted_proofs}.

\begin{theorem}[Exponential lower bound for on-policy and replay buffer distributions]
\label{thm:exponential}
There exists a family of MDPs parameterized by $H > 0$, with $|\mathcal{S}| = 2^H$, $|\mathcal{A}| = 2$, such that even with features, $\Phi$ that can represent the optimal Q-function near-perfectly, i.e., $||Q^* - \Phi w||_{\infty} \leq \varepsilon$, 
on-policy or replay-buffer Q-learning, i.e. $D_k = d^{\pi_k}$, or $D_k = \sum_{i=1}^k d^{\pi_i}$ respectively, requires $\Omega\left(\gamma^{-H}\right)$ \textit{exact} Bellman projection steps for convergence to $Q^*$, if at all the algorithm converges to $Q^*$.  
\end{theorem}
Thus, on-policy or replay buffer distributions can induce extremely slow learning in certain environments, even when all transitions are available for training (i.e., without any exploration issues), requiring exponentially many Bellman backups. When function approximation is employed, the smaller the frequency of a transition, the less likely Q-learning is to correct the Q-value of the state-action pair corresponding to this transition. In contrast, we show in Appendix~\ref{thm:discor_suboptimal} that the method we propose in this paper requires only $\mathrm{poly}(H)$ iterations in this scenario, and it behaves similarly to a method that learns Q-values from the leaf nodes, gradually moving upwards toward the root of the tree.

\subsection{Consequences of Absent Corrective Feedback}
\label{sec:consequences}
{In this section, we investigate the phenomenon of an absence of corrective feedback in a number of RL tasks.} We first plot the value error $\mathcal{E}_k$ for the Q-function by assuming access to $Q^*$ for analysis, but this is not available while learning from scratch. We first show that an absence of corrective feedback happens in practice. In Figures~\ref{fig:suboptimal_conv}, \ref{fig:instability} and \ref{fig:sparse_reward}, we plot $\mathcal{E}_k$ for on-policy and replay buffer sampling. 
Observe that ADP methods can suffer from prolonged periods where $\mathcal{E}_k$ is increasing or fluctuating, and returns degrade or stagnate (Figure~\ref{fig:instability}).
Next, we empirically analyze a number of pathological outcomes that can occur when corrective feedback is absent.

\begin{enumerate}
\item \textbf{Convergence to suboptimal Q-functions.} We find that on-policy sampling can cause ADP to converge to a suboptimal
solution, even in the absence of sampling error. {This is not an issue with the capacity of the function approximator -- even when the optimal Q-function $Q^*$ can be represented in the function class~\citep{fu19diagnosing}, learning converges to a suboptimal fixed point far from $Q^*$.} 
Figure~\ref{fig:suboptimal_conv} shows that the value error $\mathcal{E}_k$ rapidly decreases initially, and eventually converges to a value significantly greater than $0$, from which the learning process never recovers.

\item \textbf{Instability in the learning process.}
Q-learning with replay buffers may not converge to sub-optimal solutions as often as on-policy sampling (Figure~\ref{fig:exact_fqi_runs}). However, we observe that even then (Figure~\ref{fig:instability}), ADP with replay buffers can be
unstable. {For instance, the algorithm is prone to degradation even if the latest policy obtains returns that are very close to optimal returns  (Figure~\ref{fig:instability}).} This instability is often correlated with a lack of corrective feedback, and exists even with all transitions present in the buffer controlling for sampling error. 

\item \textbf{Inability to learn with low signal-to-noise ratio.} 
Lack of corrective feedback can also prevent ADP algorithms from learning quickly in scenarios with \textit{low} \textit{signal-to-noise}
ratio, such as tasks with sparse or noisy rewards (Figure~\ref{fig:sparse_reward}). For efficient learning, Q-learning needs to effectively ``exploit'' the reward signal even in the presence of noise and delay, and we find that the learning becomes significantly worse in the presence of these factors. {This is not an exploration issue, since all transitions in the MDP are provided to the algorithm}.
\end{enumerate}

\begin{figure}
    \begin{subfigure}[l]{0.3\linewidth}
    \centering
      \includegraphics[width=0.96\linewidth]{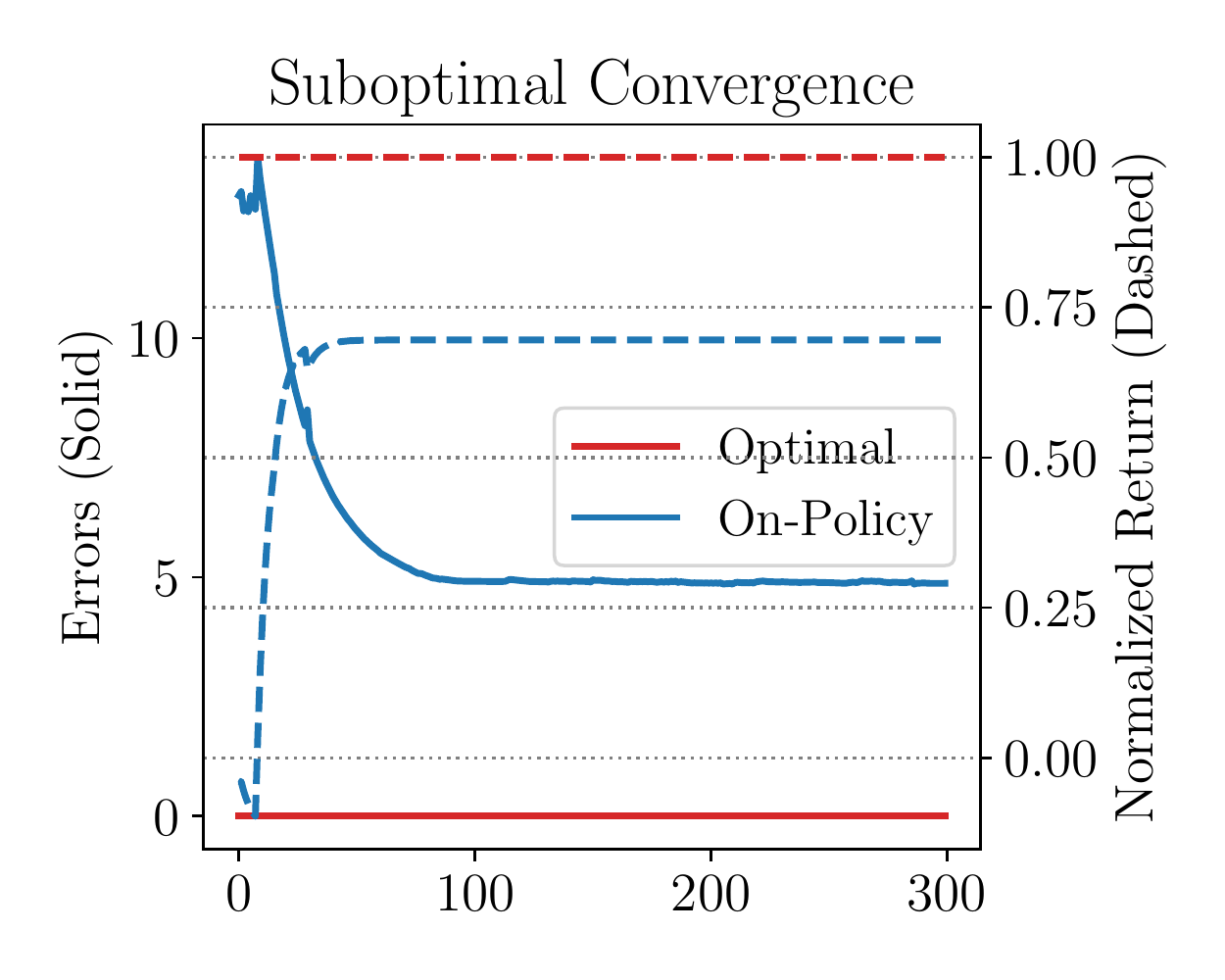}
      \caption{\footnotesize{Sub-optimal convergence for on-policy distributions: return (dashed) and value error (solid). Note that value error decreases rapidly at the start and finally converges to a nonzero value, leading to sub-optimal return.}}
      \label{fig:suboptimal_conv}
    \end{subfigure}
    ~
    \begin{subfigure}[l]{0.3\linewidth}
    \centering
    \includegraphics[width=0.95\linewidth]{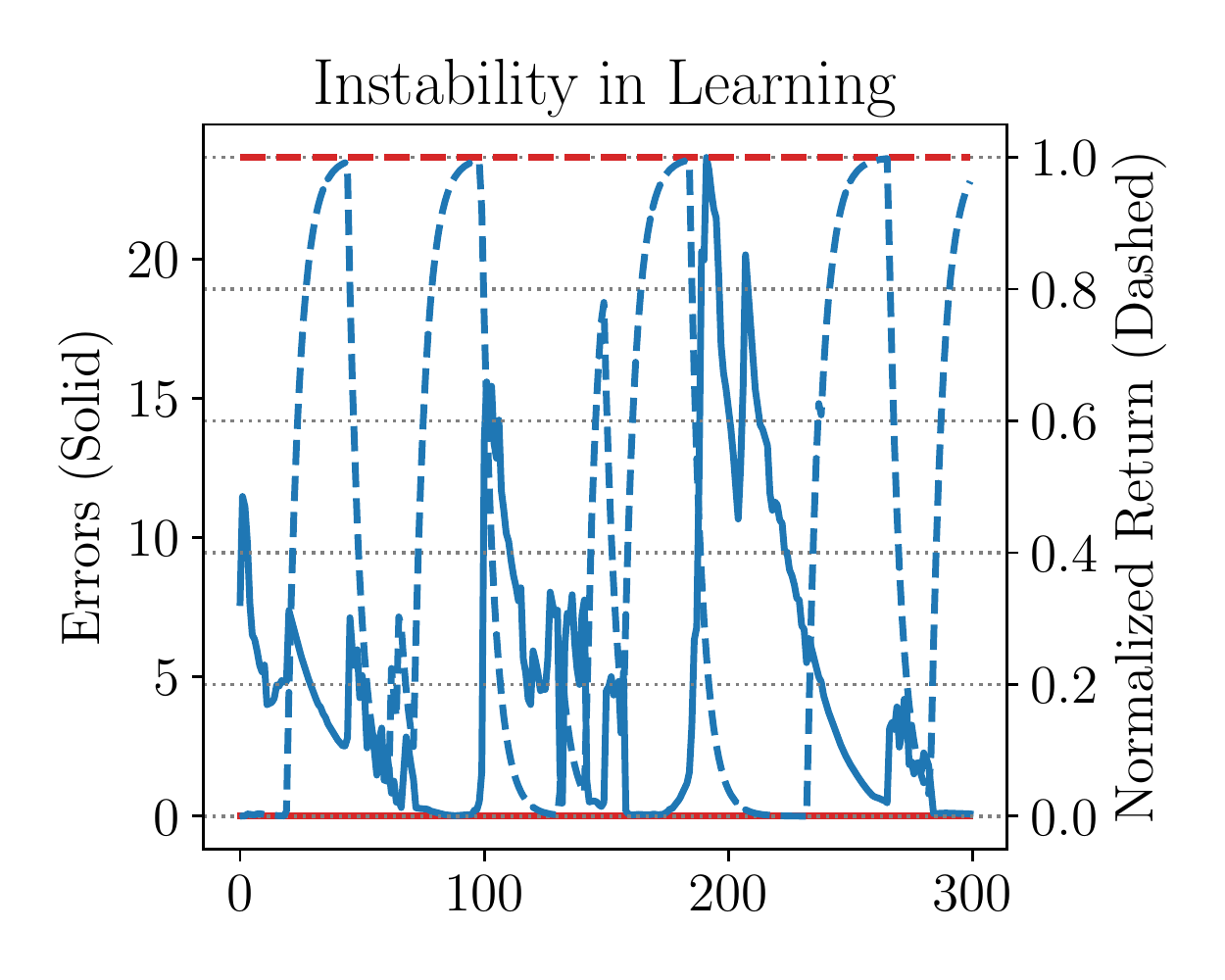}
    \caption{\footnotesize{Instability for replay buffer distributions: return (dashed) and value error (solid) over training iterations. Note the rapid increase in value error at multiple points, which co-occurs with instabilities in returns.}}
    \label{fig:instability}
    \end{subfigure}
    ~
    \begin{subfigure}[l]{0.3\linewidth}
    \centering
      \includegraphics[width=0.95\linewidth]{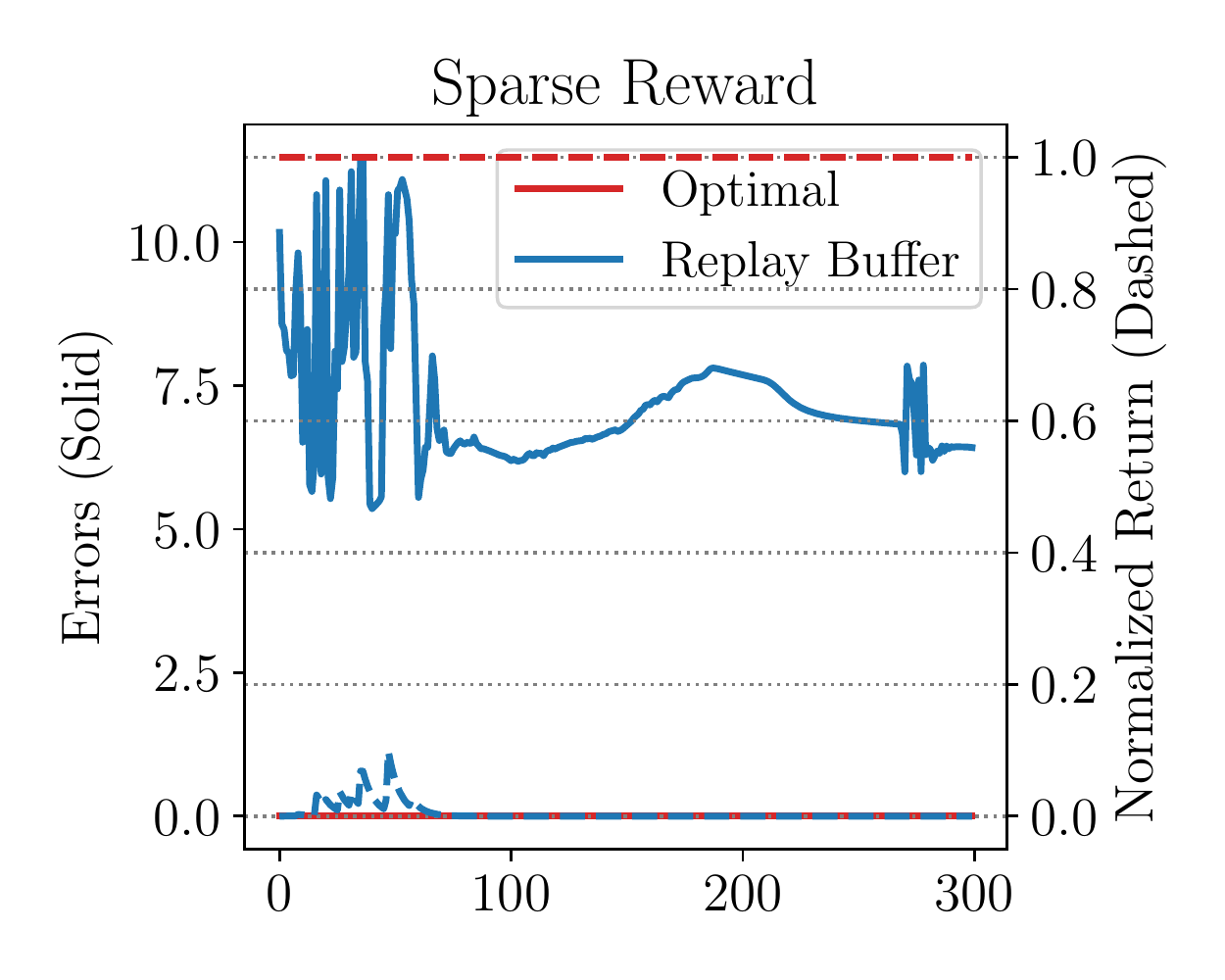}
    \caption{\footnotesize{Error (left) and returns (right) for sparse reward MDP with replay buffer distributions. Note the inability to learn, low return, and highly unstable value error $\mathcal{E}_k$, often increasing sharply, destabilizing the learning process.}}
    \label{fig:sparse_reward}
    \end{subfigure} 
    \caption{\footnotesize{Experiments showing various detrimental consequences of an absence of corrective feedback.}}
\end{figure}

\section{Optimal Distributions for Optimizing Corrective Feedback}
\label{sec:theoretical_derivation}
In the previous section, we observed that an absence of corrective feedback can occur when ADP algorithms na\"ively use the on-policy or replay buffer distributions for training Q-functions. However, if we can compute an ``optimal'' data distribution that provides maximal corrective feedback, and train Q-functions using this distribution, then we can ensure that the ADP algorithm always enjoys corrective feedback, and hence makes steady learning progress. In this section, we aim to derive a functional form for this optimal data distribution. 

We first present an optimization problem for this optimal data distribution, which we refer to as $p_k$ (different from $\mu$, which refers to the data distribution in the replay buffer, or the on-policy distribution), for any iteration $k$ of the ADP algorithm, and then present a solution to this optimization problem. 
Proofs from this section can be found in Appendix~\ref{sec:missing_proof_steps}. We will show in this section that the resulting optimization when approximated practically, yields a very simple and intuitive algorithm. A more intuitive description of the resulting algorithm can be found in Section~\ref{sec:method_description}

Our goal is to minimize the value error $\mathcal{E}_k$ at every iteration $k$, with respect to the the distribution $p_k$ used for Bellman error minimization at iteration $k$. 
The resulting optimization problem is:
\begin{align}
    \vspace{-20pt}
    \begin{split}
    \label{eqn:optimization}
        \min_{p_k}~~& \expec_{d^{\pi_{k}}} \left[ |Q_k - Q^*| \right] \\
    \text{~~s.t.~~} & Q_k = \arg\min_{Q} \expec_{p_k} \left[ (Q - \bellmanopt Q_{k-1})^2 \right]. 
    \end{split}
\end{align}
\begin{theorem}
\label{eqn:theorem_solution}
An optimal solution $p_k$ to optimization problem~\ref{eqn:optimization} satisfies the following relationship:
\begin{equation}
    \label{eqn:optimal_p}
    \vspace{-3pt}
    p_k(s, a) \propto \exp\left(-|Q_k - Q^*|(s, a)\right) \frac{|Q_k - \bellmanopt Q_{k-1}|(s, a)}{\lambda^*}
\end{equation}
where $\lambda^* \in \mathbb{R}^{+}$, is an optimal Lagrange dual variable that ensures $p_k$ is a valid distribution in Problem~\ref{eqn:optimization}. 
\end{theorem}
\begin{proof}
\textbf{(Sketch)} A complete proof is provided in Appendix~\ref{sec:missing_proof_steps}. Here we provide a rough sketch of the steps involved for completeness. We first use the Fenchel-Young inequality~\cite{rockafellar-1970a} to obtain an upper bound on the true objective, in terms of the ``soft-min'' of the errors $|Q_k - Q^*|$ and then minimize this upper bound. Formally, this is given by the RHS of the equation below, $\mathcal{H}$ denotes Shannon entropy.
\begin{equation*}
    \small{
    \expec_{d^{\pi_k}} \left[ |Q_k \!-\! Q^*| \right] \leq  \mathcal{H} \left( d^{\pi_k} \right) \!-\!\log \left(\!\sum_{s, a}\exp(-|Q_k \!-\! Q^*|) \!\right).}
\end{equation*}
Then, we solve for $p_k$ by setting the gradient of the Lagrangian for Problem~\ref{eqn:optimization} to $0$, which requires an addition of constraints $\sum_{s, a} p_k(s, a) = 1$ and $p_k(s, a) > 0$ to ensure that $p_k$ is a valid distribution. We then use the implicit function theorem (IFT)~\citep{Krantz2002TheIF} to compute implicit gradients of $Q_k$ with respect to $p_k$. IFT is required for this step since $Q_k$ is an output of a minimization procedure that uses $p_k$ as an input. Rearranging the expression gives the above result.    
\end{proof}

Intuitively, $p_k$, shown in Equation~\ref{eqn:optimal_p}, assigns higher probability to state-action tuples with high Bellman error $|Q_k - \bellmanopt Q_{k-1}|$, but only when the resulting Q-value $Q_k$ is close to $Q^*$. However, the expression for $p_k$ contains terms that depend on $Q^*$. Since both $Q_k$ and $Q^*$ are observed only \emph{after} $p_k$ is chosen, we need to estimate these quantities using surrogate quantities which we discuss in the following section.

\vspace{-5pt}
\subsection{Tractable Approximation to Optimal Distribution}
\label{sec:minimax}
\vspace{-2pt}
The expression for $p_k$ in Equation~\ref{eqn:optimal_p} contains terms dependent on $Q^*$ and $Q_k$, namely $|Q_k - Q^*|$ and $|Q_k - \bellmanopt Q_{k-1}|$. As described previously, since both $Q_k$ and $Q^*$ are observed only \emph{after} $p_k$ is chosen, in this section, we develop surrogates to estimate these quantities. For error against $Q^*$, we show that the cumulative sum of discounted Bellman errors over the past iterations of training, denoted as $\Delta_k$, shown in Equation~\ref{eqn:delta_k}, is a valid proxy (equivalent to an upper bound) for $|Q_k - Q^*|$. In fact, Theorem~\ref{thm:delta_k_upper_bound} shows that $\Delta_k$, offset by a state-action independent constant, is a tractable upper bound on $|Q_k - Q^*|$ constructed only from prior Q-function iterates, $Q_0, \cdots, Q_{k-1}$. 
\vspace{-2pt}
\begin{align}
\begin{split}
    \Delta_k = & \sum_{i=1}^{k} \gamma^{k-i} \left(\prod_{j=i}^{k-1} P^{\pi_{j}} \right) |Q_{i} - \bellmanopt Q_{i-1}| \\
   \implies  \Delta_k = & |Q_k - \bellmanopt Q_{k-1}| + \gamma P^{\pi_{k-1}} \Delta_{k-1} .
    \label{eqn:delta_k}
\end{split}
\end{align}
The following theorem formally states this result. 
\begin{theorem}
\label{thm:delta_k_upper_bound}
There exists a $k_0 \in \mathbb{N}$, such that  $\forall~ k \geq k_0$ and $\Delta_k$ from Equation~\ref{eqn:delta_k}, $\Delta_k$ satisfies the following inequality, pointwise, for each $s, a$:
\vspace{-10pt}
\begin{equation*}
    \Delta_k + \sum_{i=1}^k \gamma^{k-i} \alpha_i \geq |Q_k - Q^*|, \text{~~} \alpha_i = \frac{2 R_{\text{max}}}{1 - \gamma} \tvd(\pi_i, \pi^*)
\end{equation*}
\end{theorem}
\begin{proof}
\textbf{(Sketch)} A full proof is provided in Appendix~\ref{app:other_proofs}. The key insight in this argument is to use a recursive inequality, presented in Lemma~\ref{thm:worst_case_estimator}, App.~\ref{app:other_proofs}, to decompose $|Q_k - Q^*|$, which allows us to show that $\Delta_k + \sum_i \gamma^{k-i} \alpha_i$ is a solution to the corresponding recursive equality, and hence, an upper bound on $|Q_k - Q^*|$. Note that, the initialization $|Q_0 - Q^*|$ matters only infinitesimally once, $k \geq k_0$, with $k_0$ being such that $\gamma^{k_0} |Q_k - Q^*| < 1$, therefore, agnostic to the initialization of $\Delta$, $\Delta_0$, we note that the statement in the theorem holds true for large-enough $k$. 
\end{proof}

\paragraph{Estimating $|Q_k - \bellmanopt Q_{k-1}|$.} The expression for $p_k$ in Equation~\ref{eqn:optimal_p} also includes an unobserved Bellman error multiplier term, $|Q_k - \bellmanopt Q_{k-1}|$ as well. With no available information about $Q_k$ -- which will only be observed \textit{after} Bellman error minimization under $p_k$ -- a viable approximation is to bound this term $|Q_k - \bellmanopt Q_{k-1}|$ between the minimum and maximum Bellman errors obtained at the previous iteration, $c_1 = \min_{s, a} |Q_{k-1} - \bellmanopt Q_{k-2}|$ and $c_2 = \max_{s, a} |Q_{k-1} - \bellmanopt Q_{k-2}|$. We can simply restrict the support of state-action pairs $(s, a)$ used to compute $c_1$ and $c_2$ to come from transitions observed in the replay buffer used for the Q-function update, to ensure that both $c_1$ and $c_2$ are finite. 

\paragraph{Re-weighting the replay buffer $\mu$.} Since it is challenging to directly obtain samples from $p_k$ via online interaction, a practically viable alternative is to instead perform weighted Bellman updates by re-weighting transitions drawn from the a regular replay buffer $\mu$ using importance weights given by $w_k = \frac{p_k(s, a)}{\mu(s, a)}$. However, na\"ive importance sampling often suffers from high-variance of these importance weights, leading to instabilities in learning. To prevent such issues, instead of directly re-weighting to $p_k$, we re-weight samples from $\mu$ to a projection of $p_k$, denoted as $q_k$, that is still close to $\mu$ under the KL-divergence metric: ${q_k = \arg \min_q \kldiv(q||p_k) + (\tau - 1) \kldiv(q||\mu)}$, where $\tau$ is a scalar, $\tau > 1$.
The equation for weights $w_k$, in this case, is thus given by:
\begin{equation}
    \label{eqn:optimal_wk}
    w_k(s, a) \propto \exp\left(\frac{-|Q_k - Q^*|(s, a)}{\tau}\right) \frac{|Q_k - \bellmanopt Q_{k-1}|}{\lambda^*}
\end{equation}
A derivation is provided in Appendix~\ref{app:other_proofs}.
Having described all approximations, we now discuss how to obtain a tractable and practically usable data distribution for training the Q-function that maximally mitigates error accumulation.

\subsection{Putting it All Together}
\label{sec:putting_all}
We have noted all practical approximations to the expression for optimal $p_k$ (Equation~\ref{eqn:optimal_p}), including estimating surrogates for $Q_k$ and $Q^*$, and the usage of importance weights to develop a method that can achieve the benefits of the optimal distribution, simply by \textit{re-weighting transitions} in the replay buffer, \textit{rather than altering the exploration strategy}. We also discussed a technique to reduce the variance of weights used for this reweighting. We now put these techniques together to obtain the final, practically tractable expression for the weights used for our practical approach.

We note that the term $|Q_k - Q^*|$, appearing inside the exponent in the expression for $w_k$ in Equation~\ref{eqn:optimal_wk} can be approximated by the tractable upper bound $\Delta_k$. However, computing $\Delta_k$ requires the quantity $|Q_k - \bellmanopt Q_{k-1}|$ which also is unknown when $w_k$ is being chosen. Combining the upper bound on $|Q_k - \bellmanopt Q_{k-1}| \leq c_2$, Theorem~\ref{thm:delta_k_upper_bound} and Equation~\ref{eqn:delta_k}, we obtain the following bound:
\begin{equation}
    \label{eqn:optimal_p_approx}
    |Q_k - Q^*| \leq \gamma P^{\pi_{k-1}} \Delta_{k-1} + c_2 + \sum_i \gamma^i \alpha_i
\end{equation}
Using this bound in the expression for $w_k$, along with the lower bound, $|Q_k - \bellmanopt Q_{k-1}| \geq c_1$, we obtain the following lower bound on weights $w_k$:
\begin{equation}
    \label{eqn:worst_case_weights}
    w_k \propto \exp\left(\frac{-c_2 -\gamma \left[P^{\pi_{k-1}} \Delta_{k-1}\right](s, a)}{\tau}\right) \frac{c_1}{\lambda^*}
\end{equation}
Finally, we note that using a worst-case lower bound for $w_k$ (Equation~\ref{eqn:worst_case_weights}) will down-weight some additional transitions which in reality lead to low error accumulation, but this scheme will never up-weight a transition with high error, thus providing for a ``conservative'' distribution. A less conservative expression for getting these weights is a subject of future work.
Simplifying the constants $c_1$, $c_2$ and $\lambda^*$, the final expression for the practical choice of $w_k$ is:
\begin{equation}
    w_k(s, a) \propto \exp{\left(-\frac{\gamma \left[P^{\pi_{k-1}} \Delta_{k-1}\right](s, a)}{\tau} \right)}. %
    \label{eqn:importance_weights}
\end{equation}

\section{Distribution Correction (DisCor) Algorithm}
\label{sec:method_description}
In this section, we present the resulting algorithm, that uses weights $w_k$ from Equation~\ref{eqn:importance_weights} to re-weight the Bellman backup in order to induce corrective feedback. We first present an intuitive explanation of our algorithm, and then describe the implementation details.

\begin{algorithm}[t]
\small
\caption{\textbf{DisCor (Distribution Correction)}}
\label{alg:discor}
\begin{algorithmic}[1]
    \STATE Initialize Q-values $Q_\theta(s, a)$, initial distribution $p_0(s, a)$, a replay buffer $\mu$, and an \textdiff{error model} $\Delta_\phi(s, a)$.
    \FOR{step $k$ in \{1, \dots, N\}}
        \STATE Collect $M$ samples using $\pi_k$, add them to replay buffer $\mu$, sample $\{(s_i, a_i)\}_{i=1}^N \sim \mu$
        \STATE Evaluate $Q_\theta(s,a)$ and $\Delta_\phi(s, a)$ on samples $(s_i, a_i)$.
        \STATE Compute target values for $Q$ and $\Delta$ on samples:\\
        ${y}_i = r_i + \gamma \max_{a'} Q_{k-1}(s'_i, a')$\\ $\hat{a}_i = \arg \max_{a} Q_{k-1}(s'_i, a) $\\
        $\hat{\Delta}_{i} = |Q_\theta(s, a) - y_i| + \gamma \Delta_{k-1}(s'_i, \hat{a}_i)$
        \STATE \textdiff{Compute $w_k$ using Equation~\ref{eqn:importance_weights}}. 
        \STATE Minimize Bellman error for training $Q_\theta$ weighted by $w_k$. \\
        $ \theta_{k+1} \leftarrow \argmin{\theta} \frac{1}{N}\sum_{i=1}^N \textdiff{w_k(s_i, a_i)} \cdot (Q_\theta(s_i,a_i) - y_i)^2$
        \STATE \textdiff{Minimize ADP error for training $\phi$. \\
        $\phi_{k+1} \leftarrow \argmin{\phi} \frac{1}{N} \sum_{i=1}^N (\Delta_\theta(s_i, a_i) - \hat{\Delta}_i)^2$}
    \ENDFOR
\end{algorithmic}
\end{algorithm}

\paragraph{Intuitive Explanation.} Using weights $w_k$ in Equation~\ref{eqn:importance_weights} for weighting Bellman backups possess a very clear and intuitive explanation. $P^{\pi_{k-1}} \Delta_{k-1}$ corresponds to the estimated upper bound on the error of the target values for the current transition, due to the backup operator $P^{\pi_{k-1}}$. Intuitively, this implies that weights $w_k$ \textit{downweight} those transitions for which the bootstrapped \emph{target} Q-value estimate has a high estimated error to $Q^*$, 
effectively focusing the learning on samples where the supervision (target value) is accurate,
which are precisely the samples that we expect maximally improve the accuracy of the Q function. 
This prevents error accumulation, and hence provides correct feedback. Such a scheme also resembles prior methods for learning with noisy labels by ``abstention'' from training on labels that are likely to be inaccurate~\citep{absention}.

\paragraph{Details.} Pseudocode for our approach, which we call \textbf{DisCor} (\textbf{Dis}tribution \textbf{Cor}rection), is presented in Algorithm~\ref{alg:discor}, with the main differences from standard ADP methods highlighted in red. In addition to a standard Q-function, DisCor trains another parametric model, $\Delta_\phi$, to estimate $\Delta_k(s, a)$ at each state-action pair. The recursion in Equation~\ref{eqn:delta_k} is used to obtain a simple approximate dynamic programming update rule for the parameters $\phi$ (Line 8).
The second change is the introduction of a weighted Q-function backup with weights $w_k(s, a)$, as shown in Equation~\ref{eqn:importance_weights}, on Line 7. 
We also present a practical implementation of the DisCor algorithm, built on top of standard DQN/SAC algorithm pseudocodes in Algorithm~\ref{alg:practical_alg}, Appendix~\ref{app:exp_details}.

\section{Related Work}
\vspace{-2pt}
Prior work has pointed out a number of issues arising when dynamic programming is used with function approximation. Some work~\citep{munos2005error,farahmand2010error,munos2008finite} focused on analysing error induced from the typically used projected Bellman operator under the assumption of an abstract error model. Further, fully gradient-based objectives for Q-learning~\citep{Sutton09a,Sutton09b,maei09nonlineargtd} gave rise to
convergent ADP algorithms with function approximation. 
In contrast to these works, which mostly focus on ensuring {convergence of the Bellman backup}, we focus on the interaction between the ADP update and the data distribution $\mu$. 
On the other hand, prior work on fully online Q-learning from a stochastic approximation viewpoint analyzes time-varying $\mu$, but in the absence of function approximation~\citep{Watkins92,Tsitsiklis1994,devraj2017zap}, or at the granularity of a single time-step in the environment~\citep{Tsitsiklis97ananalysis}, unlike our setting. Our setting involves both a time-varying $\mu$ depending on the prior and latest Q-functions as well as function approximation.  %

A number of prior works have focused on studying non-convergence and generalization effects in ADP methods with deep net function approximators, both theoretically~\citep{Achiam2019TowardsCD} and empirically~\citep{fu19diagnosing,martha2018sparse,kumar19bear}. In this paper, we study a different issue that is distinct from non-convergence and generalization {issues due to the specific choice of deep net function approximators:} the interaction between the data distribution and the fitting error in the value function.

In the fully offline setting, prior works have noted that sampling distributions can affect performance of ADP methods~\citep{fu19diagnosing}. This has been generally aimed at resolving what is typically known as the ``deadly triad''~\citep{suttonrlbook}: {divergence} caused by an interaction between function approximation, bootstrapped updates and off-policy sampling, resulting in the development of batch RL algorithms that choose sampling schemes~\citep{Kolter2011TheFP,sutton16emphatic} {for guaranteed convergence}. However, we show that sampling distributions can have a drastic impact on the learning process even where the algorithm performs online (on-policy) interaction to collects its own data.
\cite{schaul2019ray} studies the interaction between data collection and training for multi-objective policy gradient methods and note the bias towards optimizing only a few components of the objective that arises.

Our proposed algorithm weights the transition in the buffer based on an estimate of their error to the true optimal value function.
Related to our approach, prioritized sampling has been used previously in ADP methods, such as PER~\citep{Schaul2015}, to prioritize transitions with higher Bellman error. We show in Section~\ref{sec:experiments}, that this choice may fail to perform in many cases. 
Recent work~\citep{du2019distributioncheck} has attempted to use a distribution-checking oracle to control the amount of exploration performed. Our work aims at ensuring corrective feedback, solely by re-weighting the data distribution for Q-learning with bootstrapped backups. 

We further discuss the relationship with other prior works in an extended related work section in  Appendix~\ref{app:related_work_extended}.

\section{Experimental Evaluation of DisCor}
\label{sec:experiments}
The goal of our empirical evaluation is to study the following questions: 
\begin{enumerate}
    \item Can DisCor ensure continuous corrective feedback in RL tasks, mitigating the issues raised in Section~\ref{sec:consequences}? 
    \item How does DisCor compare to prior methods, including those that also reweight the data in various ways?
    \item Can DisCor attain good performance in challenging settings, such as multi-task RL or noisy reward signals?
    \item How do approximations from Section~\ref{sec:theoretical_derivation} affect the efficacy of DisCor in mitigating error accumulation?
\end{enumerate}
We start by presenting a detailed analysis on tabular MDPs with function approximation, studying each component of DisCor in isolation,
and then study six challenging robotic manipulation tasks, the multi-task MT10 benchmark from MetaWorld~\citep{yu2019meta}, and three Atari games~\citep{bellemare2013ale}.

\begin{figure}
    \centering
    \begin{subfigure}[t]{.99\linewidth}
        \centering
        \includegraphics[width=0.17\linewidth]{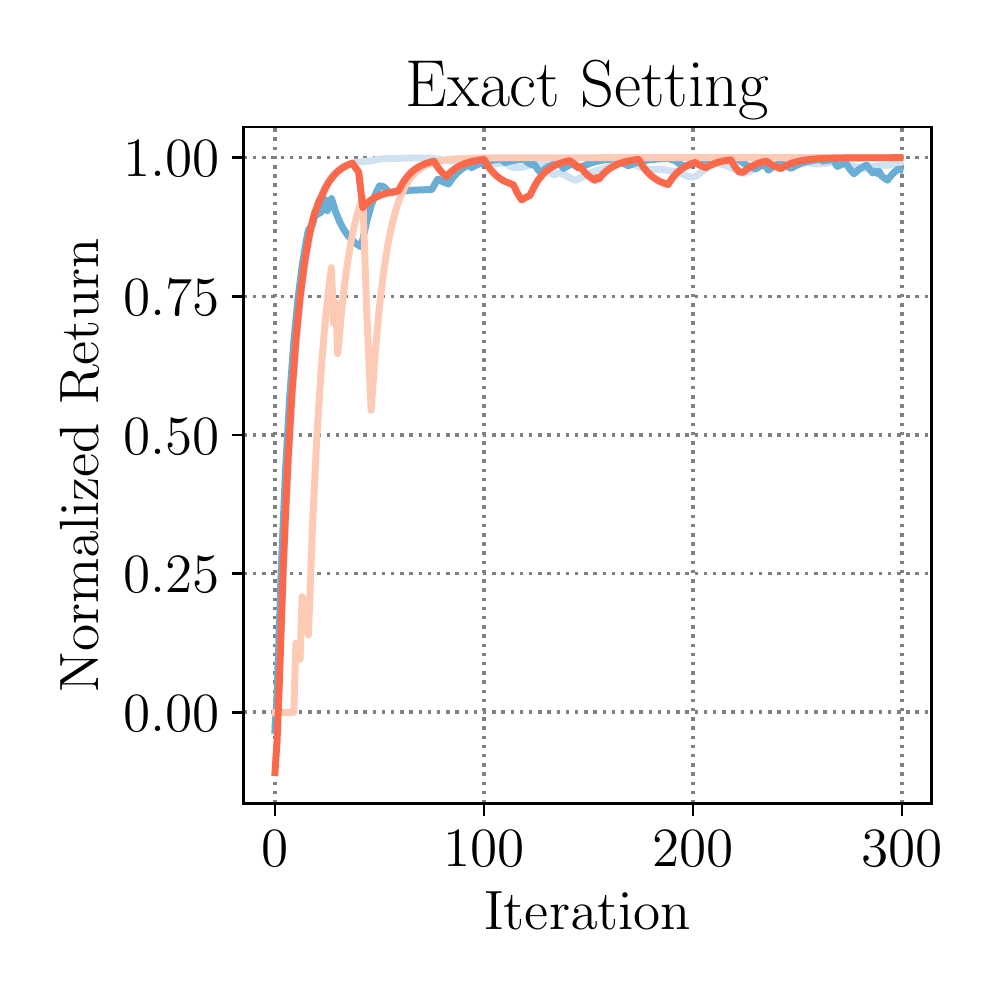}
        \includegraphics[width=0.17\linewidth]{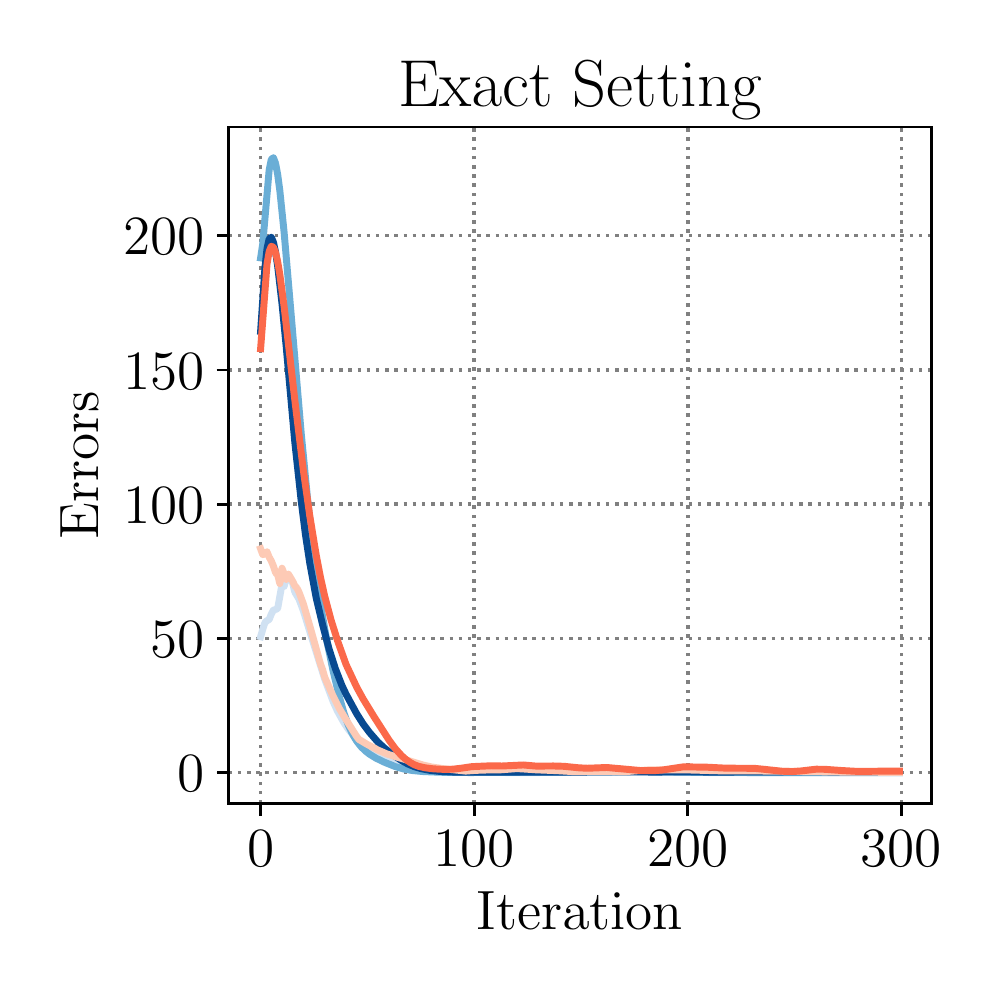}
        ~\vline~
        \includegraphics[width=0.17\linewidth]{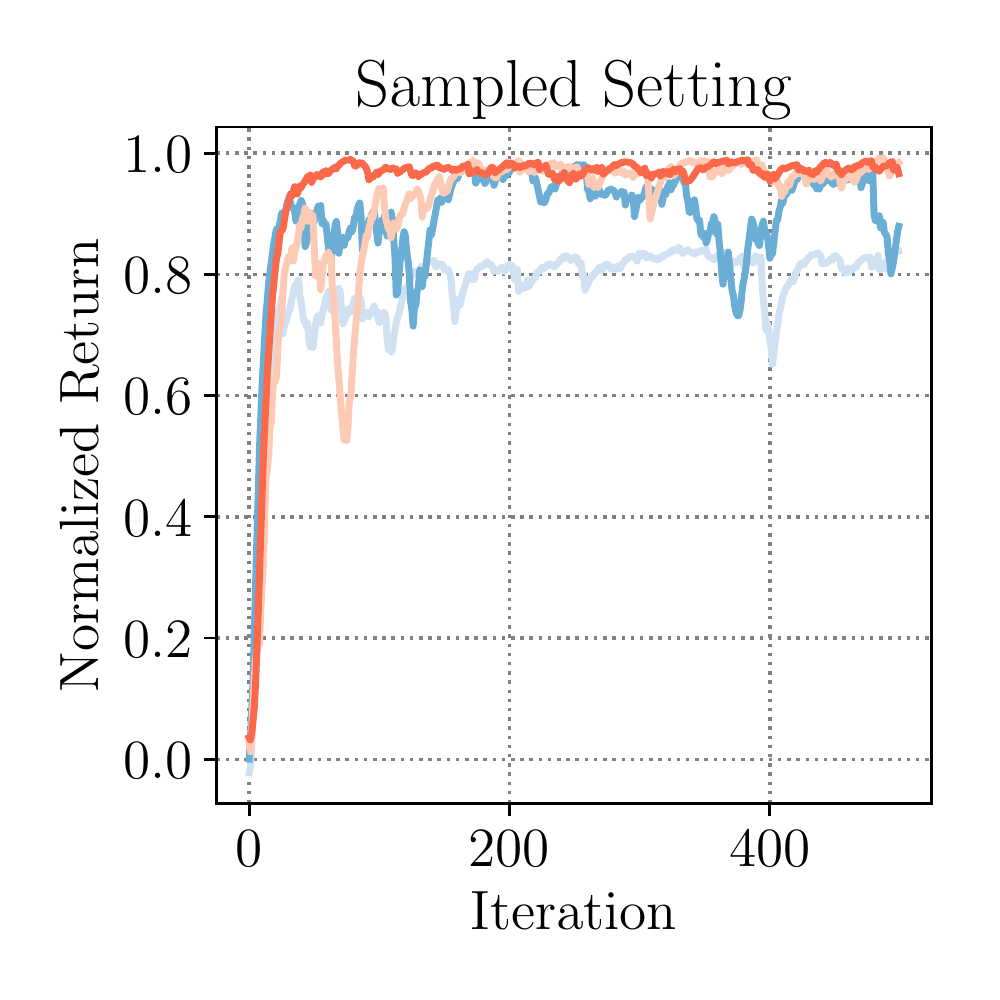}
        \includegraphics[width=0.17\linewidth]{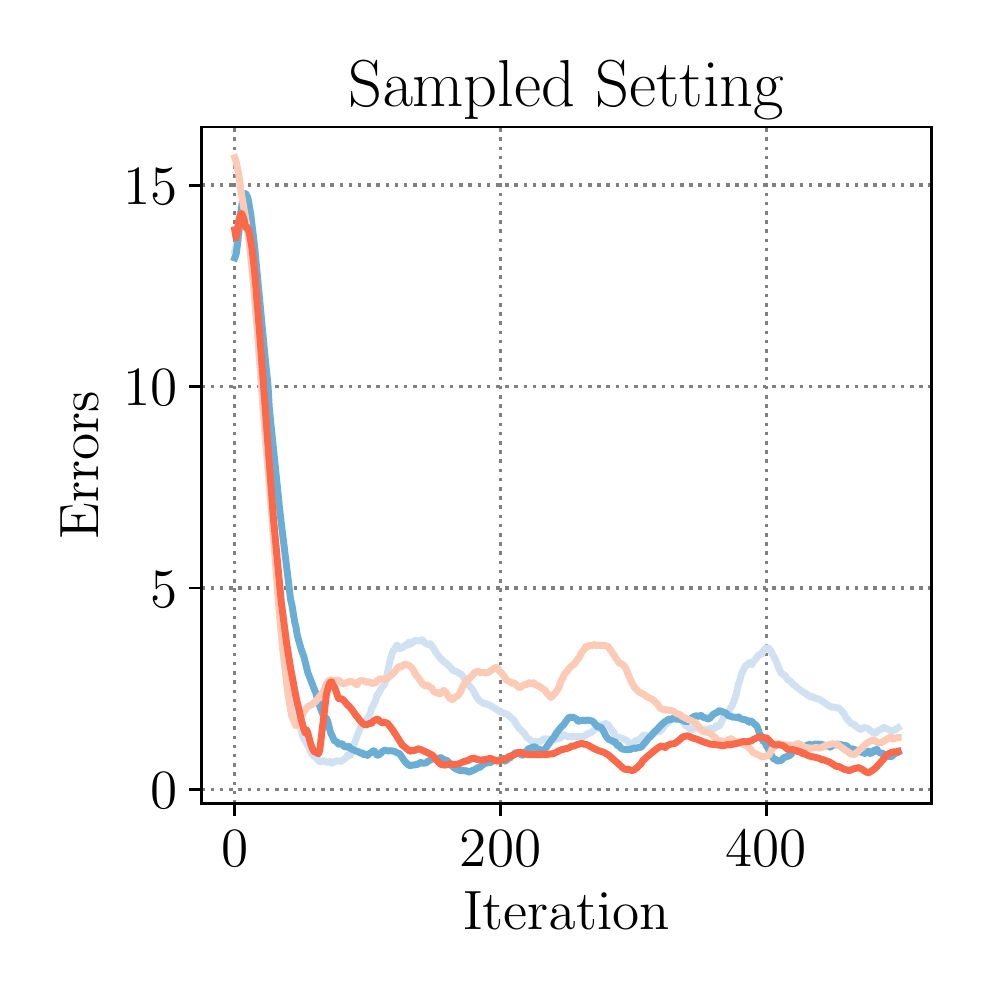}
        \caption{\footnotesize{Value Error $\mathcal{E}_k$/ return for \textit{two} runs of DisCor (blue) and DisCor(oracle) (red) in exact (left) and sampled (right) settings. Observe that (i) DisCor achieves similar performance as DisCor (oracle) generally, (ii) DisCor provides corrective feedback: value error decreases with both DisCor and DisCor(oracle).}}
    \end{subfigure}
    ~
    \begin{subfigure}[t]{.95\linewidth}
        \centering
        \includegraphics[width=0.19\linewidth]{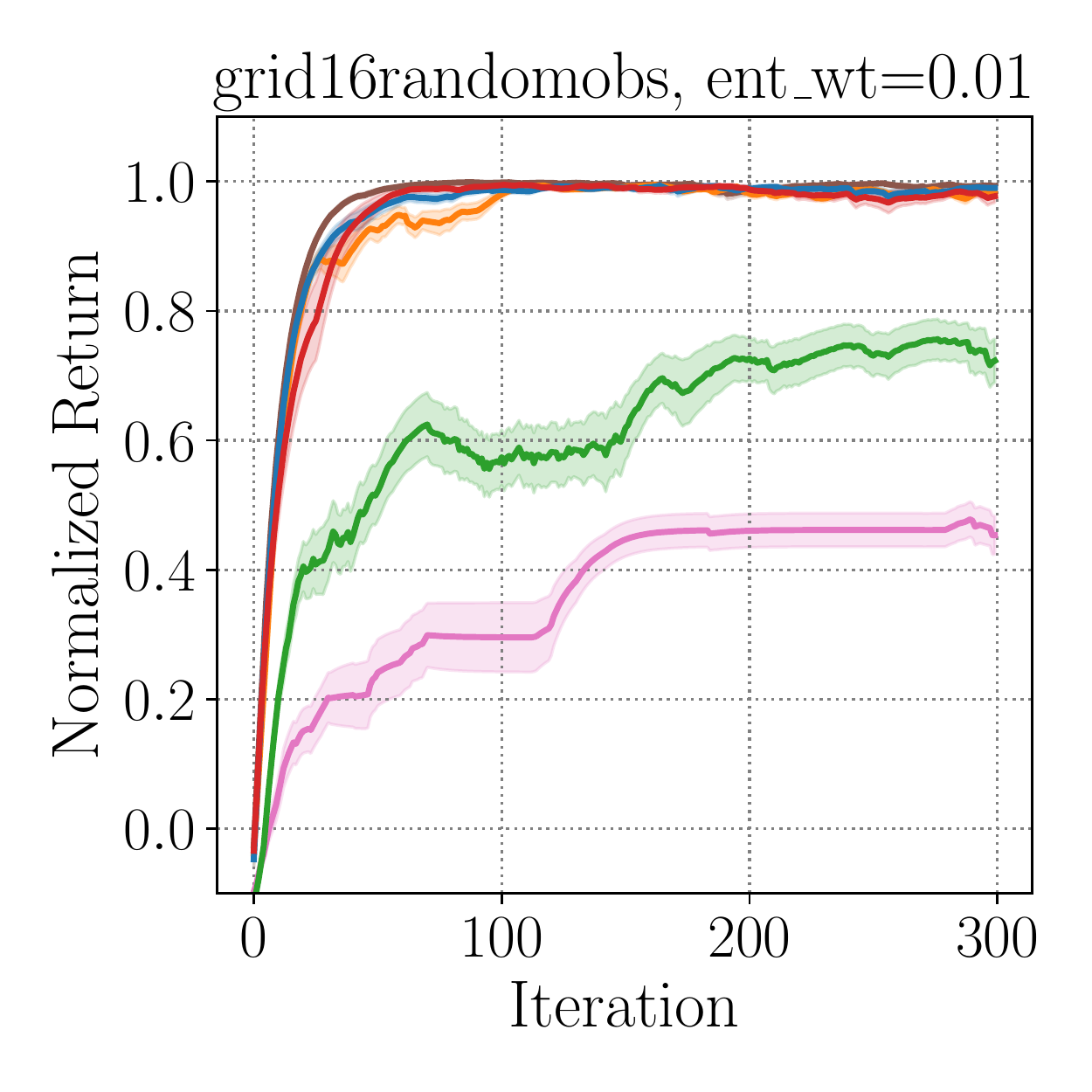}
        \includegraphics[width=0.19\linewidth]{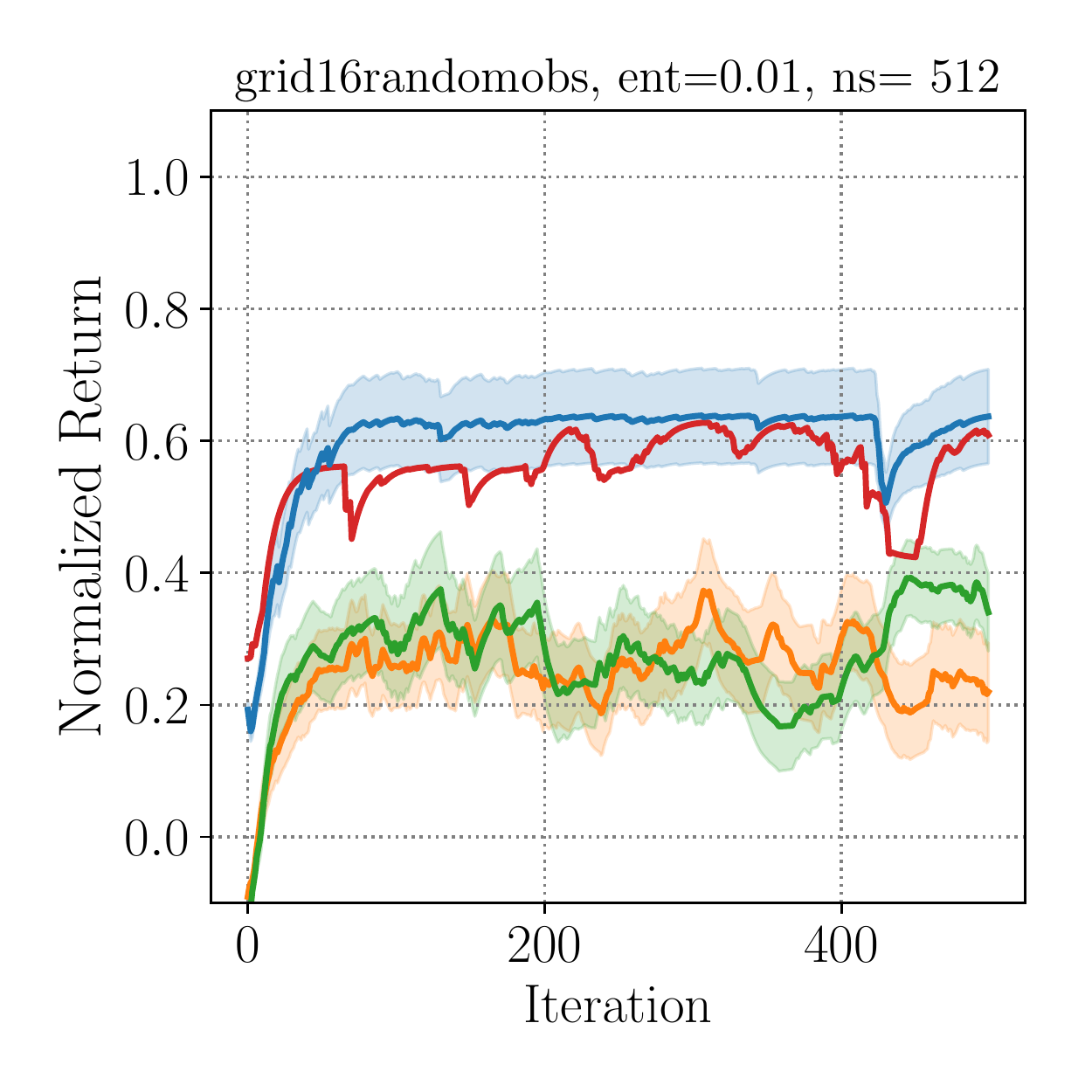}
        \includegraphics[width=0.23\linewidth,trim={2cm 0cm 3cm 2.0},clip]{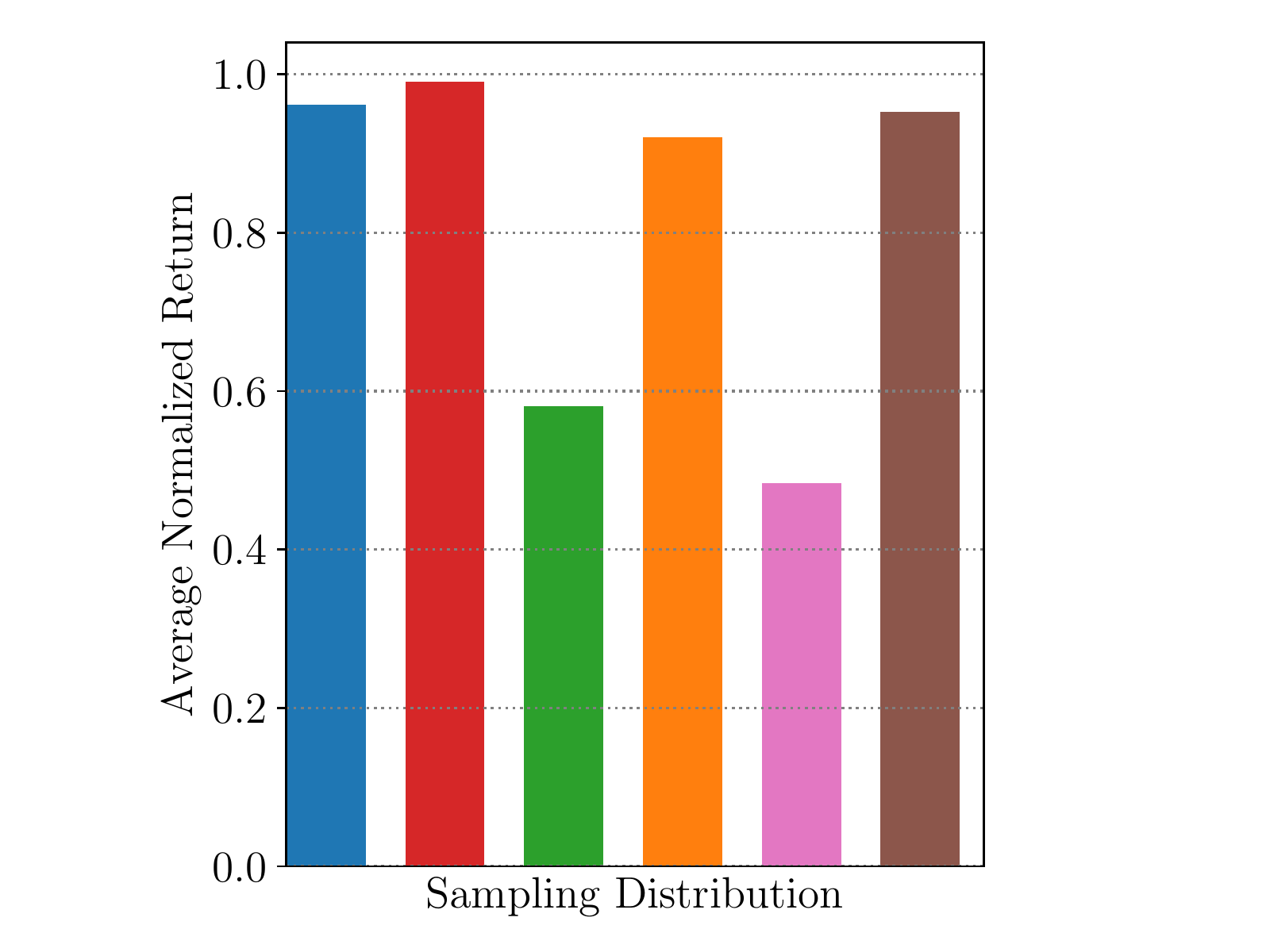}
        \includegraphics[width=0.194\linewidth,trim={3.4cm 1cm 4.0cm 3.5},clip]{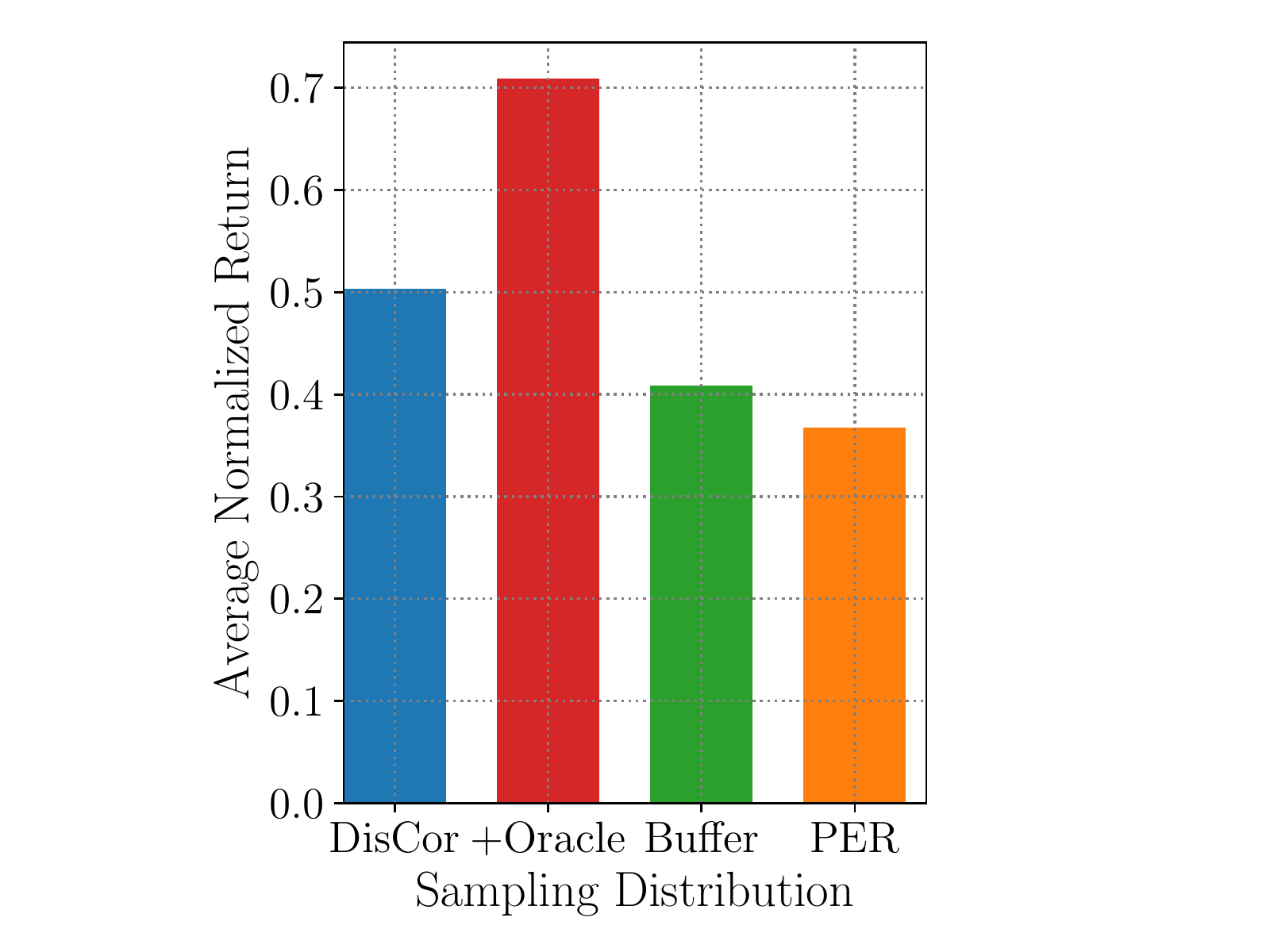}
        \includegraphics[width=0.15\linewidth]{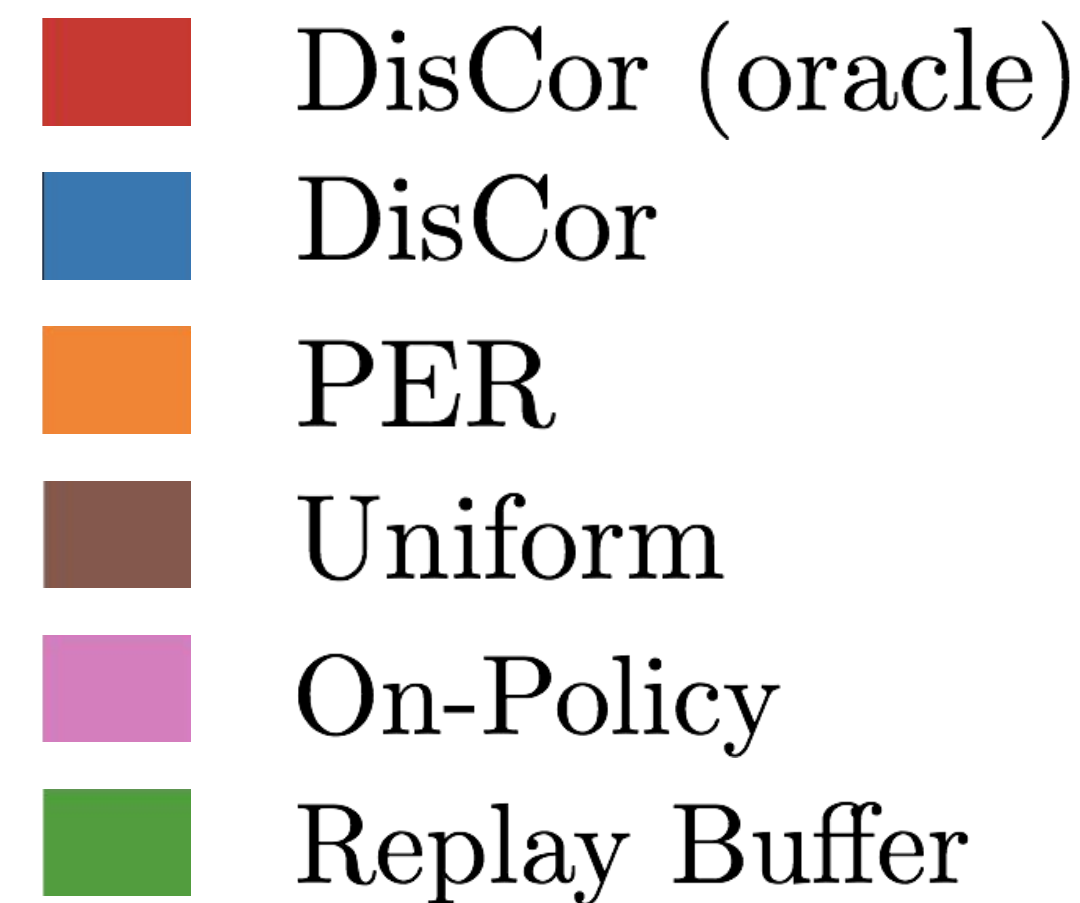}
        \caption{\footnotesize{Comparative Performance on tabular domains}}
    \end{subfigure}
    \caption{\footnotesize{Performance of \textbf{DisCor}, \textbf{DisCor (oracle)}, replay buffer, PER, on-policy and uniform sampling averaged across tabular domains with (right) and without (middle) sampling error. Note that: (1) DisCor generally ensures corrective feedback, (2) DisCor is generally comparable to DisCor (oracle), however, DisCor (oracle) outperforms DisCor, as expected, and (3) DisCor (DisCor and DisCor (oracle) generally outperform all distributions. Exact setup for these domains is described in Appendix~\ref{sec:app_exps_gridworld}.}}
    \label{fig:gridworls_samling}
    \vspace{-0.5cm}
\end{figure}

\subsection{Analysis of DisCor on Tabular Environments}
\label{sec:gridworlds}

We first use the tabular domains from Section~\ref{sec:problem_description}, described in detail in Appendix~\ref{sec:app_exps_gridworld}, to analyze the corrective feedback induced by DisCor and evaluate the effect of the approximations used in our method, such as the upper bound estimator $\Delta_k$. We first study the setting without sampling error, where all transitions in the MDP are added to the training buffer, and then consider the setting with sampling error, where the algorithm also needs to explore the environment and collect transitions.

\paragraph{Results.} In both settings, shown in Figure~\ref{fig:gridworls_samling}, DisCor consistently provides correct feedback (Figure~\ref{fig:gridworls_samling}(a)). An oracle version of the algorithm (labeled DisCor (oracle); Equation~\ref{eqn:optimal_wk}), which uses the true error $|Q_k - Q^*|$ in place of $\Delta_k$, is somewhat better than DisCor (Figure~\ref{fig:gridworls_samling}(b) histograms, red vs blue), but DisCor still outperforms on-policy and replay buffer sampling schemes, which often suffer from an absence of corrective feedback as shown in Section~\ref{sec:problem_description}.
Prioritizing based on high Bellman error struggles on domains with sparse rewards (App.~\ref{sec:app_exps_gridworld}, Fig.~\ref{fig:app_fig_sampled}). 

\paragraph{Analysis.} While DisCor (oracle) consistently performs better than DisCor, as we would expect, the approximate DisCor method still attains better performance than na\"{i}ve uniform weighting and prioritization similar to PER. This shows that the principle behind DisCor is effective when applied exactly, and that even the approximation that we utilize to make the method practical still improves performance considerably.

\begin{figure}[ht]
    \centering
    \includegraphics[width=0.4\linewidth]{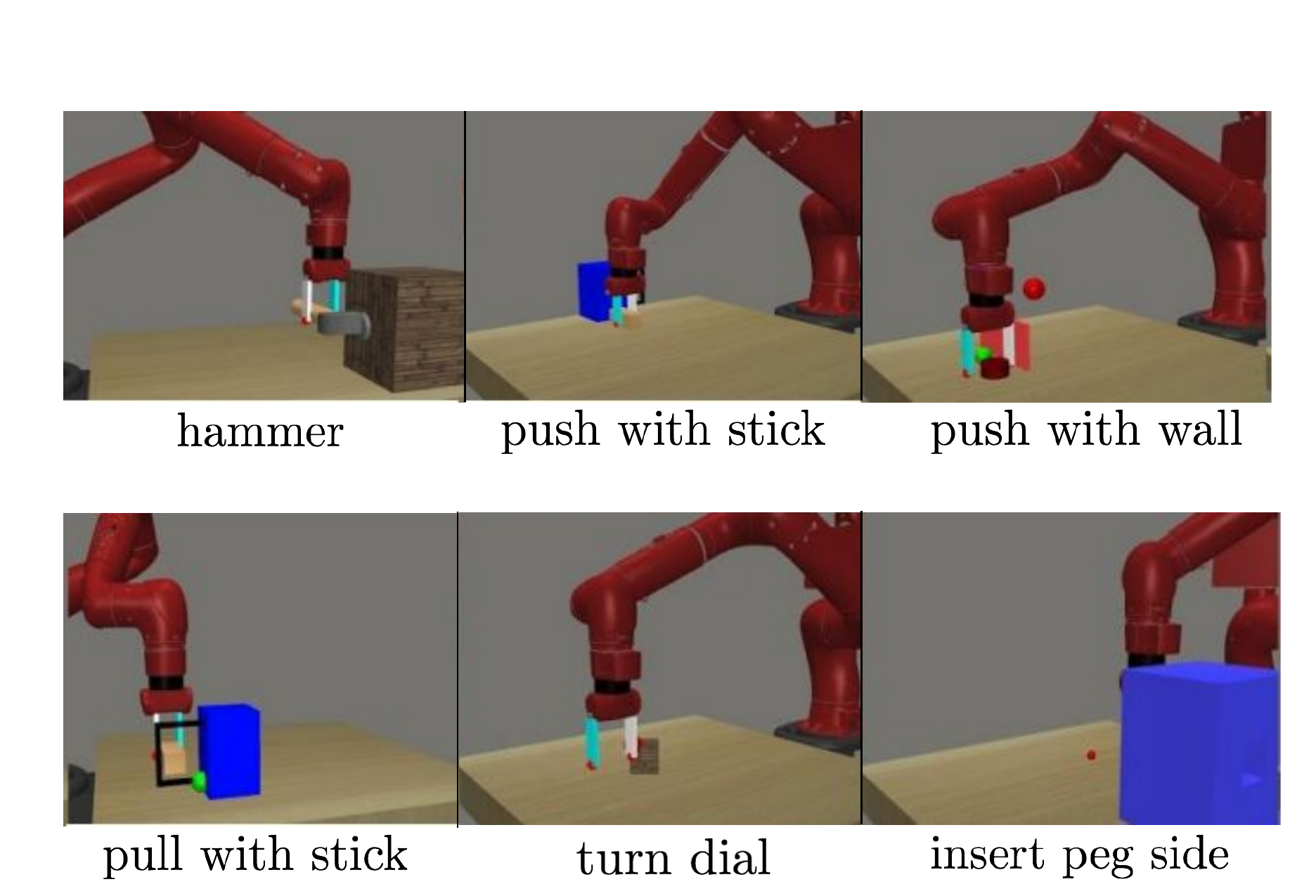}
    \caption{\footnotesize{Visual description of the six MetaWorld tasks used in our experiments in Section~\ref{sec:experiments}. Figures taken from \cite{yu2019meta}.}}
    \label{fig:metaworld_tasks_main}
\end{figure}

\subsection{Continuous Control Experiments}
\label{sec:continuous_control}
We next perform a comparative evaluation of DisCor on several continuous control tasks, using six robotic manipulation tasks from the Meta-World suite (pull stick, hammer, insert peg side, push stick, push with wall and turn dial). A pictorial description of these tasks is provided in Figure~\ref{fig:metaworld_tasks_main}. These domains were chosen because they are challenging for state-of-the-art RL methods, such as soft actor-critic (SAC).

\begin{figure*}[t!]
    \centering
    \includegraphics[width=0.25\linewidth]{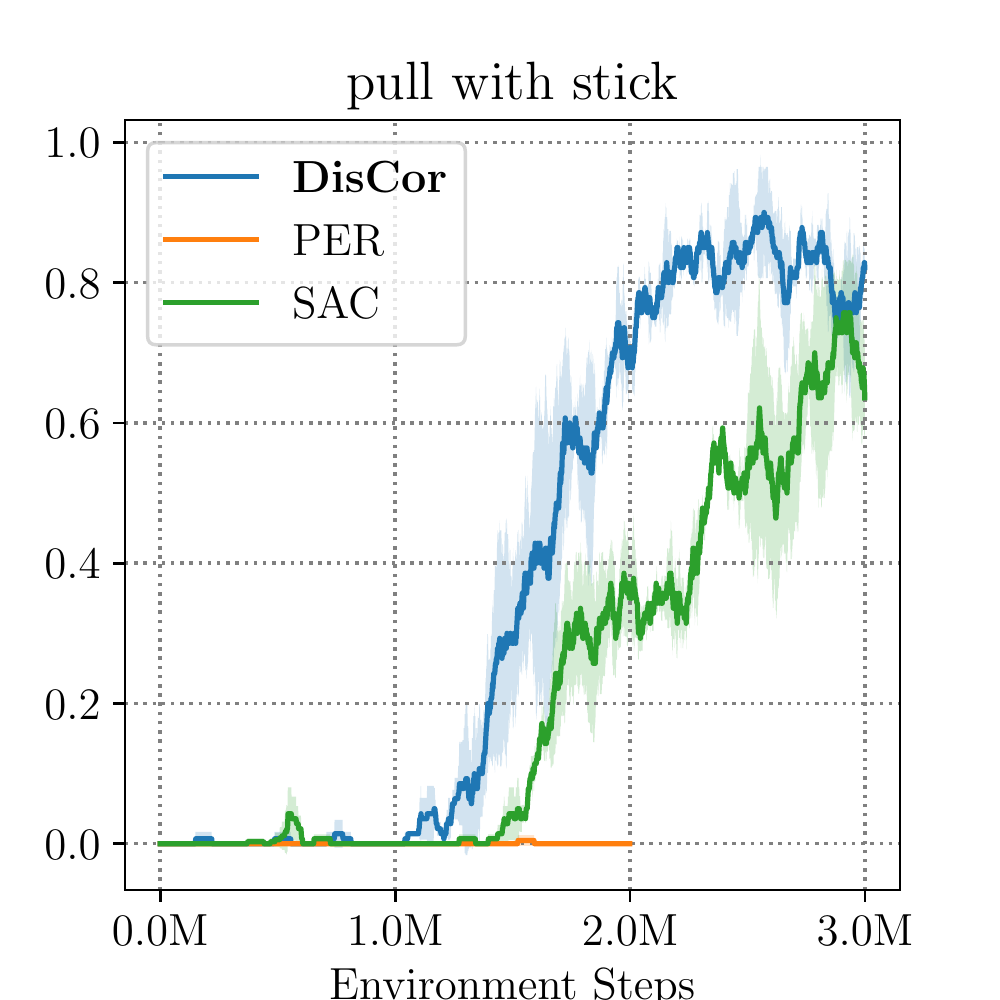}
    \includegraphics[width=0.25\linewidth]{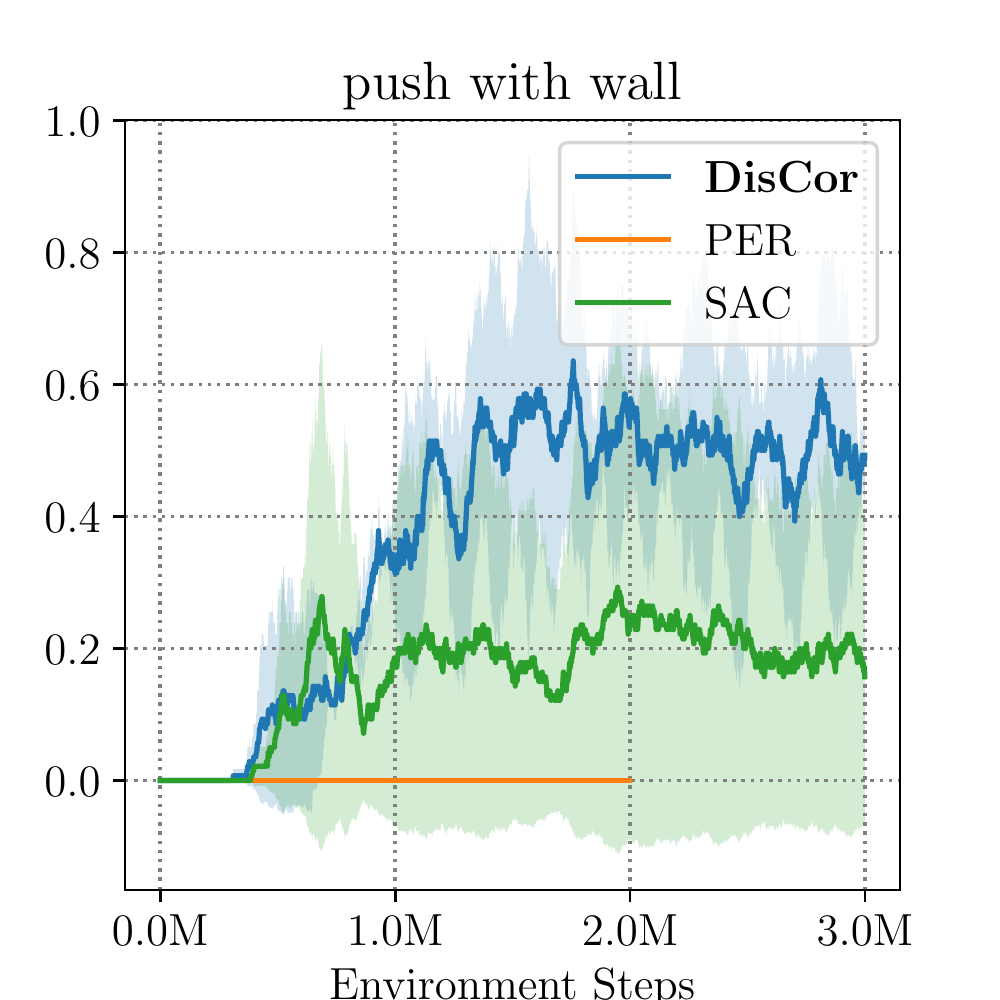}
    \includegraphics[width=0.25\linewidth]{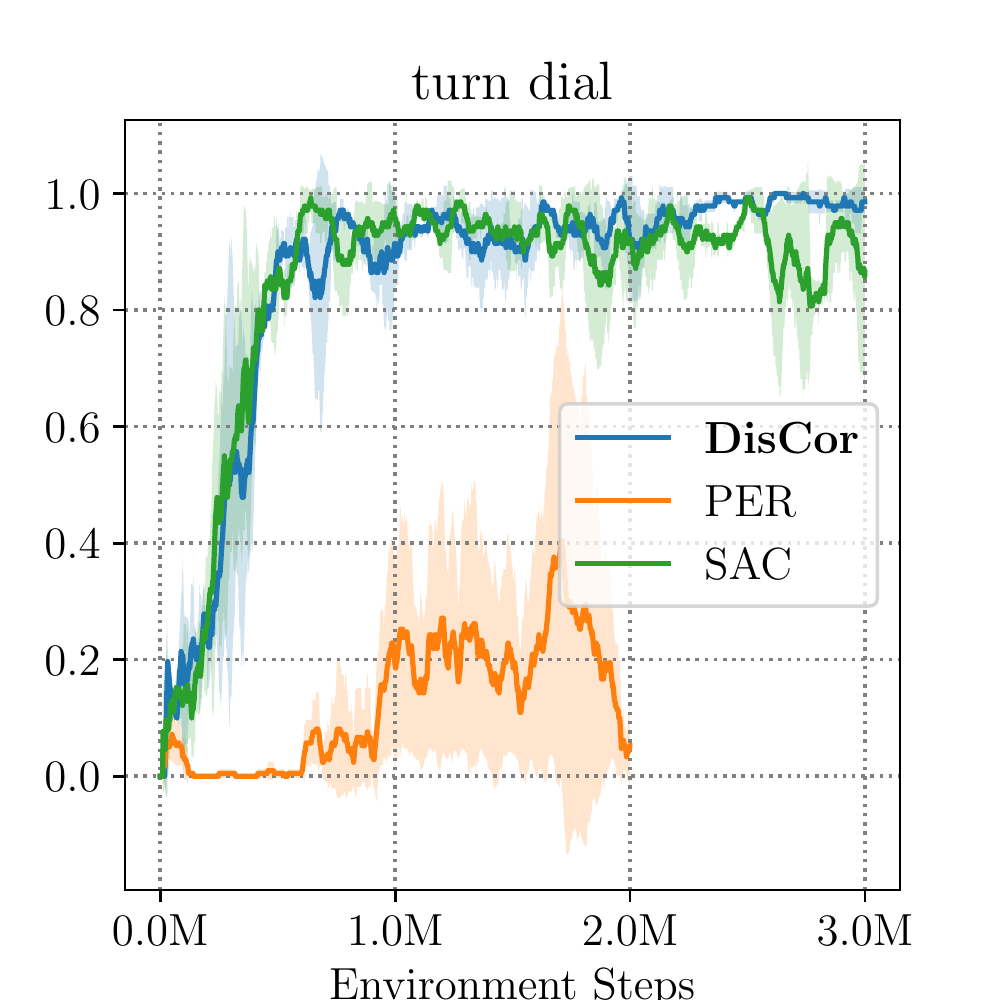} \\
    \includegraphics[width=0.25\linewidth]{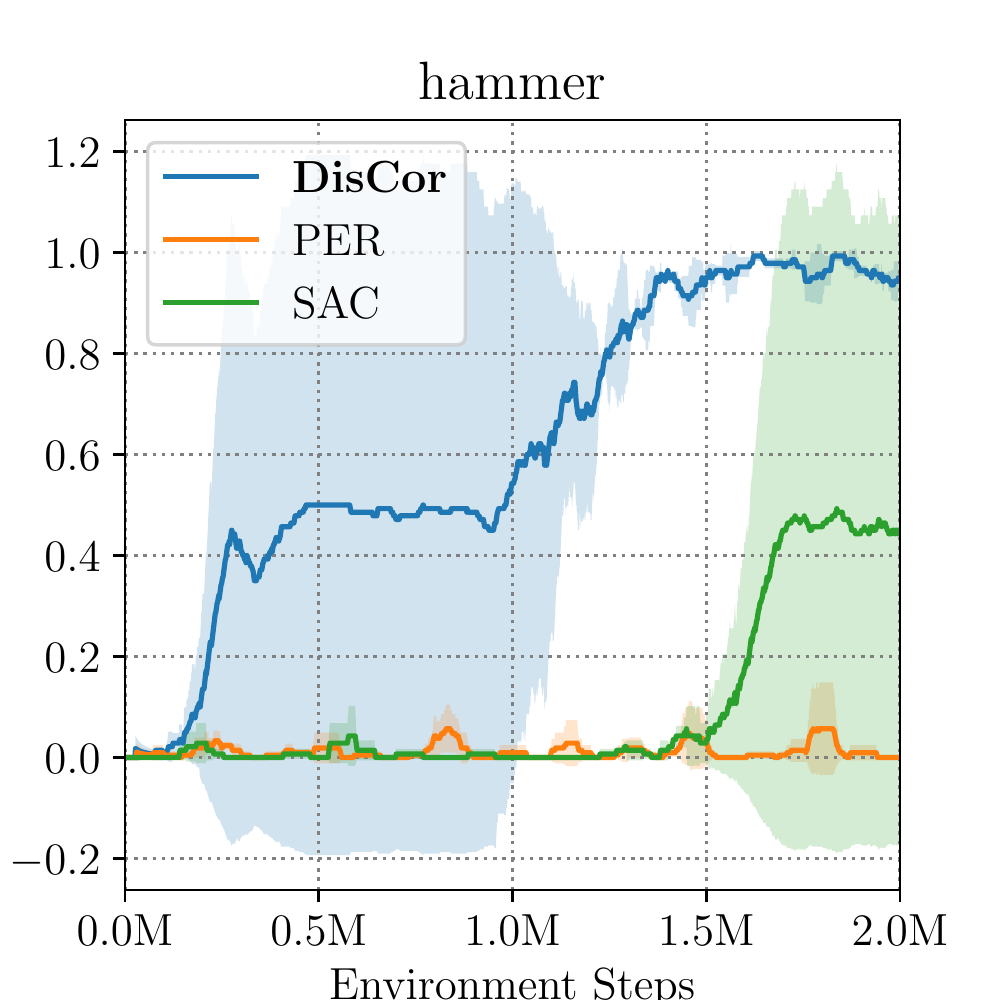}
    \includegraphics[width=0.25\linewidth]{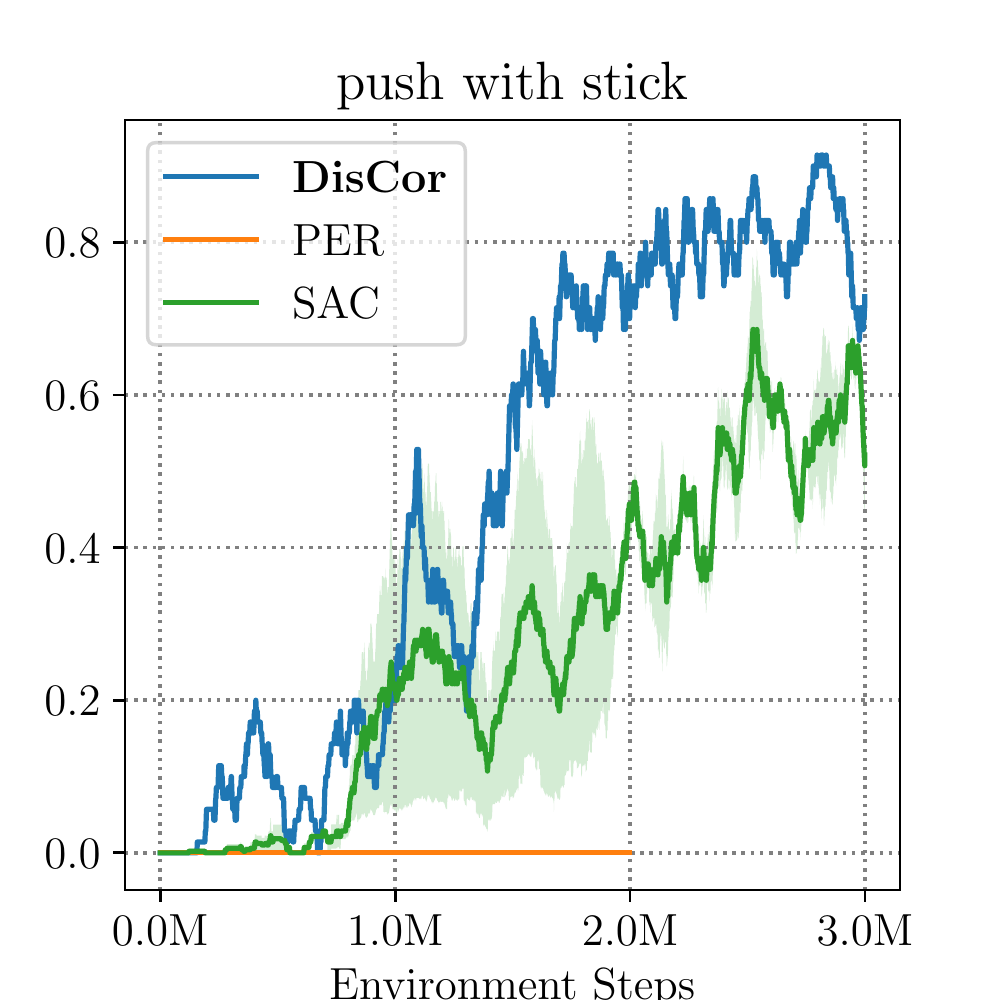}
    \includegraphics[width=0.25\linewidth]{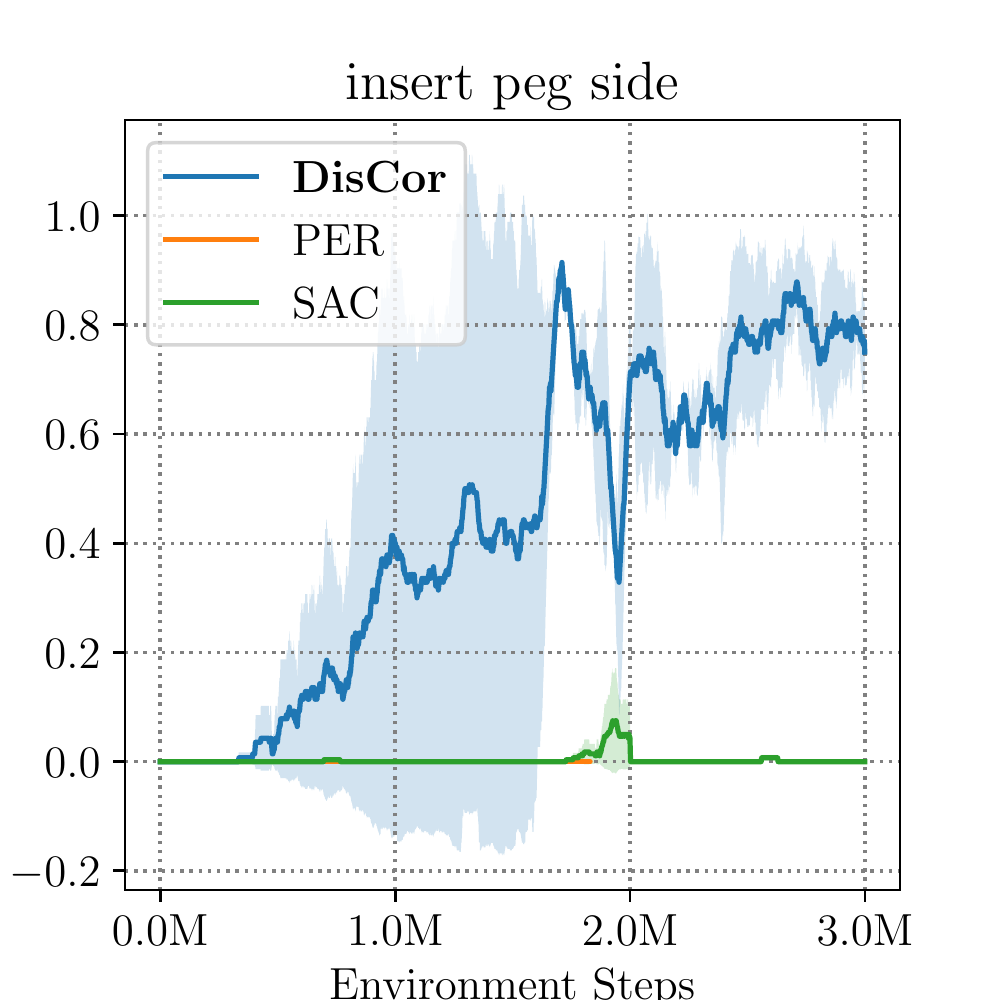}
    \caption{\footnotesize{Evaluation success of DisCor, unweighted SAC and PER on six MetaWorld tasks. From left to right: pull stick, push with wall, push stick, turn dial, hammer and insert peg side. Note that DisCor achieves better final success rates or learns faster on most of the tasks. }}
    \label{fig:success_rates_metaworld}
\end{figure*}

\paragraph{Meta-World.} We applied DisCor to these tasks by modifying the weighting of samples in soft actor-critic (SAC)~\cite{Haarnoja18}. DisCor does not alter any hyperparameter from SAC, and requires minimal tuning.
There is only one additional temperature hyperparameter, which is also automatically chosen across all domains.
More details are presented in Appendix~\ref{sec:app_exp_details}.

We compare DisCor to standard SAC without weighting, as well as prioritized experience replay (PER)~\cite{Schaul2015}, which uses weights based on the last Bellman error. On these tasks, DisCor often attains better performance, as shown in Figures \ref{fig:success_rates_metaworld} and \ref{fig:app_returns_metaworld}, achieving better success rates than standard SAC on several tasks. 
DisCor is also more efficient, achieving good performance earlier than the other methods on these tasks. %
PER, on the other hand, performs poorly, as prioritizing high Bellman error states may lead to higher error accumulation if the target values are incorrect. 

\paragraph{Gym.} We also performed comparisons on the conventional OpenAI gym benchmarks, where we see a small but consistent benefit from DisCor reweighting. Since prior methods, such as SAC already solves these tasks easily, and have been tuned well for them, the room for improvement is very small. We include these results in Appendix~\ref{sec:app_mujoco_benchmarks} for completeness. We also modified the gym environments to have \textit{stochastic} rewards, thus lowering the signal-to-noise ratio, and compare different algorithms on these domains. The results, along with a description of the noise added, are shown in Appendix~\ref{sec:app_mujoco_benchmarks}. In these experiments, DisCor generally exhibits better sample efficiency as compared to other methods.

\vspace{-2pt}
\subsection{Multi-Task Reinforcement Learning}
\label{sec:multi-task}
\vspace{-2pt}
Another challenging setting for current RL methods is the multi-task setting. This is known to be difficult, to the point that oftentimes learning completely separate policies for each of the tasks is actually faster, and resulting in better performance, than learning the tasks together~\citep{yu2019meta,schaul2019ray}. 

\begin{figure}[h]
    \centering
    \begin{subfigure}[t]{.4\linewidth}
        \centering
        \includegraphics[width=0.74\linewidth]{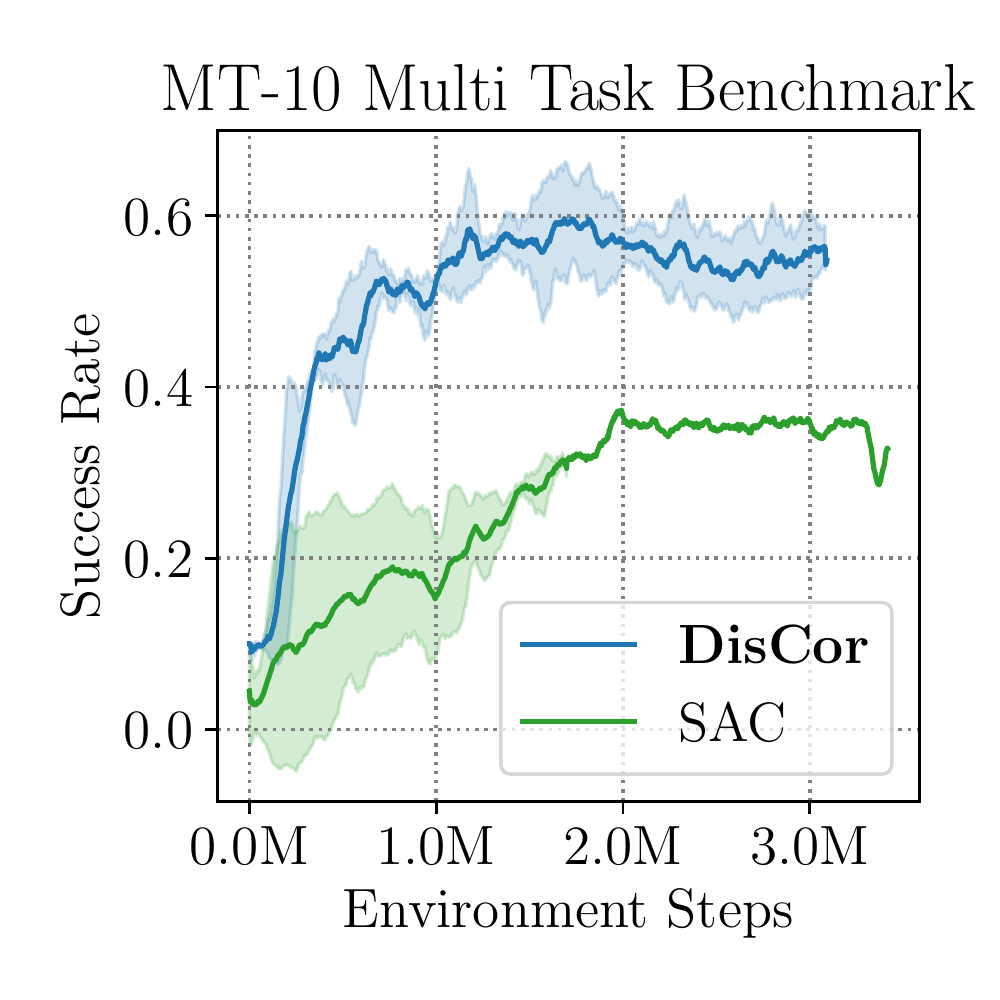}
        \caption{\footnotesize{Average success rate}}
    \end{subfigure}
    ~
    \begin{subfigure}[t]{.4\linewidth}
        \centering
        \includegraphics[width=0.95\linewidth]{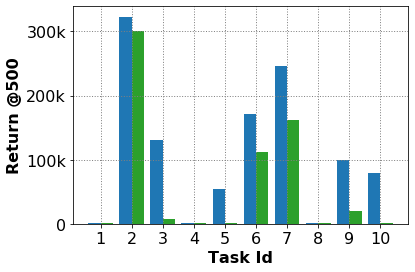}
        \caption{\footnotesize{Per-task return at 500k environment steps}} 
    \end{subfigure}
    \caption{\footnotesize{Performance of DisCor (blue) and unweighted SAC (green) on the MT10 benchmark. We observe that: (1) DisCor outperforms unweighted SAC by a factor of \textbf{1.5} in terms success rate; (2) DisCor achieves a non-trivial return on \textbf{7/10} tasks after 500k environment steps, as compared to \textbf{3/10} for unweighted SAC.}}
    \label{fig:mt_results}
\end{figure}

\paragraph{MT10 multi-task benchmark.} We evaluate on the MT10 MetaWorld benchmark~\citep{yu2019meta}, which consists of ten robotic manipulation tasks to be learned jointly. We follow the protocol from \cite{yu2019meta}, and append task ID to the state. As shown in Figure~\ref{fig:mt_results}(a), DisCor applied on top of SAC outperforms standard unweighted SAC by a large margin, achieving \textbf{50\%} higher success rates compared to SAC, and a high overall return (Fig \ref{fig:app_mt10}). Figure~\ref{fig:mt_results}(b) shows that DisCor makes progress on \textbf{7/10} tasks, as compared to \textbf{3/10} for SAC. 

\paragraph{MT50 multi-task benchmark.} We further compare the performance of DisCor and SAC on the more challenging MT50 benchmark~\citep{yu2019meta}, which is shown in Figure~\ref{fig:app_mt50}, and observe a similar benefit as compared to MT10, where the standard unweighted SAC algorithm tends to saturate at a suboptimal success rate for about 4M environment steps, whereas DisCor continuously keeps learning, and achieves asymptotic performance faster than SAC.

\subsection{Arcade Learning Environment}
\label{sec:atari_exps}
Our final experiments were aimed at testing the efficacy of DisCor on stochastic, discrete-action environments. To this end, we evaluated DisCor on three commonly reported tasks from the Atari suite -- Pong, Breakout and Asterix.
We compare to the baseline DQN~\cite{Mnih2015}, all our implementations are built off of Dopamine~\cite{castro18dopamine} and use the evaluation protocol with sticky actions~\cite{machado18sticky}. We build DisCor on top of DQN by simply replacing the standard replay buffer sampling scheme in DQN with the DisCor weighted update. We show in Figure~\ref{fig:atari_results} that DisCor usually outperforms unweighted DQN in learning speed and performance. 
\begin{figure}[ht]
    \centering
    \includegraphics[width=0.25\linewidth]{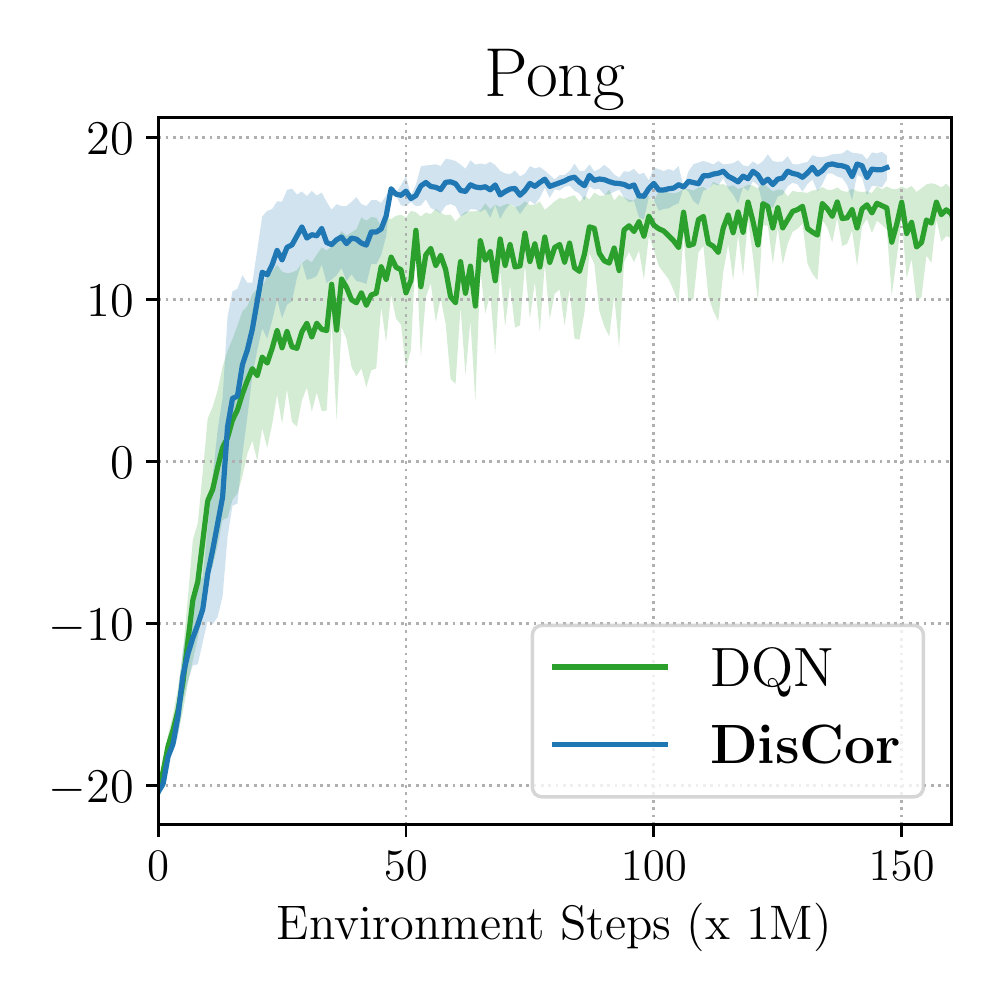}
    \includegraphics[width=0.25\linewidth]{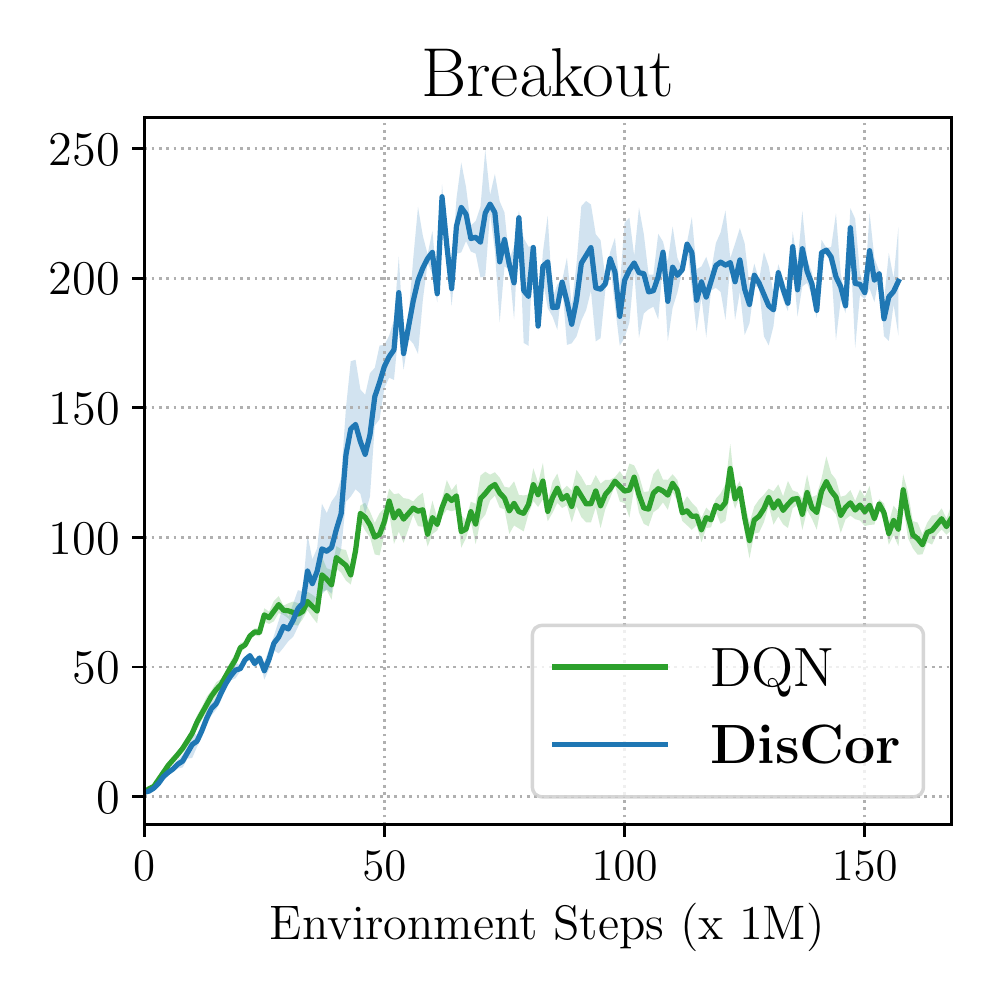}
    \includegraphics[width=0.25\linewidth]{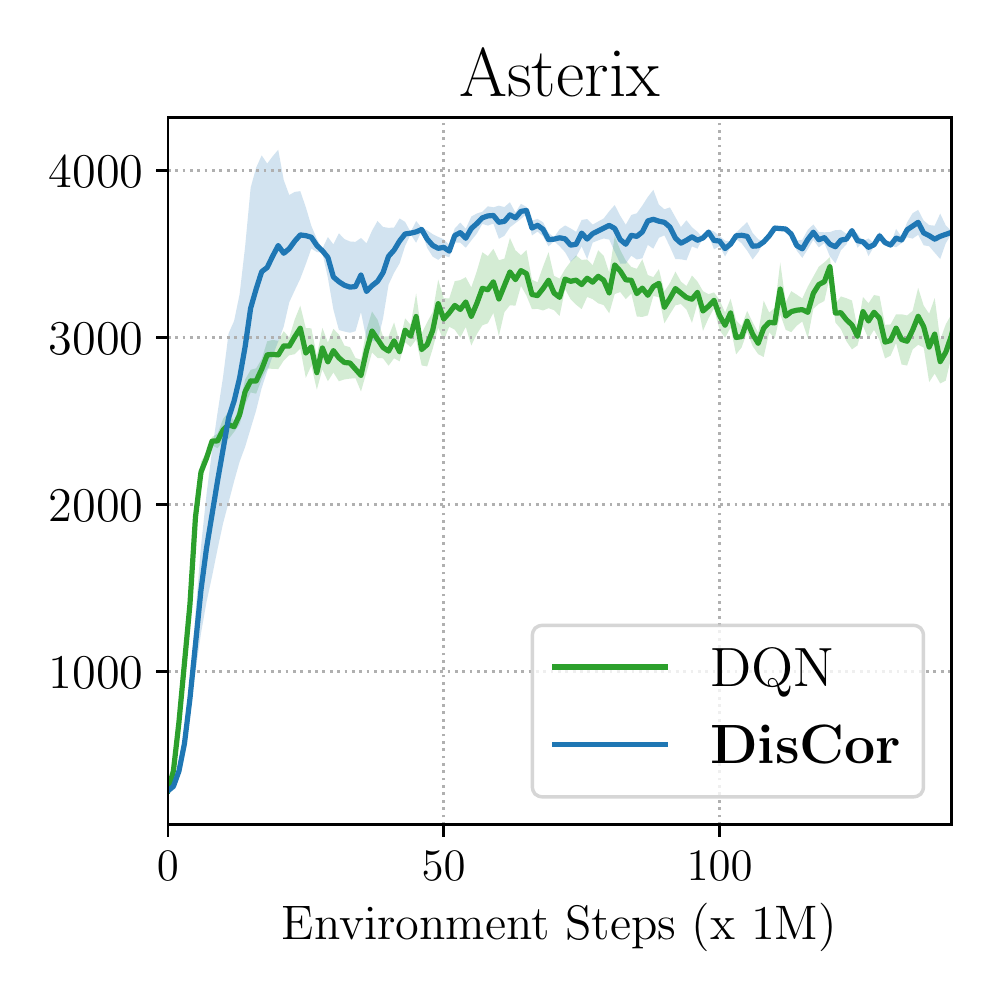}
    \caption{\footnotesize{DQN vs DisCor on Atari. Note that DisCor generally improves learning speed and asymptotic performance on all three tasks that we evaluated.}}
    \label{fig:atari_results}
\end{figure}
\section{Discussion, Future Work and Open Problems}
In this work, we show that deep RL algorithms suffer from the absence of a corrective feedback mechanism in scenarios with na\"ive online data collection. This results in a number of problems during learning, including slow convergence and oscillation. We propose a method to compute the optimal data distribution in order to achieve corrective feedback, and design a practical algorithm, DisCor, that applies this distribution correction by re-weighting the transitions in the replay buffer based on an estimate of the accuracy of their target values. DisCor yields improvements across a wide range of RL problems, including challenging robotic manipulation tasks, multi-task reinforcement learning and Atari games and can be easily combined with a variety of ADP algorithms. 

As shown through our analysis and experiments, studying the connection between data distributions, function approximation, and approximate dynamic programming can allow us to devise stable and efficient reinforcement learning algorithms. This suggests several exciting directions for future work. First, a characterization of the learning dynamics and their interaction with corrective feedback and data distributions in ADP algorithms may lead to even better and more stable algorithms, by better identifying the accuracy of target values. Second, our approach is limited by the transitions that are actually collected by the behavior policy, since it only reweights the replay buffer.
An exciting direction of future investigation would involve studying how we might directly modify the exploration policy to change which transitions are collected, so as to more directly alter the training distribution and produce large gains in sample efficiency and asymptotic performance. 
If we can devise RL methods that are guaranteed to enjoy corrective feedback and thus are stable, robust, and effective, then RL algorithms can be reliably scaled to large open-world settings.

\section*{Acknowledgements}
We thank Xinyang Geng and Aurick Zhou for helpful discussions. We thank Vitchyr Pong, Greg Kahn, Xinyang Geng, Aurick Zhou, Avi Singh, Nicholas Rhinehart, and Michael Janner for feedback on an earlier version of this paper, and all the members of the RAIL lab for their help and support. We thank Tianhe Yu, Kristian Hartikainen, and Justin Yu for help with debugging and setting up the implementations. This research was supported by: the National Science Foundation, the Office of Naval Research, and the DARPA Assured Autonomy program. We thank Google, Amazon and NVIDIA for providing compute resources.

\bibliography{main}
\bibliographystyle{unsrt}

\clearpage
\appendix
\part*{Appendices}

\section{Detailed Proof of Theorem~\ref{eqn:theorem_solution} (Section~\ref{sec:theoretical_derivation})}
\label{sec:missing_proof_steps}
In this appendix, we present a detailed proofs for the theoretical derivation of DisCor outlined in Section~\ref{sec:theoretical_derivation}. To get started, we mention the optimization problem being solved for convenience.
\begin{align}
    \begin{split}
    \min_{p_k}~~& \expec_{d^{\pi_{k}}} \left[ |Q_k - Q^*| \right] \\
    \text{~~s.t.~~} & Q_k = \arg\min_{Q} \expec_{p_k} \left[ (Q - \bellmanopt Q_{k-1})^2 \right]. 
    \end{split}
    \label{eqn:optimization_app}
\end{align}

We break down this derivation in steps marked as relevant paragraphs. The first step is to decompose the objective into a more tractable one via an application of the Fenchel-Young inequality~\cite{rockafellar-1970a}.

\paragraph{Step 1: Fenchel-Young Inequality.} The optimization objective in Problem~\ref{eqn:optimization_app} is the inner product of $d^{\pi_{k-1}}$ and $|Q_k - Q^*|$. We can decompose this objective by applying the Fenchel-Young inequality~\cite{rockafellar-1970a}. For any two vectors, $x, y \in \mathbb{R}^d$, and any convex function $f$ and its Fenchel conjugate $f^*$, we have that, $x^T y \leq f(x) + f^*(y)$. We therefore have:
\begin{equation}
    \label{eqn:convex_conjugate_app}
    \expec_{d^{\pi_k}} \left[ |Q_k - Q^*| \right] \leq f\left( |Q_k - Q^*| \right) + f^* \left( d^{\pi_k} \right) .
\end{equation} 
Since minimizing an upper bound leads to minimization of the original objective, we can replace the objective in Problem~\ref{eqn:optimization_app} with the upper bound in Equation~\ref{eqn:convex_conjugate_app}. As we will see below, a convenient choice of $f$ is the \emph{soft-min} function:
\begin{equation}
    f(x) = - \log \big( \sum_{i} e^{-x_i} \big), ~~ f^*(y) = \mathcal{H}(y).
    \label{eqn:app_f_and_conjugate}
\end{equation}
$f^*$ in this case is given by the Shannon entropy, which is defined as $\mathcal{H}(y) = - \sum_{j} y_j \log y_j$. Plugging this back in problem~\ref{eqn:optimization_app}, we obtain an objective that dictates minimization of the marginal state-action entropy of the policy $\pi$. 

In order to make this objective even more convenient and tractable, we upper bound the Shannon entropy, $\mathcal{H}(y)$ by the entropy of the uniform distribution over states and actions, $\mathcal{H}(\mathcal{U})$. This step ensures that the entropy of the state-action marginal $d^{\pi}$ is not reduced drastically due to the choice of $p$. 
We can now minimize this upper bound, since minimizing an upper bound, leads to a minimization of the original problem, and therefore, we obtain the following new optimization problem shown in Equation~\ref{eqn:final_optimization_app} is:
\begin{align}
\begin{split}
    \label{eqn:final_optimization_app}
    \min_{p_k}~~~& - \log \left( \sum_{s, a} \exp({- |Q_k - Q^*|(s, a)}) \right) \\
    \text{~~s.t.~~} & Q_k = \arg\min_{Q} \expec_{p_k} \left[ (Q - \bellmanopt Q_{k-1})^2 \right].
\end{split}
\end{align}
Another way to interpret this step is to modify the objective in Problem~\ref{eqn:optimization_app} to maximize entropy-augmented $\valuefeedback$: $\valuefeedback(k) + \mathcal{H}(d^{\pi_k})$ as is common in a number of prior works, albeit with entropy over different distributions such as \cite{hazan2019maxent,Haarnoja18}. This also increases the smoothness of the loss landscape, which is crucial for performance of RL algorithms~\citep{ahmed19understanding}.

\paragraph{Step 2: Computing the Lagrangian.} In order to solve optimization problem Problem~\ref{eqn:final_optimization_app}, we follow standard procedures for finding solutions to constrained optimization problems. We first write the Lagrangian for this problem, which includes additional constraints to ensure that $p_k$ is a valid distribution:
\begin{equation}
    \mathcal{L}(p_k; \lambda, \mu) =  - \log \left( \sum_{s, a} \exp({- |Q_k - Q^*|(s, a)}) \right) + \lambda \left( \sum_{s, a} p_k(s, a) - 1 \right) - \mu^T p_k.
    \label{eqn:lagrangian}
\end{equation}
with constraints ${\sum_{s, a} p_k(s, a) = 1}$ and ${p_k(s, a) \geq 0}$ ($ \forall s, a$) and their corresponding Lagrange multipliers, $\lambda$ and $\mu$, respectively, that ensure $p_k$ is a valid distribution.
An optimal $p_k$ is obtained by setting the gradient of the Lagrangian with respect to $p_k$ to 0. This requires computing the gradient of $Q_k$, resulting from Bellman error minimization, i.e. computing the derivative through the solution of another optimization problem,  with respect to the distribution $p_k$. We use the implicit function theorem (IFT)~\citep{Krantz2002TheIF} to compute this gradient. We next present an application of IFT in our scenario.

\paragraph{Step 3: IFT gradient used in the Lagrangian.} We derive an expression for $\frac{\partial Q_k}{\partial p_k}$ which will be used while computing the gradient of the Lagrangian Equation~\ref{eqn:lagrangian} which involves an application of the implicit function theorem. The IFT gradient is given by: 
\begin{equation}
    \label{eqn:ift_gradient}
    \small
    \frac{\partial Q_k}{\partial p_k} \Bigg|_{Q_k, p_k} = - \left[ \diag (p_k)\right]^{-1} \left[ \diag (Q_k - \bellmanopt Q_{k-1}) \right].
\end{equation}
To get started towards showing Equation~\ref{eqn:ift_gradient}, we consider a non-parametric representation for $Q_k$ (a table), so that we can compute a tractable term without going onto the specific calculations for Jacobian or inverse-Hessian vector products for different parametric models. In this case, the Hessians in the expression for IFT and hence, the implicit gradient are given by:
\begin{multline}
    \label{eqn:hessians_regular}
    \quad \quad \quad \quad \quad H_{Q} = 2 ~\diag(p_k) ~~~ \quad H_{Q, p_k} = 2 ~ \diag(Q_k - \bellmanopt Q_{k-1}) \\
    \frac{\partial Q_k}{\partial p_k} = - \left[ H_Q \right]^{-1} H_{Q, p_k} = - \diag\left( \frac{Q_k - \bellmanopt Q_{k-1}}{p_k} \right) \quad \quad \quad \quad.
\end{multline}
provided $p_k \geq 0$ (which is true, since we operate in a full coverage regime, as there is no exploration bottleneck when all transitions are provided). This quantity is $0$ only if the Bellman residuals $Q_k - \bellmanopt Q_{k-1}$ are 0, however, that is rarely the case, hence this gradient is non-zero, and intuitively quantifies the right relationship: Bellman residual errors $Q_k - \bellmanopt Q_{k-1}$ should be higher at state-action pairs with low density $p_k$, indicating a roughly inverse relationship between the two terms -- which is captured by the IFT term.

\paragraph{Step 4: Computing optimal $p_k$.} Now that we have the equation for IFT (Equation~\ref{eqn:ift_gradient}) and an expression for the Lagrangian (Equation~\ref{eqn:lagrangian}), we are ready to compute the optimal $p_k$ via an application of the KKT conditions. At an optimal $p_k$, we have, 
\begin{equation}
    \frac{\partial \mathcal{L}(p_k; \lambda, \mu)}{\partial p_k} = 0 \implies
    \frac{ \mathrm{sgn}(Q_k - Q^*) \exp({- |Q_k - Q^*|(s, a)})}{\sum_{s', a'}\exp({- |Q_k - Q^*|(s', a')})} \cdot \frac{\partial Q_k}{\partial p_k} + \lambda - \mu_{s, a} = 0.
    \label{eqn:lagrangian_simplify}
\end{equation}
Now, re-arranging Equation~\ref{eqn:lagrangian_simplify} and plugging in the expression for $\frac{\partial Q_k}{\partial p_k}$ from Equation~\ref{eqn:ift_gradient} in this Equation to obtain an expression for $p_k(s, a)$, we get:
\begin{equation}
    \label{eqn:optimal_p_app}
    p_k(s, a) \propto \exp\left(-|Q_k - Q^*|(s, a)\right) \frac{|Q_k - \bellmanopt Q_{k-1}|(s, a)}{\lambda^*}.
\end{equation}
Provided, each component of $p$ is positive, i.e. $p_k(s, a) \geq 0$ for all $s, a$, the optimal dual variable $\mu^*_{s, a} = 0$, satisfies $\mu^*(s, a) p_k(s, a) = 0$ by KKT-conditions, and $\mu^* \geq 0$ (since it is a Lagrange dual variable), thus implying that $\mu^* = \mathbf{0}$.

Intuitively, the expression in Equation~\ref{eqn:optimal_p_app} assigns higher probability to state-action tuples with high Bellman error $|Q_k - \bellmanopt Q_{k-1}|$, but only when the \textit{post-update} Q-value $Q_k$ is close to $Q^*$. Hence we obtain the required theorem. 
\paragraph{Summary of the derivation.} To summarize, our derivation for the optimal $p_k$ consists of the following key steps:
\begin{itemize}
    \item Use the Fenchel-Young inequality to get a convenient form for the objective.
    \item Compute the Lagrangian, and use the implicit function theorem to compute gradients of the Q-function $Q_k$ with respect to the distribution $p_k$.
    \item Compute the expression for optimal $p_k$ by setting the Lagrangian gradient to $0$. 
\end{itemize}

\section{Proofs for Tractable Approximations in Section~\ref{sec:minimax}}
\label{app:other_proofs}
Here we present the proofs for the arguments behind each of the approximations described in Section~\ref{sec:minimax}.
\paragraph{Computing weights $w_k$ for re-weighting the buffer distribution, $\mu$.} 
Since sampling directly from $p_k$ may not be easy, we instead choose to re-weight samples transitions drawn from a replay buffer $\mu$, using weights $w_k$ to make it as close to $p_k$. How do we obtain the exact expression for $w_k(s, a)$? One option is to apply importance sampling: choose $w_k$ as the importance ratio, $w_k(s, a) = \frac{p_k(s, a)}{\mu(s, a)}$, however this suffers from two problems -- (1) importance weights tend to exhibit high variance, which can be detrimental for stochastic gradient methods; and (2) densities $\mu(s, a)$, needed to compute $w_k$ are unknown. 

In order to circumvent these problems, we solve a different optimization problem, shown in Problem~\ref{ref:eqn_bias_variance_wk} to find the optimal \textit{surrogate} projection distribution $q_k$, which is closest to $p_k$, and at the same time closest to $\mu$ as well, trading off these quantities by a factor $\tau - 1$.
\begin{equation}
\label{ref:eqn_bias_variance_wk}
   q^*_k = \arg \min_{q_k} \kldiv(q_k||p_k) + (\tau - 1) \kldiv(q_k||\mu)
\end{equation}
where $\lambda$ is a temperature hyperparameter that trades of bias and variance. The solution to the above optimization is shown in Equation~\ref{ref:eqn_final_importance_weights},
where the second statement follows by using a tractable approximation of setting $\mu^{1 - \frac{1}{\tau}}$ to be equal to $\mu$, which can be ignored if $\tau$ is large, hence $1 - \frac{1}{\tau} \approx 1$. The third statement follows by an application of Equation~\ref{eqn:optimal_p_app} and the fourth statement denotes the importance ratio, $\frac{q_k(s, a)}{\mu_k(s, a)}$, as the weights $w_k$. 
\begin{align}
\begin{split}
\label{ref:eqn_final_importance_weights}
    q^*_k(s, a) \propto & \left(\mu_k\right)^{1 - \frac{1}{\tau}} \cdot \exp\left(\frac{\log p_k(s, a)}{\tau}\right) \\
    \therefore q^*_k(s, a) \propto & \left(\mu_k\right) \cdot \exp\left(\frac{\log p_k(s, a)}{\tau}\right) \quad \quad \quad \quad \quad \quad \quad \\
    \therefore \frac{q^*_k}{\mu_k} \propto & \exp\left(\frac{-|Q_k - Q^*|(s, a)}{\tau}\right) \frac{|Q_k - \bellmanopt Q_{k-1}|(s, a)}{\lambda^*}\\
    \therefore w_k \propto & \exp\left(\frac{-|Q_k - Q^*|(s, a)}{\tau}\right) \frac{|Q_k - \bellmanopt Q_{k-1}|(s, a)}{\lambda^*}.
\end{split}
\end{align}

Our next proof justifies the usage of the estimate $\Delta_k$, which is a worst-case upper bound on $|Q_k - Q^*|$ in Equation~\ref{ref:eqn_final_importance_weights}.

\paragraph{Proof of Theorem~\ref{thm:delta_k_upper_bound}.} 
We now present a Lemma~\ref{thm:worst_case_estimator} which proves a recursive inequality for $|Q_k - Q^*|$, then show that the corresponding recursive estimator upper bounds $|Q_k - Q^*|$ pointwise in Lemma~\ref{thm:intermediate}, and then finally show that our chosen estimator $\Delta_k$ is equivalent to this recursive estimator in Theorem~\ref{thm:alpha} therefore proving Theorem~\ref{thm:delta_k_upper_bound}.
\begin{lemma}
\label{thm:worst_case_estimator}
For any $k \in \mathbb{N}$, $|Q_k - Q^*|$ satisfies the following recursive inequality, pointwise for each $s, a$:
\begin{equation*}
|Q_k - Q^*| \leq |Q_k - \bellmanopt Q_{k-1}| + \gamma P^{\pi_{k-1}} |Q_{k-1} - Q^*| + \frac{2 R_{\text{max}}}{1 - \gamma} \max_{s} \tvd(\pi_{k-1}, \pi^*).    
\end{equation*}
\end{lemma}
\begin{proof}
Our proof relies on a worst-case expansion of the quantity $|Q_k - Q^*|$. The proof follows the following steps. The first few steps follow common expansions/inequalities operated upon in the work on error propagation in Q-learning~\cite{munos2005error}.
\begin{align*}
\begin{split}
    \label{eqn:proof_minimax_error}
    |Q_k - Q^*| \stackrel{(a)}{=} & |Q_k - \bellmanopt Q_{k-1} + \bellmanopt Q_{k-1} - Q^*| \\
    \stackrel{(b)}{\leq} & |Q_{k} - \bellmanopt Q_{k-1}| + |\bellmanopt Q_{k-1} - \bellmanopt Q^*| \\
    \stackrel{(c)}{=} & |Q_k - \bellmanopt Q_{k-1}| + |R + \gamma P^{\pi_{k-1}} Q_{k-1} - R - \gamma P^{\pi^*} Q^*| \\
    \stackrel{(d)}{=} & |Q_ k - \bellmanopt Q_{k-1}| + \gamma |P^{\pi_{k-1}} Q_{k-1} - P^{\pi_{k-1}} Q^* + P^{\pi_{k-1}} Q^* - P^{\pi^*}Q^* | \\ 
    \stackrel{(e)}{\leq} & |Q_k - \bellmanopt Q_{k-1} | + \gamma P^{\pi_{k-1}} |Q_{k-1} - Q^*| + \gamma |P^{\pi_{k-1}} - P^{\pi^*}| |Q^*| \\
    \stackrel{(f)}{\leq} &  |Q_k - \bellmanopt Q_{k-1}| + \gamma P^{\pi_{k-1}} |Q_{k-1} - Q^*| + \frac{2 R_{\max}}{1 - \gamma} \max_s \tvd(\pi_{k-1}, \pi^*)
\end{split}
\end{align*}
where (a) follows from adding and subtracting $\bellmanopt Q_{k-1}$, (b) follows from an application of triangle inequality, (c) follows from the definition of $\bellmanopt$ applied to two different Q-functions, (d) follows from algebraic manipulation, (e) follows from an application of the triangle inequality, and (f) follows from bounding the maximum difference in transition matrices $|P^{\pi_{k-1}} - P^*|$ by maximum total variation divergence between policy $\pi_{k-1}$ and $\pi^*$, and bounding the maximum possible value of $Q^*$ by $\frac{R_{\max}}{1 - \gamma}$.
\end{proof}

We next show that an estimator that satisfies the recursive equality corresponding to Lemma~\ref{thm:worst_case_estimator} is a pointwise upper bound on $|Q_k - Q^*|$.

\begin{lemma}
\label{thm:intermediate}
For any $k \in \mathbb{N}$, an vector $\Delta'_k$ satisfying
\begin{equation}
    \label{eqn:delta_prime_recursion}
    \Delta'_k := |Q_k - \bellmanopt Q_{k-1}| + \gamma P^{\pi_{k-1}} \Delta'_{k-1}.
\end{equation}
with $\alpha_k = \frac{2R_{\max}}{1 - \gamma} \max_s \tvd(\pi_k, \pi^*)$, and an initialization $\Delta'_0 := |Q_0 - Q^*|$, pointwise upper bounds $|Q_k - Q^*|$ with an offset depending on $\alpha_i$, i.e. $\Delta'_k + \sum_{i} \alpha_i \gamma^{k-i} \geq |Q_k - Q^*|$.
\end{lemma}
\begin{proof}
Let $\Delta'_k$ be an estimator satisfying Equation~\ref{eqn:delta_prime_recursion}. In order to show that $\Delta'_k + \sum_i \gamma^{k-i} \alpha_i \geq |Q_k - Q^*|$, we use the principle of mathematical induction. The base case, $k = 0$ is satisfied, since $\Delta'_0 + \alpha_0 \geq |Q_0 - Q^*|$. Now, let us assume that for a given $k = m$, $\Delta'_{m} + \sum_{i} \gamma^{m-i} \alpha_i \geq |Q_m - Q^*|$ pointwise for each $(s, a)$. Now, we need to show that a similar relation holds for $k = m+1$, and then we can appeal to the principle of mathematical induction to complete the argument. In order to show this, we note that,
\begin{align}
    \label{eqn:step1}
    \Delta'_{m+1} = & |Q_{m+1} - \bellmanopt Q_m| + \gamma P^{\pi_{m}} \Delta'_{m} + \sum_{i}^{m+1} \gamma^{m+1-i} \alpha_i\\
    \label{eqn:step2}
    = & |Q_{m+1} - \bellmanopt Q_m| + \gamma P^{\pi_m} (\Delta'_{m} + \sum_{i=0}^{m} \gamma^{m-i} \alpha_i ) + \alpha_{m+1}\\
    \label{eqn:step3}
    \geq & |Q_{m+1} - \bellmanopt Q_m| + \gamma P^{\pi_m} |Q_m - Q^*| + \alpha_m  \quad \quad\\
    \label{eqn:step4}
    \geq & |Q_{m+1} - Q^*| \quad \quad \quad \quad \quad \quad \quad 
\end{align}
where (\ref{eqn:step1}) follows from the definition of $\Delta'_k$, (\ref{eqn:step2}) follows by rearranging the recursive sum containing $\alpha_i$, for $i \leq m$ alongside $\Delta_{m}$, (\ref{eqn:step3}) follows from the inductive hypothesis at $k = m$, and (\ref{eqn:step4}) follows from Lemma~\ref{thm:worst_case_estimator}.

Thus, by using the principle of mathematical induction, we have shown that $\Delta'_k + \sum_i \gamma^{k-i} \alpha_i \geq |Q_k - Q^*|$ pointwise for each $s, a$, for every $k \in \mathbb{N}$.
\end{proof}

The final piece in this argument is to show, that the estimator $\Delta_k$ used by the DisCor algorithm (Algorithm~\ref{alg:discor}), which is initialized randomly, i.e. not initialized to $\Delta_0 = |Q_0 - Q^*|$, still satisfies the property from Lemma~\ref{thm:intermediate}, possibly for certain $k \in \mathbb{N}$.

Therefore, we now show why: $\Delta_k + \sum_{i=1}^k \alpha_i \gamma^{k-i} \geq |Q_k - Q^*|$ point-wise for a sufficiently large $k$. We restate a slightly modified version of Theorem~\ref{thm:delta_k_upper_bound} for convenience. 
\begin{theorem}[Formal version of Theorem~\ref{thm:delta_k_upper_bound}]
\label{thm:alpha}
For a sufficiently large $k \geq k_0 = \frac{\log (1 - \gamma)}{\log \gamma}$, the error estimator $\Delta_k$ pointwise  satisfies:
$$ \Delta_k + \sum_{i=0}^k \gamma^i \alpha_{k-i} \geq |Q_k - Q^*| $$
where $\alpha_i$'s are scalar constants independent of any state-action pair. \textit{(Note that Theorem~\ref{thm:delta_k_upper_bound} has a typo $\gamma^i$ instead of $\gamma^{k-i}$, this theorem presents the correct inequality.})
\end{theorem}

\begin{proof}
\paragraph{Main Idea/ Sketch:} As shown in Algorithm~\ref{alg:discor}, the estimator $\Delta_k$ is initialized randomly, without taking into account $|Q_0 - Q^*|$. Therefore, in this theorem, we want to show that \textit{irrespective} of the initialization of $Q_0$, a randomly initialized $\Delta_k$ eventually satisfies the inequality shown in Theorem~\ref{thm:delta_k_upper_bound}. Now, we present the formal proof. \\

Consider $k_0$ to be the smallest $k$, such that the following inequality is satisfied:
\begin{equation}
    \label{eqn:smallest_k_inequality}
    \gamma^k \max_{Q_0, Q^*} |Q_0 - Q^*| \leq 1
\end{equation}
Thus, $k_0 \geq \frac{\log (1 - \gamma)}{\log \gamma}$, assuming $R_{\max} = 1$ without loss of generality. For a different reward scaling, the bound can be scaled appropriately. To see this, we substitute $|Q_0 - Q^*|$ as an upper-bound $R_{\max}/ (1 - \gamma)$, and  bound $R_{\max}$ by $1$. 

Let $\Delta'_k$ correspond to the upper-bound estimator as derived in Lemma~\ref{thm:intermediate}. For each $k \geq k_0$, the contribution of the initial error $|Q_0 - Q^*|$ in $|Q_k - Q^*|$ is upper-bounded by $1$, and gradually decreases with a rate $\gamma$ as more backups are performed, i.e., as $k$ increases. Thus we can construct another sequence $\Delta_1, \cdot, \Delta_k, \cdots$ which chooses to ignore $|Q_0 - Q^*|$, and initializes $\Delta_0 = 0$ (or randomly) and the sequences $\Delta$ and $\Delta'_k$ satisfy:
\begin{equation}
    \label{eqn:sequence_convergence}
    |\Delta'_k - \Delta_k| < 1, \forall k \geq k_0 
\end{equation}
Furthermore, the difference $|\Delta'_k - \Delta_k|$ steadily shrinks to 0, with a linear rate $\gamma$, so the overall contribution of the initialization sub-optimality $|Q_0 - Q^*|$ drops linearly with a rate of $\gamma$. Hence, $\Delta'$ and $\Delta$ converge to the same sequence beyond a fixed $k = k_0$. Since $\Delta'_k$ is computed using the RHS of Lemma~\ref{thm:worst_case_estimator}, it is guaranteed to be an upper bound on $|Q_k - Q^*|$:
\begin{equation}
    \label{eqn:worst_case_surrogate}
    \left| \left(\Delta_k + \sum_{i=1}^k \gamma^{k-i} \alpha_i \right) - \left(\Delta'_k  + \sum_{i=1} \gamma^{k-i} \alpha_i \right) \right| \leq 1.
\end{equation}
Since, $\Delta'_k + \sum_i \gamma^{k-i} \alpha_i \geq |Q_k - Q^*|$, we get $\forall~ k \geq k_0$, using \ref{eqn:worst_case_surrogate}, that 
\begin{equation}
    \Delta_k + \sum_{i=1}^k \gamma^{k-i} \alpha_i \geq |Q_k - Q^*| - \gamma^{k - k_0}. 
\end{equation}
Hence, $\Delta_k + \sum_{i=1}^k \gamma^{k-i} \alpha_i \geq |Q_k - Q^*|$ for large $k$.
\paragraph{A note on the value of $k_0$.} For a discounting of $\gamma = 0.95$, we get that $k_0 \approx 59$ and for $\gamma = 0.99$, $k_0 \approx 460$. In practical instances, an RL algorithm takes a minimum of about $\geq$ 1M gradient steps, so this value of $k_0$ is easy achieved. Even in the gridworld experiments presented in Section~\ref{sec:gridworlds}, $\gamma = 0.95$, hence, the effects of initialization stayed significant only untill about $59$ iterations during training, out of a total of 300 or 500 performed, which is a small enough percentage.
\end{proof}

\paragraph{Summary of Proof for Theorem~\ref{thm:delta_k_upper_bound}.} $\Delta_k$ in DisCor is given by the quantity $\Delta_k = |Q_k - \bellmanopt Q_{k-1}| + \gamma P^{\pi_{k-1}} \Delta_{k-1}$, is an upper bound for the error $|Q_k - Q^*|$, and we can safely initialize the parametric function $\Delta_\phi$ using standard neural network initialization, since the value of initial error will matter \textit{only} infinitesimally after a large enough $k$. 

As $k \rightarrow \infty$, the following is true:
\begin{align}
\label{eqn:limit}
    \lim_{k \rightarrow \infty} \Big||\Delta_k - |Q_k - Q^*|\Big| \leq \lim_{k \rightarrow \infty} \sum_{i=1}^k \gamma^{k-i} \alpha^i \\
    \label{eqn:limit2}
    = \lim_{k \rightarrow \infty} \sum_{i=1}^k \gamma^{k-i} \tvd(\pi_i, \pi^*) 
\end{align}
Also, note that if $\pi_k$ is improving, i.e. $\pi_k \rightarrow \pi^*$, then, we have that $\tvd(\pi_k, \pi^*) \rightarrow 0$, and since limit of a sum is equal to the sum of the limit, and $\gamma < 1$, therefore, the final inequality in Equation~\ref{eqn:limit2} tends to 0 as $k \rightarrow \infty$.

\section{Missing Proofs From Section~\ref{sec:problem_description}}
\label{sec:omitted_proofs}
In this section, we provide omitted proofs from Section~\ref{sec:problem_description} of this paper. Before going into the proofs, we first describe notation and prove some lemmas that will be useful later in the proofs.

We also describe the underlying ADP algorithm we use as an ideal algorithm for the proofs below.

\begin{algorithm}[H]
\small
\caption{Generic ADP algorithm}
\label{alg:fqi}
\begin{algorithmic}[1]
    \STATE Initialize Q-values $Q_0$.
    \FOR{step $t$ in \{1, \dots, N\}}
        \STATE Collect trajectories using $\pi_t$
        \STATE Choose distribution $D_t$ for projection.
        \STATE $Q_{t+1} \leftarrow \prod_{D_t} \bellmanopt Q_t$ \\ $ \prod_{D_t} \bellmanopt= \arg \min_Q \expec_{D_t}[(Q(s, a) - \bellmanopt Q_{t-1}(s, a))^2]$
    \ENDFOR
\end{algorithmic}
\end{algorithm}

\paragraph{Assumptions.} The assumptions used in the proofs are as follows:
\begin{itemize}
    \item Q-function is linearly represented, i.e. given a set of features, $\phi(s, a) \in \mathbb{R}^d$ for each state and action, concisely represented as the matrix $\Phi \in \mathbb{R}^{|S| |A| \times d}$, Q-learning aims to learn a d-dimensional feature vector $w$, such that $Q(s, a) = w^T \phi(s, a)$. This is not a limiting factor as we will prove in Assumption~\ref{assumption:optimal_features}, for problems with sufficiently large $|S|$ and $|A|$.
\end{itemize}

\subsection{Suboptimal Convergence of On-policy Q-learning}
We first present a result that describes how Q-learning can converge sub-optimally when performed with on-policy distributions, thus justifying our empirical observation of suboptimal convergence with on-policy distributions. 
\begin{theorem}[\cite{farias_fixed_points}]
Projected Bellman optimality operator under the on-policy distribution $\mathcal{H} = \prod_{D_\pi} \bellmanopt$ with a Boltzmann policy, $\pi \propto \exp(Q/\tau)$, where $0 < \tau$ always has one or more fixed points.
\end{theorem}
\begin{proof}
This statement was proven to be true in \cite{farias_fixed_points}, where it was shown the projection operator $\mathcal{H}$ has the same fixed points as another operator, $F_\alpha$ given by:
\begin{equation}
    F_\alpha(x) : = x + \alpha \Phi^T D_{\pi} \left( \bellmanopt \Phi x - \Phi x \right)
\end{equation}
where $\alpha \in (0, 1)$ is a constant. They showed that the operator $F_\alpha$ is a contraction for small-enough $\alpha$ and used a compact set argument to generalize it to other positive values of $\alpha$. We refer the reader to \cite{farias_fixed_points} for further reference.

They then showed a 2-state MDP example (Example 6.1, \cite{farias_fixed_points}) such that the Bellman operator $\mathcal{H}$ has \textbf{2} fixed points, thereby showing the existence of one or more fixed points for the on-policy Bellman backup operator.  
\end{proof}

\subsection{Proof of Theorem~\ref{thm:exponential}}

We now provide an existence proof which highlights the difference in the speeds of learning accurate Q-values from online or on-policy and replay buffer distributions versus a scheme like DisCor. 
We first state an assumption (Assumption~\ref{assumption:optimal_features}) on the linear features parameterization used for the Q-function. This assumption ensures that the optimal Q-function exists in the function class (i.e. linear function class) used to model the Q-function. This assumption has also been used in a number of recent works including \cite{du2020is}. Analogous to \cite{du2020is}, in our proof, we show that this assumption is indeed satisfied for features lying in a space that is logarithmic in the size of the state-space. For this theorem, we present an episodic example, and operate in a finite horizon setting with discounting $\gamma$ and $H$ denotes the horizon length. An episode terminates deterministically as soon as the run reaches a terminal node -- in our case a leaf node of the tree MDP, i.e. a node at level $H-1$ -- as we will see next.

\begin{assumption}
\label{assumption:optimal_features}
There exists $\delta \geq 0$, and $w \in \mathbb{R}^d$, such that for any $(s, a) \in \mathcal{S} \times \mathcal{A}$, the optimal Q-function satisfies: $|Q^*(s, a) - w^T \phi(s,a)| \leq \delta$.
\end{assumption}

We first prove an intermediate property of $\Phi$ satisfying the above assumption that will be crucial for the lower bound argument for on-policy distributions.
\begin{corollary}
\label{cor:epsilon_rank_lemma}
There exists a set of features $\Phi \in \mathbb{R}^{2^H \times O(H^2/\varepsilon^2)}$ satisfying assumption~\ref{assumption:optimal_features}, such that the following holds: $||I_{2^H} - \Phi\Phi^T||_\infty \leq \epsilon$.
\end{corollary}
\begin{proof}
This proof builds on the existence argument presented in \cite{du2020is}. Using the $\varepsilon$-rank property of the identity matrix, one can show that there exists a feature set $\Phi \in \mathbb{R}^{2^H \times O(H/\varepsilon^2)}$ such that $||I_{2^H} - \Phi\Phi^T||_\infty \leq \epsilon$. Thus, we can choose any such $\Phi$, for a sufficiently low threshold $\varepsilon$. In order to assign features $\Phi$ to a state, we can simply perform an enumeration of nodes in the tree via a standard graph search procedure such as depth first search and assign a node $(s, a)$ a feature vector $\phi(s, a)$. To begin with, let's show how we can satisfy assumption~\ref{assumption:optimal_features} by choosing a different weight vector $w_h$ for each level $h$, such that we obtain $|Q_h(s, a) - w_h^T \phi(s, a)| \leq \epsilon$.
Since for each level $h$ exactly one state satisfies $Q^*(s_j, a_j) = \gamma^{H-j+1}$, so we can just let $w_j = \gamma^{H-j +1} \phi(s_j, a_j)$ and thus we are able to satisfy Assumption~\ref{assumption:optimal_features}. This is the extent of the argument used in \cite{du2020is}.

Now we generalize this argument to find a single $w \in \mathbb{R}^d$, unlike different weights $w_h$ for different levels $h$. In order to do this, we create a new $\Phi'$, of size $\Phi' \in \mathbb{R}^{2^H \times O(H^2/\epsilon^2)}$ (note $H^2$ versus $H$ dimensions for $\Phi'$ and $\Phi$) given any $\Phi$ satisfying the argument in the above paragraph, such that
\begin{equation}
    \Phi'(s, a) = \left[\underbrace{0,...,0}_{h \times \dim(\phi(s, a))}, \underbrace{\Phi(s, a)}_{\dim(\phi(s,a))}, 0,...,0 \right] 
    \label{eqn:new_features}
\end{equation}
Essentially, we pad $\Phi$ with zeros, such that for $(s, a)$ belonging to a level $h$, $\Phi'$ is equal to $\Phi$ in the $h-$th, $\dim(\phi(s, a))$-sized block.

A choice of a single $w \in \mathbb{R}^{\dim(\Phi'(s, a))}$ for $\Phi'$ is given by simply concatenating $w_1,\cdots, w_h$ found earlier for $\Phi$.
\begin{equation}
    w = \left[ w_1, w_2, \cdots, w_H \right]
\end{equation}
It is easy to see that $w^T \Phi'$ satisfies assumption~\ref{assumption:optimal_features}. A fact that will be used in the proof for Theorem~\ref{thm:app_exp_lower_bound}, is that this construction of $\Phi'$ also satisfies: $||I_{2^H} - \Phi'\Phi'^T||_\infty \leq \epsilon$.
\end{proof}

We now restate the theorem from Section~\ref{sec:problem_description} and provide a proof below. 
\begin{figure}
    \centering
 \begin{tikzpicture}[->,>=stealth',level/.style={sibling distance = 4cm/#1,
  level distance = 1.5cm}] 
\node [arn_x] {r(s, a) = 0}
    child{ node [arn_x] {r(s) = 0} 
            child{ node [arn_x] {r(s) = 0}} 
            child{ node [arn_x] {r(s) = 0}}
    }
    child{ node [arn_x] {r(s) = 0}
            child{ node [arn_x] {r(s, $a_1$) = 1} }
            child{ node [arn_x] {r(s) = 0} }
		}
; 
\end{tikzpicture}
    \caption{\footnotesize{Example element of the tree family of MDPs used to prove the lower bound in Theorem~\ref{thm:app_exp_lower_bound}. Here, the depth of the tree $H = 2$. $r(s) = 0$ implies that executing any action $a_1$ or $a_2$, a reward of 0 is obtained as state $s$. $(s^*, a^*)$ is given by state marked r(s, $a_1$) = 1.}}
    \label{fig:tree_fig}
\end{figure}
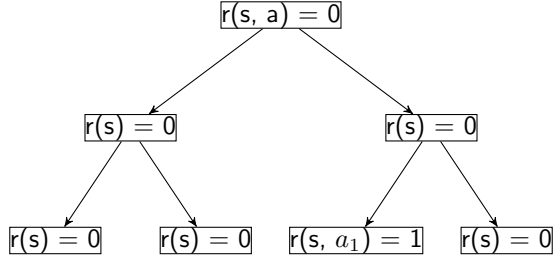

\begin{theorem}[Exponential lower bound for on-policy distributions]
\label{thm:app_exp_lower_bound}
There exists a family of MDPs parameterized by $H > 0$, with $|\mathcal{S}| = 2^H$,  $|\mathcal{A}| = 2$ and a set of features satisfying Assumption~\ref{assumption:optimal_features}, such that on-policy sampling distribution, i.e. $D_k = d^{\pi_k}$, requires $\Omega\left(\gamma^{-H}\right)$ \textit{exact} fixed-point iteration steps in the generic algorithm (Algorithm~\ref{alg:fqi}) for convergence, if at all, the algorithm converges to an $\varepsilon-$accurate Q-function. 
\end{theorem}

\begin{proof}
\paragraph{Tree Construction.} Consider the family of tree MDPs like the one shown in Figure~\ref{fig:tree_fig}. Both the transition function $T$ and the reward function $r$ are deterministic, and there are two actions at each state: $a_1$ and $a_2$. There are $H$ level of states, thereby forming a full binary tree of depth $H$. Executing action $a_1$ transitions the state to its left child int he tree and executing action $a_2$ transitions the state to its right child. There are $2^h$ states in level $h$. Among the $2^{H-1}$ states in level $H-1$, there is one state, $s^*$, such that action $a^*$ at this state yields a reward of $r(s^*, a^*) = 1$. For other states of the MDP, $r(s, a) = 0$. This is a typical example of a sparse reward problem, generally used for studying exploration~\cite{du2020is}, however, we re-iterate that in this case, we are primarily interested in the number of iterations needed to learn, and thereby assume that the algorithm is given infinite access to the MDP, and all transitions are observed, and the algorithm just picks a distribution $D_k$, in this case, the on-policy state-action marginal for performing backups.
 
\paragraph{Main Argument.} Now, we are equipped with a family of the described tree MDPs and a corresponding set of features $\Phi$ which can represent an $\varepsilon-$accurate Q-function. Our aim is to show that on-policy Q-learning takes steps, exponential in the horizon for solving this task.

For any stochastic policy $\pi(a|s)$, and $\bar{p}$ defined as $\bar{p} = \min_{s \in \mathcal{S}, a \in \mathcal{A}} \pi(a|s)$, $0 < \bar{p} < 0.5$, the marginal state-action distribution satisfies:
\begin{equation}
\label{eqn:max_dsa}
    d^{\pi} (s^*, a^*) \leq \gamma^{H} \cdot (1 - \bar{p})^{H+1}
\end{equation} 
Since $d^{\pi}$ is a discounted state-action marginal distribution, another property that it satisfies is that:
\begin{equation}
\label{eqn:spectral_bound}
c \leq ||d^{\pi}||_2 \leq \frac{1}{1 - \gamma} 2^H    
\end{equation}
where c is a constant $c > 0$. The above is true, since, there are $2^H$ states in this MDP, and the maximum values of any entry in $d^{\pi}$ can be $\frac{1}{1 - \gamma}$ since, $1 - \gamma$ is the least eigenvalue of $(I - \gamma P^\pi)$ for any policy $\pi$, since $||P^{\pi}||_2 = 1$. 

Now, under an on-policy sampling scheme and a linear representation of the Q-function as assumed, the updates on the weights for each iteration of Algorithm~\ref{alg:fqi} are given by ($D_{\pi_k}$ represents $\diag(d^{\pi_k})$): 
\begin{equation}
    w_{k+1} = \left(\Phi^T D_{\pi_k} \Phi \right)^{-1} \Phi^T D_{\pi_k} \left( r + \gamma P^{\pi_k} \Phi w_k \right)
\end{equation}
Now, $||D_{\pi_k} r || \leq \gamma^H (1 - \bar{p})^{H+1} ||\phi(s^*, a^*)||$ from the property Equation~\ref{eqn:max_dsa}. Hence, the maximum 2-norm of the updated $w_{k+1}$ is given by:
\begin{align}
\begin{split}
\label{eqn:sufficient_cond}
    ||w_{k+1}||_2 \leq &~~ ||\left(\Phi^T D_{\pi_k} \Phi \right)^{-1} \Phi^T D_{\pi_k} R||_2 + \gamma || \left(\Phi^T D_{\pi_k} \Phi \right)^{-1} \Phi^T D_{\pi_k} P^{\pi_k} \Phi w_k ||_2\\
    \leq &~~ \frac{\gamma^H (1 - \bar{p})^{H+1}} {||D_{\pi_k}||_F \cdot {(1 - \varepsilon)} \cdot 2^{H-1} } + \gamma ||w_k||_2 \\
    \leq &~~ \frac{\gamma^H (1 - \bar{p})^{H+1} c}{{(1 - \varepsilon)} \cdot{2^{H-1}}} + ||w_k||_2 \\
    = &~~ \left({\gamma}\right)^H \cdot c \cdot \frac{(1)}{{(1 - \varepsilon)} \cdot 2^{H-1}} + ||w_k||_2.
\end{split}
\end{align}
where the first inequality follows by an application of the triangle inequality, the second inequality follows by using the minimum value of the Frobenius norm of the matrix $\Phi$ to be $(1 - \varepsilon) \cdot 2^{H-1}$ (using the $\varepsilon-$rank lemma used to satisfy Assumption~\ref{assumption:optimal_features}) in the denominator of the first term, bounding $||D_{\pi_k} r||$ by Equation~\ref{eqn:max_dsa}, and finally bounding the second term by $\gamma ||w_k||_2$, since the maximum eigenvalue of the entire matrix in front of $w_k$ is $\leq 1$, as it is a projection matrix with a discount $\gamma$ valued scalar multiplier. The third inequality follows from lower bounding $D_{\pi_k}$ by $c$ using Equation~\ref{eqn:spectral_bound}.\\

The optimal $w^*$ is given by the fixed point of the Bellman optimality operator, and in this case satisfies the following via Cauchy-Schwartz inequality,
\begin{align}
    \begin{split}
\label{eqn:neccessary_w_star}
    & (I - \gamma P^*) \Phi w^* = r \\
    \implies & ||\Phi||_F  \cdot ||w^*||_2 \geq ||(I - \gamma P^*)^{-1} r|| \geq \frac{1}{1+ \gamma} ||r||_2 \\
    \implies & (1 + \varepsilon) \cdot 2^{H-1} \cdot ||w^*||_2 \geq \frac{1}{1 + \gamma}\\
    \implies & ||w^*||_2 \geq \frac{1}{1 + \gamma} \cdot 2^{-H + 1} \cdot (1 + \varepsilon)^{-1} 
\end{split}
\end{align}
Thus, in order for $w_k$ to be equal to $w^*$, it must satisfy the above condition (Equation~\ref{eqn:neccessary_w_star}). If we choose an initialization $w_0 = \mathbf{0}$ (or a vector sufficiently close to 0), we can compute the minimum number of steps it will take for on-policy ADP to converge in this setting by using \ref{eqn:sufficient_cond} and \ref{eqn:neccessary_w_star}:
\begin{align}
\begin{split}
    k \geq &~ \frac{(1 + \gamma)^{-1} \cdot 2^{-H + 1} \cdot (1 + \varepsilon)^{-1}}{\left({\gamma}\right)^H \cdot (1 - \bar{p})^H \cdot \frac{(c)}{{(1 - \varepsilon)} \cdot 2^{H-1}}} ~~~~~~ \\
    \implies k \approx &~ \Omega \left({\gamma}^{-H} \right)~~~~~~~~~~~~~
\end{split}
\end{align}
for sufficiently small $\varepsilon$. Hence, the bound follows.

\paragraph{A note on the bound.} Since typically RL problems usually assume discount factors $\gamma$ close to 1, one might wonder the relevance is this bound in practice. We show via an example that this is indeed relevant. In particular, we compute the value of this bound for commonly used $\gamma, \bar{p}$ and $H$. For a discount $\gamma = 0.99$, and a minimum probability of $\bar{p} = 0.01$ (as it is common to use entropy bonuses that induce a minimum probability of taking each action), this bound is of the order of 
\begin{equation}
    \left(\gamma \cdot (1 - \bar{p})\right)^H \approx 10^9 \text{~~ for H = 1000 ~~}
\end{equation}
for commonly used horizon lengths of 1000 (example, on the gym benchmarks).
\end{proof}

\begin{corollary}[Extension to replay buffers]
There exists a family of MDPs parameterized by $H > 0$, with $|\mathcal{S}| = 2^H$, $|\mathcal{A}| = 2$ and a set of features $\Phi$ satisfying assumption~\ref{assumption:optimal_features}, such that ADP with replay buffer distribution takes $\Omega(\gamma^{-H})$ many steps of exact fixed-point iteration for convergence of ADP, if at all convergence happens to an $\epsilon-$accurate Q-function.
\end{corollary}
\begin{proof}
For replay buffers, we can prove a similar statement as previously. The steps in this proof follow exactly the steps in the proof for the previous theorem.

With replay buffers, the distribution for the projection at iteration $k$ is given by:
\begin{equation}
    \label{eqn:replay_buffer_dist}
    d_k(s, a) = \frac{1}{k} \sum_{i=1}^k d_{\pi_k}(s, a)
\end{equation}
Therefore, we can bound the probability of observing any state-action pair similar to Equation~\ref{eqn:max_dsa} as:
\begin{equation}
    \label{eqn:max_probability}
    d_k(s^*, a^*) \leq \frac{1}{k} \sum_{i=1}^k \gamma^H \cdot (1 - \bar{p})^{H+1} 
\end{equation}
with $\bar{p}$ as defined previously. Note that this inequality is the same as the previous proof, and doesn't change. We next bound the 2-norm of the state-visitation distribution, in this case, the state-distribution in the buffer.
\begin{equation}
    \label{eqn:state_dist_frobenius}
    c \leq ||d_k||_2 \leq \frac{1}{1 - \gamma} \cdot 2^H 
\end{equation}
where $c > 0$. The two main inequalities used are thus the same as the previous proof. Now, we can simply follow the previous proof to prove the result.
\end{proof}

\paragraph{Practical Implications.} In this example, both on-policy and replay buffer Q-learning suffer from the problem of exponentially many samples need to reach the optimal Q-function. Even in our experiments in Section~\ref{sec:problem_description}, we find that on-policy distributions tend to reduce errors very slowly, at a rate that is very small. The above bound extends this result to replay buffers as well. 

In our next result, however, we show that an optimal choice of distribution, including DisCor, can avoid the large iteration complexity in this family of MDPs. Specifically, using the errors against $Q^*$, i.e. $|Q_k - Q^*|$ can help provide a signal to improve the Q-function such that this optimal distribution / DisCor will take only $\mathrm{poly}(H)$ many iterations for convergence.

\begin{corollary}[Optimal distributions / DisCor]
\label{thm:discor_suboptimal}
In the tree MDP family considered in Theorem~\ref{thm:exponential}, with linear function approximation for the Q-function, and with Assumption~\ref{assumption:optimal_features} for the features $\Phi$, DisCor takes $\mathrm{poly}(H)$ many exact iterations for $\varepsilon-$accurate convergence to the optimal Q-function.
\end{corollary}
\begin{proof}
We finally show that the DisCor algorithm, which prioritizes states based on the error in target values, will take $\mathrm{poly}(H)$ many steps for convergence. Assume that Q-values are initialized randomly, for example via a normal random variable with standard deviation $\sigma$, i.e., $Q_0(s, a) \sim \mathcal{N}(0, \sigma^2)$, however, $\sigma$ is very small, but is more than 0 ($\sigma >0$) (this proof is still comparable to the proof for on-policy distributions, since Q-values can also be initialized very close to $0$ even in that case, and the proof of Theorem~\ref{thm:app_exp_lower_bound} still remains valid.). 

Now we reason about a run of DisCor in this case.

\paragraph{Iteration 1.} In the first iteration, among all nodes in the MDP, the leaf nodes (depth $H$-1) have $0$ error at the corresponding target values, since an episode terminates once a rollout reaches a leaf node. Hence, the algorithm will assign equal mass to all leaf node states, and exactly update the Q-values for nodes in this level (upto $\varepsilon$-accuracy).

\paragraph{Iteration 2.} In the second iteration, the leaf nodes at level $H-1$ have accurate Q-values, therefore, the algorithm will pick nodes at the level $H-2$, for which the target values, i.e. Q-values for nodes at level $H-1$, have 0 error along with nodes at level $H-1$. The algorithm will update Q-values at these nodes at level $H-2$, while ensuring that the incurred error at the nodes at level $H-1$ isn't beyond $\varepsilon$, since nodes at both levels are chosen. Since, the optimal value function $Q^*$ can be represented upto $\varepsilon-$accuracy, we can satisfy this criterion.

\paragraph{Iteration $k$.} In iteration $k$, the algorithm updates Q-values for nodes at level $H-k$, while also ensuring Q-values for all nodes at a level higher than $H-k$ are estimated within the range of $\varepsilon-$allowable error, since all the nodes below level $H-k$ are updated. This is feasible since, $Q^*$ is expressible with $\varepsilon-$accuracy within the linear function class chosen. 

This iteration process continues, and progress level by level, from the leaves (level $H-1$) to the root (level $0$). At each iteration Q-values for all states at the same level, and below are learned together. Since learning progresses in a ``one level at-a-time'' fashion, with guaranteed correct target values (i.e. target values are equal to the optimal Q-function $Q^*$) for any update that the algorithm performs, it would take at most $\mathrm{poly}(H)$ many iterations (for example, multiple passes through the depth of the tree) for $\varepsilon$-accurate convergence to the optimal $Q$-function.
\end{proof}

\section{Extended Related Work}
\label{app:related_work_extended}

\paragraph{Error propagation in ADP.} A number of prior works have analysed error propagation in ADP methods. Most work in this area has been devoted to analysing how errors in Bellman error minimization propagate through the learning process of the ADP algorithm, typically focusing on methods such as fitted Q-iteration (FQI)~\citep{Riedmiller2005} or approximate policy iteration~\citep{perkins2002api}. Prior works in this area assume an abstract error model, and analyze how errors propagate. Typically these prior works only limitedly explore reasons for error propagation or present methods to curb error propagation. \cite{munos2003api} analyze error propagation in approximate policy iteration methods using quadratic norms. \cite{munos2005error} analyze the propagation of error across iterations of approximate value iteration (AVI) for $L_p$-norm $p=(1,2)$. \cite{munos2008finite} provide finite sample guarantees of AVI using error propagation analysis. Similar ideas have been used to provide error bounds for a number of different methods -- \cite{farahmand2010error,scherrer15a,Lesner2013TightPB,scherrer_comparison} and many more. In this work, we show that ADP algorithms suffer from an absence of corrective feedback, which arises because the data distribution collected by an agent is insufficient to ensure that error propagation is eventually corrected for. We further propose an approach, DisCor, which can be used in conjunction with modern deep RL methods. 

\paragraph{Offline / Batch Reinforcement Learning.} Our work bears similarity to the recent body of literature on batch, or offline reinforcement learning~\citep{kumar19bear,fujimoto19a,wu2019behavior}. All of these works exploit the central idea of constraining the policy to lie in a certain neighborhood of the behavior, data-collection policy. While \cite{kumar19bear} show that this choice can be motivated from the perspective of error propagation, we note that there are clear differences between our work and such prior works in batch RL. First, the problem statement of batch RL requires learning from completely offline experience, however, our method learns online, via on-policy interaction. While error propagation is a reason behind incorrect Q-functions in batch RL, we show that such error accumulation also happens in online reinforcement, which results in a lack of corrective feedback.  

\paragraph{Generalization effects in deep Q-learning.} There are a number of recent works that theoretically analyze and empirically demonstrate that certain design decisions for neural net architectures used for Q-learning, or ADP objectives can prove to be significant in deep Q-learning. For instance, \cite{martha2018sparse} point out that sparse representations may help Q-learning algorithms, which links back to prior literature on state-aliasing and destructive interference. \cite{Achiam2019TowardsCD} uses an objective inspired from the neural tangent kernel (NTK)~\citep{ntk} to ``cancel'' generalization effects in the Q-function induced across state-action pairs to mimic tabular and online Q-learning. Our approach, DisCor, can be interpreted as \textit{only} indirectly affecting generalization via the target Q-values for state-action pairs that will be used as bootstrap targets for the Bellman backup, which are expected to be accurate with DisCor, and this can aid generalization, similar to how generalization can be achieved via abstention from training on noisy labels in supervised learning~\citep{absention}.

\paragraph{Replay Buffers and Generalization.} There are some prior works performing analytical studies on the size of the replay buffer~\citep{zhang2017deeper,liu2018effects}, which propose that larger replay buffers might hurt training with function approximation. Partly this problem goes away if smaller buffers are used, or if the algorithm chooses to replay recent experience more often. Our work indicates an absence of corrective feedback problem -- online data collection might not be able to correct errors in the Q-function -- which is distinct from the size of the replay buffer.   

\section{Experimental Details}
\label{app:exp_details}
In this section, we provided experimental details, such as the DisCor algorithm in practice (Section~\ref{sec:discor_practice}), and the hyperparameter choices (Section~\ref{sec:app_exp_details}).

\subsection{DisCor in Practice}
\label{sec:discor_practice}
In this section, we provide details on the experimental setup and present the pseudo-code for the practical instantiation of our algorithm, DisCor.

The pseudocode for the practical algorithm is provided in Algorithm~\ref{alg:practical_alg}. Like any other ADP algorithm, such as DQN or SAC, our algorithm maintains a pair of Q-functions -- the online Q-network $Q_\theta$ and a target network $Q_{\bar{\theta}}$. For continuous control domains, we use the clipped double Q-learning trick~\cite{pmlr-v80-fujimoto18a}, which is also referred to as the ``twin-Q'' trick, and it further parametrizes another pair of online and target Q-functions, and uses the minimum Q-value for backup computation. In addition to Q-functions, in a continuous control domain, we parametrize a separate policy network $\pi_\psi$ similar to SAC. In a discrete action domain, the policy is just given by a greedy maximization of the online Q-network. 

DisCor further maintains a model for accumulating errors $\Delta_\phi$ parameterized by $\phi$ and the corresponding target error network $\Delta_{\bar{\phi}}$. In the setting with two Q-functions, DisCor models two networks, one for modelling error in each Q-function. At every step, a few (depending upon the algorithm) gradient steps are performed on $Q$ and $\Delta$, and $\pi$ -- if it is explicitly modeled, for instance in continuous control domains. This is a modification of generalized ADP Algorithm~\ref{alg:fqi} and the corresponding DisCor version (Algorithm~\ref{alg:discor}), customized to modern deep RL methods.

\begin{algorithm}[t!]
\small
\caption{\textbf{DisCor: Deep RL Version}}
\label{alg:practical_alg}
\begin{algorithmic}[1]
    \STATE Initialize online Q-network $Q_\theta(s, a)$, target Q-network, $Q_{\bar{\theta}}(s, a)$, error network $\Delta_\phi(s, a)$, target error network $\Delta_{\bar{\phi}}$, initial distribution $p_0(s, a)$, a replay buffer $\beta$ and a policy $\pi_\psi(a|s)$, number of gradient steps $G$, target network update rate $\eta$, initial temperature for computing weights $w_k$, $\tau_0$.
    \FOR{step $k$ in \{1, \dots, \}}
        \STATE Collect $M$ samples using $\pi_\psi(a|s)$, add them to replay buffer $\beta$, sample $\{(s_i, a_i)\}_{i=1}^N \sim \beta$
        \STATE Evaluate $Q_\theta(s,a)$ and $\Delta_\phi(s, a)$ on samples $(s_i, a_i)$.
        \STATE Compute target values for $Q$ and $\Delta$ on samples:
        $${y}_i = r_i + \gamma \expec_{a' \sim \pi_\psi(a'|s')} [Q_{\bar{\theta}}(s'_i, a')] $$
        $$\hat{\Delta}_{i} = |Q_\theta(s, a) - y_i| + \gamma \expec_{\hat{a}_i \sim \pi(a_i|s')} [\Delta_{\bar{\phi}}(s'_i, \hat{a}_i)]$$
        \STATE {Compute $w_k$ using Equation~\ref{eqn:importance_weights} with temperature $\tau_k$} 
        \STATE Take $G$ gradient steps on the Bellman error for training $Q_\theta$ weighted by $w_k$.
        $$ \theta \leftarrow \theta -  \alpha \nabla_\theta \frac{1}{N}\sum_{i=1}^N {w_k(s_i, a_i)} \cdot (Q_\theta(s_i,a_i) - y_i)^2$$
        \STATE {Tale $G$ gradient steps to minimize unweighted (regular) Bellman error for training $\phi$.
        $$\phi \leftarrow \phi - \alpha \nabla_\phi \frac{1}{N} \sum_{i=1}^N (\Delta_\theta(s_i, a_i) - \hat{\Delta}_i)^2$$}
        \STATE Update the policy $\pi_\psi$ if it is explicitly modeled.
        $$\psi \leftarrow \psi + \alpha \nabla_\psi \expec_{s \sim \beta, a \sim \pi_\psi(a|s)} [Q_\theta(s, a)]$$ 
        \STATE Update target networks using soft updates (SAC), hard updates (DQN)
           $$\bar{\theta} \leftarrow (1 - \eta) \bar{\theta} + \eta \theta$$ 
           $$\bar{\phi} \leftarrow (1 - \eta) \bar{\phi} + \eta \phi$$
        \STATE Update temperature hyperparameter for DisCor:
        $$ \tau_{k+1} \leftarrow (1 - \eta) \tau_k + \eta ~~\textsc{batch-mean}(\Delta_\phi(s_i, a_i))$$
    \ENDFOR
\end{algorithmic}
\end{algorithm}

\subsection{Experimental Hyperparameter Choices}
\label{sec:app_exp_details}
We finally specify the hyperparameters we used for our experiments. These are as follows:
\begin{itemize}
    \item \textit{Temperature $\tau$:} DisCor mainly introduces one hyperparameter, the temperature $\tau$ used to compute the weights $w_k$ in Equation~\ref{eqn:importance_weights}. As shown in Line 11 of Algorithm~\ref{alg:practical_alg}, DisCor maintains a moving average of the temperatures and uses this average to perform the weighting. This removes the requirement for tuning the temperature values at all. For initialization, we chose $\tau_0 = 10.0$ for all our experiments, irrespective of the domain or task. 
    \item \textit{Architecture for $\Delta_\phi$:} For the design of the error network, $\Delta_\phi$, we utilize a network with 1 extra hidden layer than the corresponding Q-network. For instance, in metaworld domains, the standard Q-network used was [256, 256, 256] in size, and thus we used an error network of size: [256, 256, 256, 256], and for MT10 tasks we used [160, 160, 160, 160, 160, 160] sized Q-networks~\citep{yu2020gradient} and 1-extra layer error networks $\Delta_\phi$.
    \item \textit{Target net updates:} We performed target net updates for $\Delta_{\bar{\phi}}$ in the same manner as standard Q-functions, in all domains. For instance, in MetaWorld, we update the target network $\Delta_{\bar{\phi}}$ with a soft update rate of 0.005 at each environment step, as is standard with SAC~\cite{haarnoja2018sacapps}, whereas in DQN~\citep{Mnih2015}, we use hard target resets.
    \item \textit{Learning rates for $\Delta_\phi$:} These were chosen to be the same as the corresponding learning rate for the Q-function, which is $3e-4$ for SAC and $0.0025$ for DQN.
    \item \textit{Official Implementation repositories used for our work:} 
        \begin{enumerate}
            \item Soft-Actor-Critic: \url{https://github.com/rail-berkeley/softlearning/}
            \item Dopamine~\citep{castro18dopamine}: Offical DQN implementation \url{https://github.com/google/dopamine}, and the baseline DQN numbers were reported from the logs available at: \url{https://github.com/google/dopamine/tree/master/baselines}
            \item Tabular environments~\citep{fu19diagnosing}: \url{https://github.com/justinjfu/diagnosing_qlearning}
        \end{enumerate}
    \item We perform self-normalized importance sampling across a batch, instead of regular importance sampling, since that gives rise to more stable training, and suffers less from the curse of variance in importance sampling.
    \item \textit{Seeds}: In all our experiments, we implemented our methods on top of the official repositories, ran each experiment for 4 randomly chosen seeds from the interval, $[10, 10000]$, in Meta-World, OpenAI gym and tabular environments. For DQNs on atari, we were only able to run 3 seeds for each game for our method, however, we found similar performances, and less variance across seeds, as is evident from the variance bands in the corresponding results. For baseline DQN, we just used the log files provided by the dopamine repository for our results. 
\end{itemize}  

\section{Additional Experiments}
\label{sec:additional_exps}
We now present some additional experimental results which could not have been presented in Section~\ref{sec:experiments}. 

\subsection{Tabular Environment Analysis}
\label{sec:app_exps_gridworld}
\paragraph{Environment Setup.} We used the suite of tabular environments from from \cite{fu19diagnosing}, which provides a suite of 8 tabular environments and a suite of algorithms based on fitted Q-iteration~\citep{Riedmiller2005}, which forms the basis of modern deep RL algorithms based on ADP. We evaluated performance on different variants of the $(16, 16)$ gridworld provided, with different reward styles (sparse, dense), different observation functions (one-hot, random features, locally smooth observations), and different amounts of entropy coefficients (0.01, 0.1). We evaluated on five different kinds of environments: grid16randomobs, grid16onehot, grid16smoothobs, grid16smoothsparse, grid16randomsparse -- which cover a wide variety of combinations of feature and reward types. We also evaluated on CliffWalk, Sparsegraph and MountainCar MDPs in Figures~\ref{fig:app_fig_exact} and \ref{fig:app_fig_sampled}. 

\paragraph{Sampling Modes.} We evaluated in two modes -- (1) exact mode, in the absence of sampling error, where an algorithm is provided with all transitions in the MDP and simply chooses a weighting over the states rather than sampling transitions from the environment, and (2) sampled mode, which is the conventional RL paradigm, where the algorithm performs online data collection to collect its own data.  

\begin{figure}
    \centering
\includegraphics[width=0.6\linewidth]{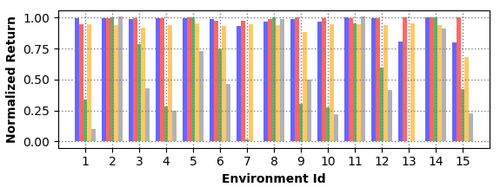}
    \caption{\footnotesize{Performance of different methods: DisCor (blue), DisCor (oracle) (red), Replay buffer Q-learning (green, on-policy (grey) and prioritized updates (orange), across different environments measured in terms of smooth normalized returns in the \textbf{exact} setting with all transitions. Note that DisCor and DisCor (oracle) generally tend to perform better.}}
    \label{fig:app_fig_exact}
\end{figure}

\paragraph{Setup for Figures~\ref{fig:visitation_doesnt_correct_eror} and \ref{fig:suboptimal_conv}.} For Figures~\ref{fig:visitation_doesnt_correct_eror} and \ref{fig:suboptimal_conv}, we used the grid16randomobs MDP (which is a $16 \times 16$ gridworld with randomly initialized vectors as observations), with an entropy penalty of 0.01 to the policy. For Figure~\ref{fig:instability} we used the grid16smoothobs MDP with locally smooth observations, with an entropy penalty of 0.01 as well, and for Figure~\ref{fig:sparse_reward}, we used grid16smoothsparse environment, with sparse reward and smooth features. 

\paragraph{Results.} We provide some individual environment performance curves showing the smoothed normalized return achieved at the end of 300 steps of training in both exact (Figure~\ref{fig:app_fig_exact}) and sampled (Figure~\ref{fig:app_fig_sampled}) settings. We also present some individual-environment learning curves for these environments comparing different methods in both exact (Figure~\ref{fig:exact_fqi_runs}) and sampled (Figure~\ref{fig:sampled_fqi_runs}). 

\begin{figure}
    \centering
\includegraphics[width=0.6\linewidth]{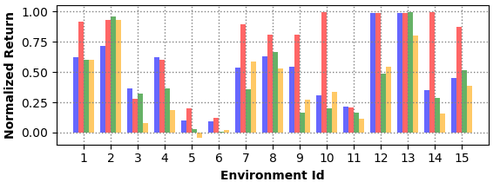}
    \caption{\footnotesize{Performance of different methods: DisCor (blue), DisCor (oracle) (red), Replay buffer Q-learning (green) and prioritized updates (orange). across different environments measured in terms of smooth normalized return with \textbf{sampled} transitions. Note that DisCor and DisCor (oracle) generally tend to perform better.}}
    \label{fig:app_fig_sampled}
\end{figure}

\begin{figure*}
    \centering
    \includegraphics[width=0.19\linewidth]{images/plot_gridworld_new/grid16randomobs_ent0_01_exact_new.pdf}
    \includegraphics[width=0.19\linewidth]{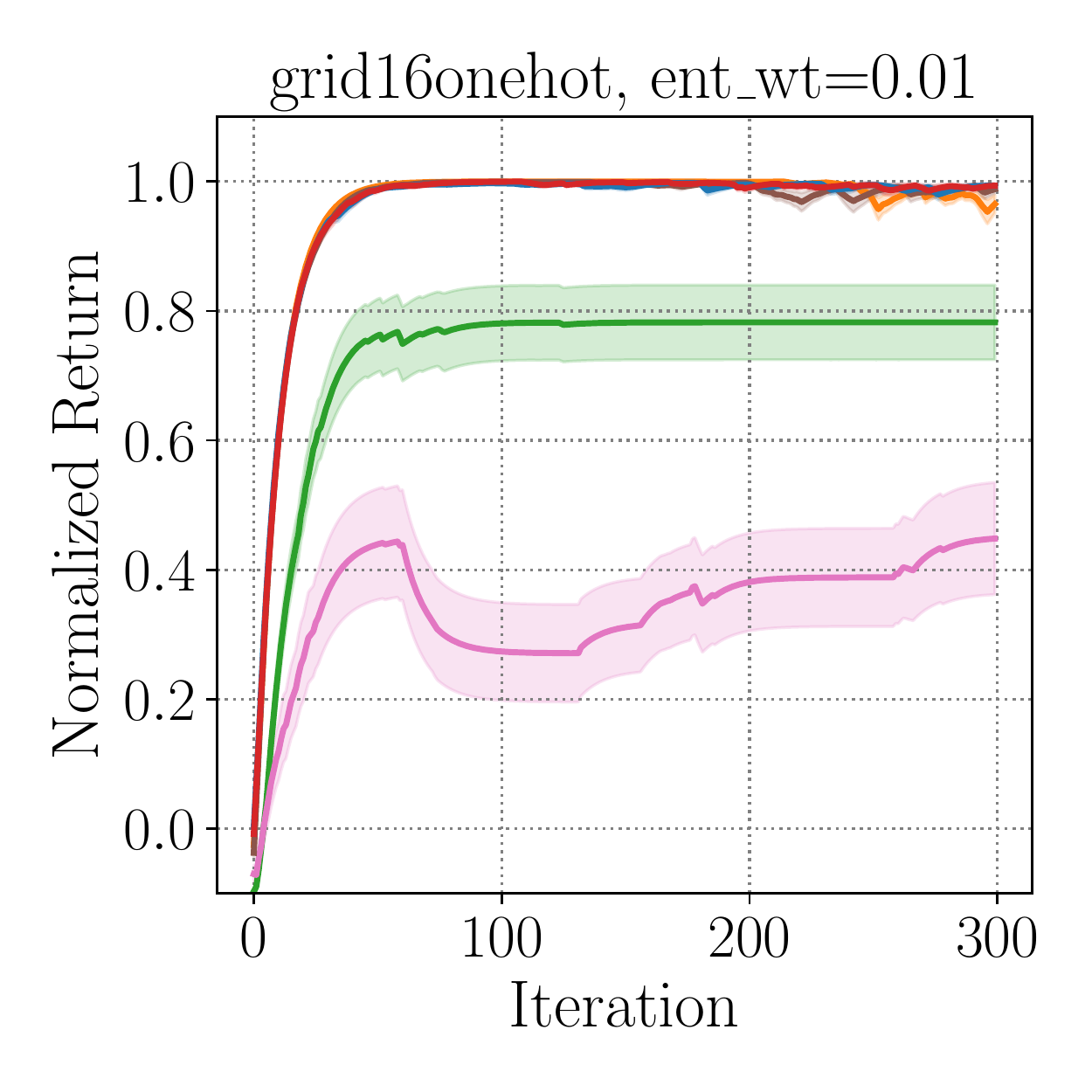}
     \includegraphics[width=0.19\linewidth]{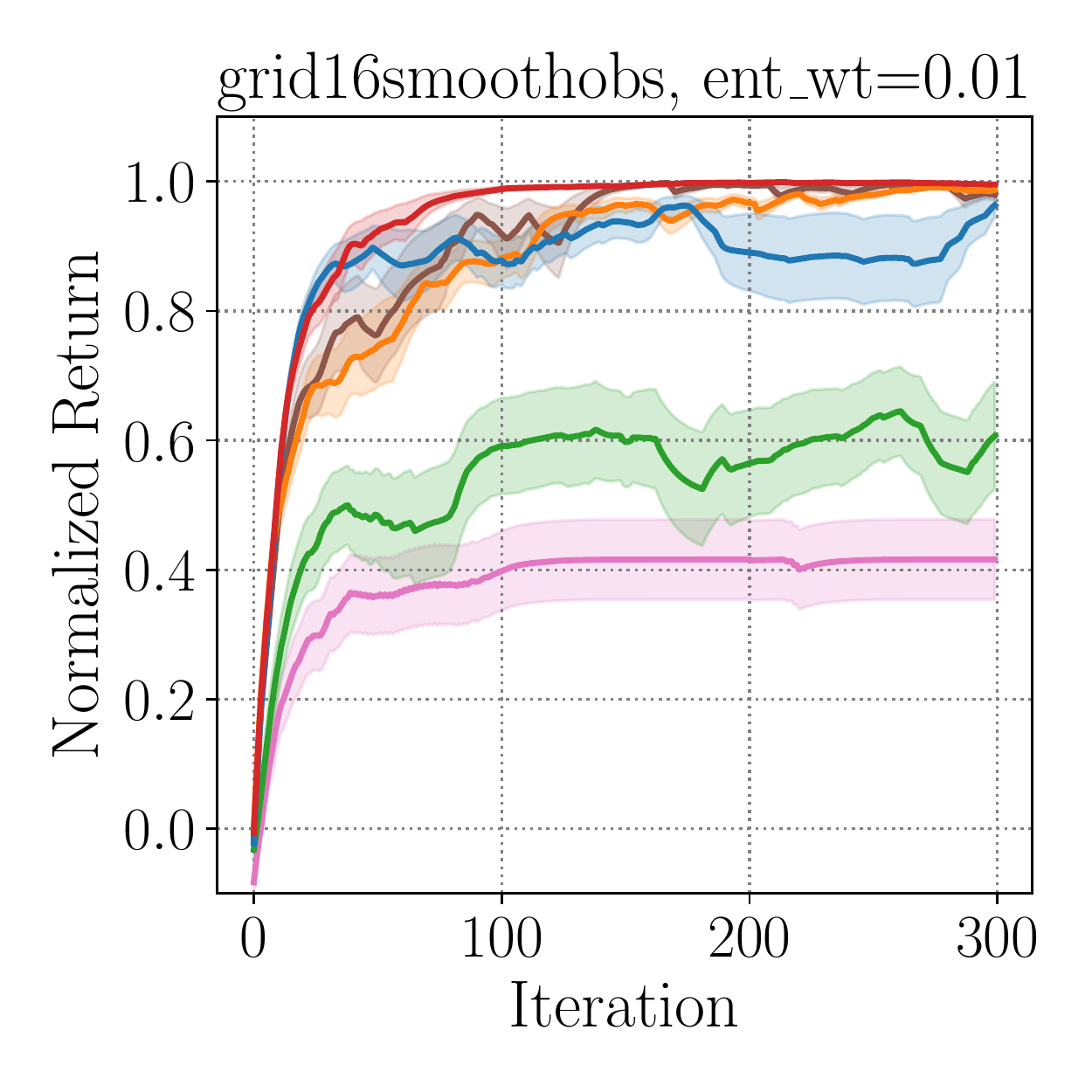}
     \vline
    \includegraphics[width=0.19\linewidth]{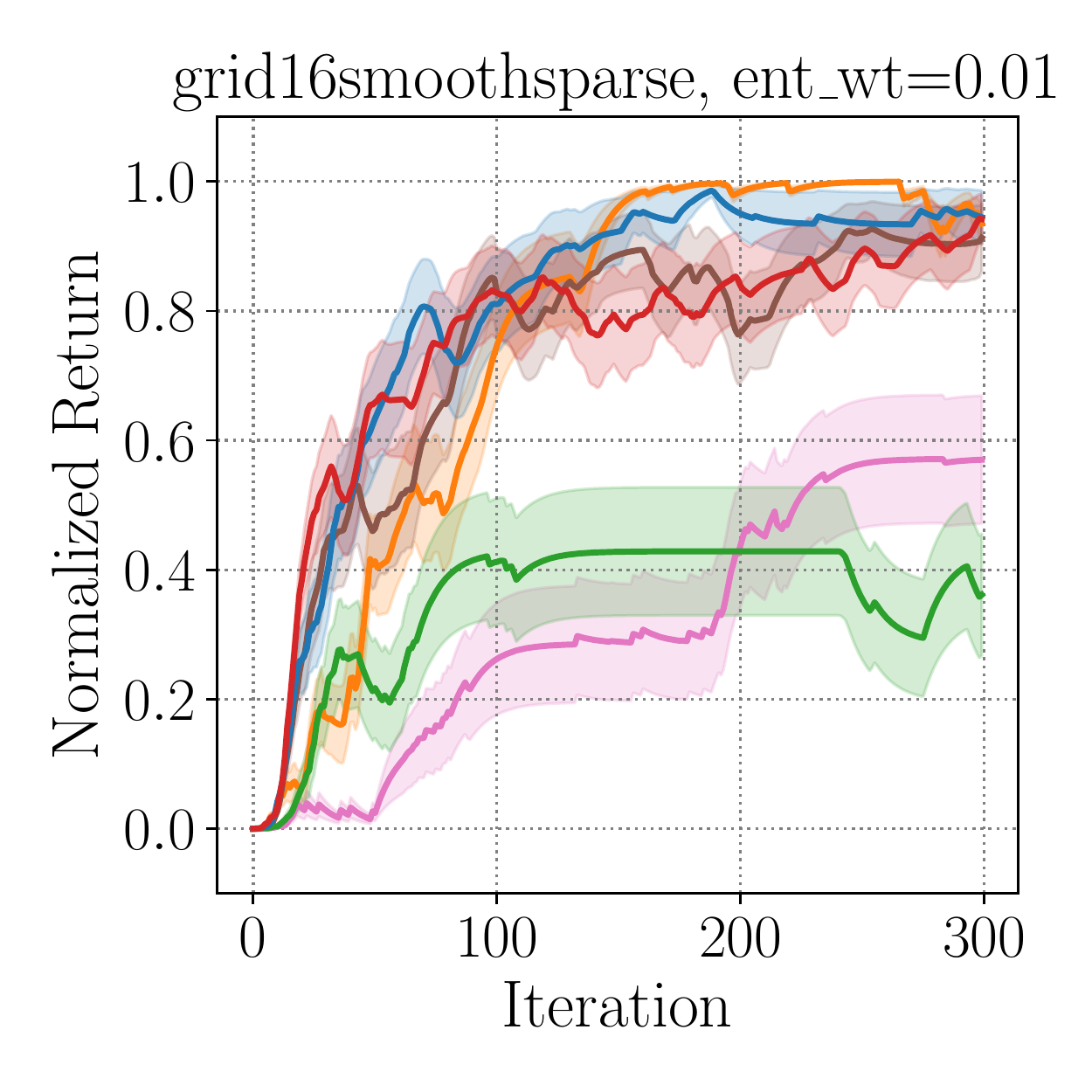}
    \includegraphics[width=0.19\linewidth]{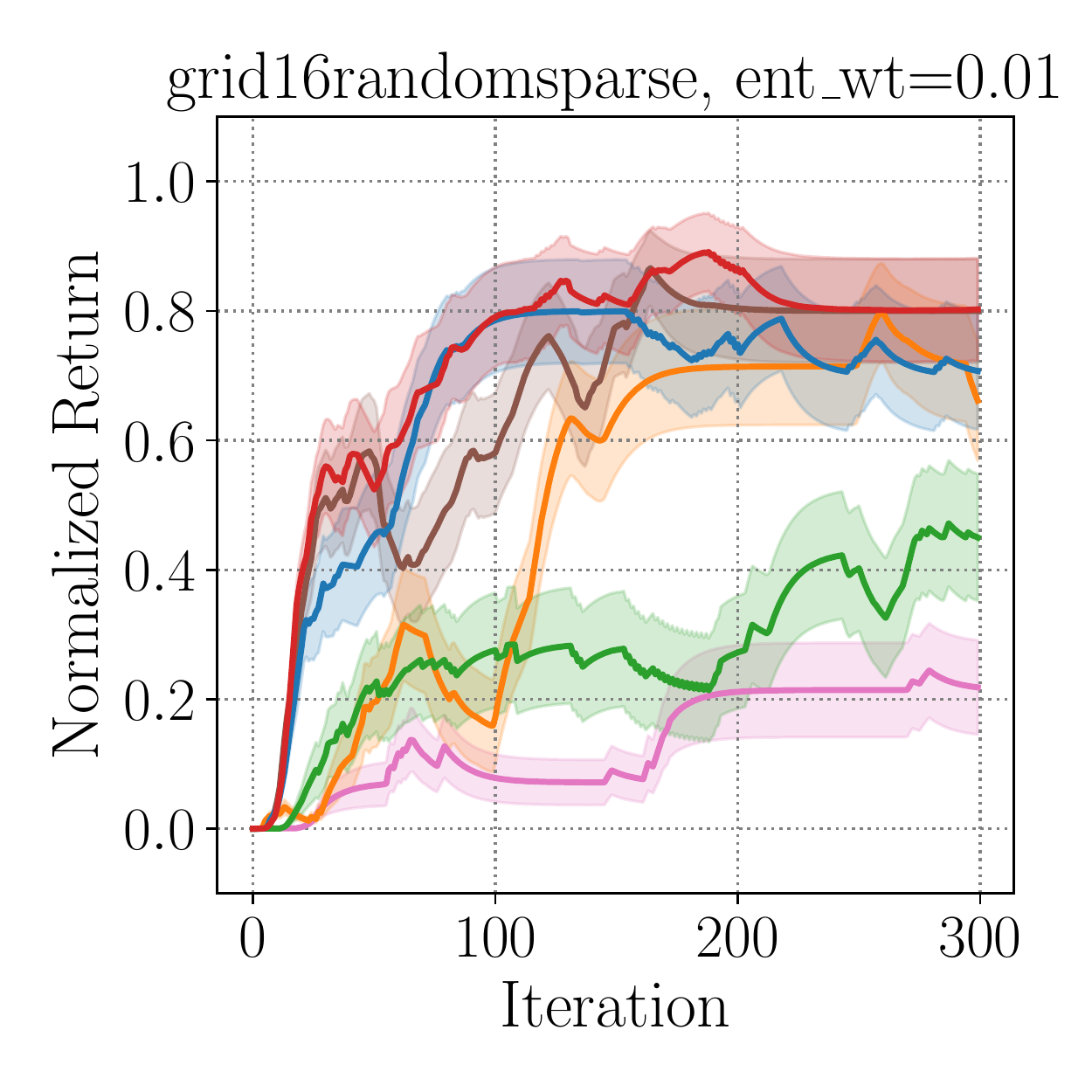}
    \includegraphics[width=0.75\linewidth]{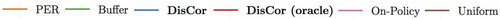}
    \caption{\footnotesize{Learning curves for different algorithms in the \textbf{exact} setting. Note that DisCor (blue) and DisCor (oracle) (red) are generally the best algorithms in these settings. Replay Buffers (green) help over on-policy (pink) distributions. Prioritizing transitions based on high Bellman error (orange) is performant in some cases, but hurts in the other cases -- it is especially slow in cases with sparse rewards, note the speed of learning on grid16randomsparse and grid16smoothsparse (\textbf{right} of the vertical line) environments.}}
    \label{fig:exact_fqi_runs}
\end{figure*}

\begin{figure*}
    \centering
    \includegraphics[width=0.19\linewidth]{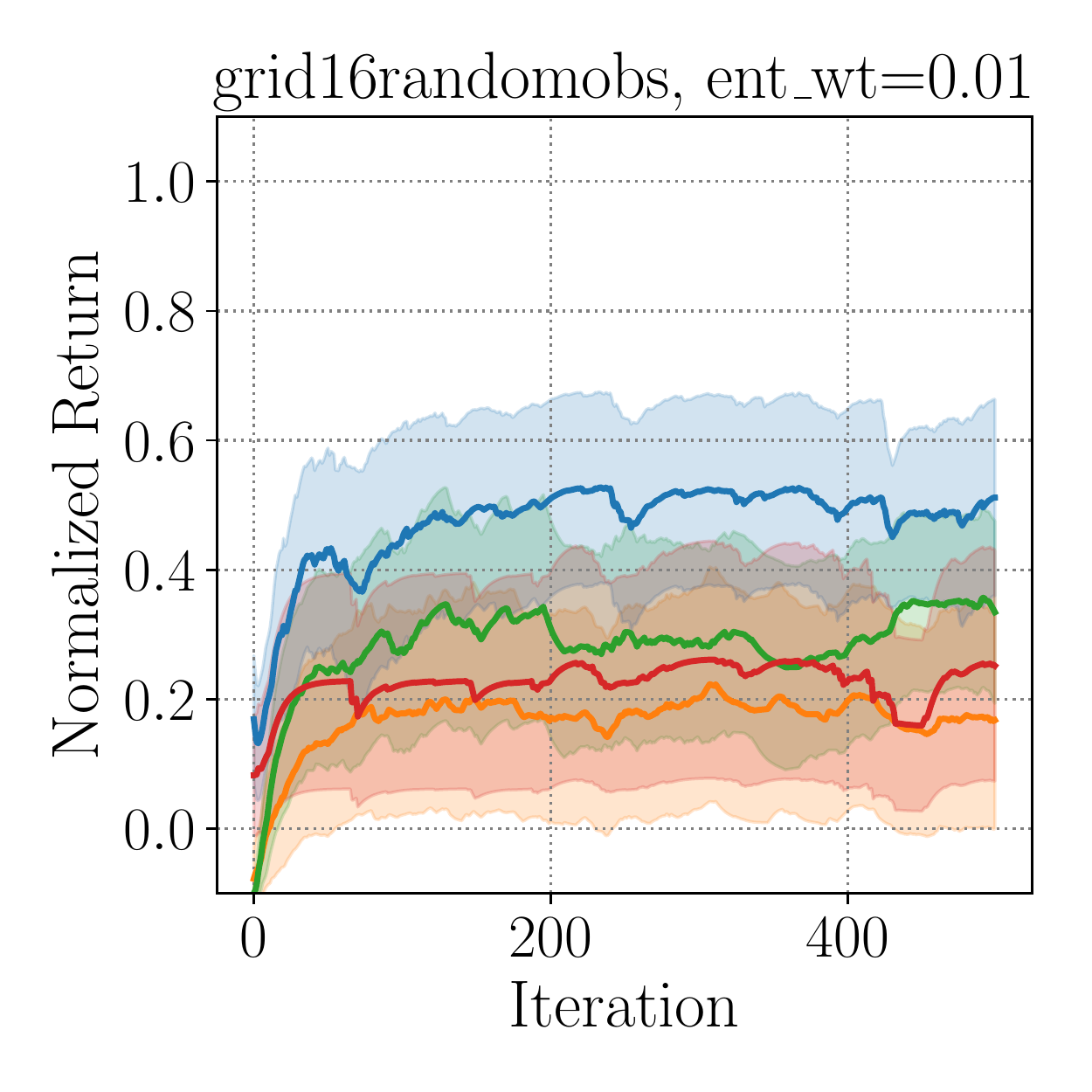}
    \includegraphics[width=0.19\linewidth]{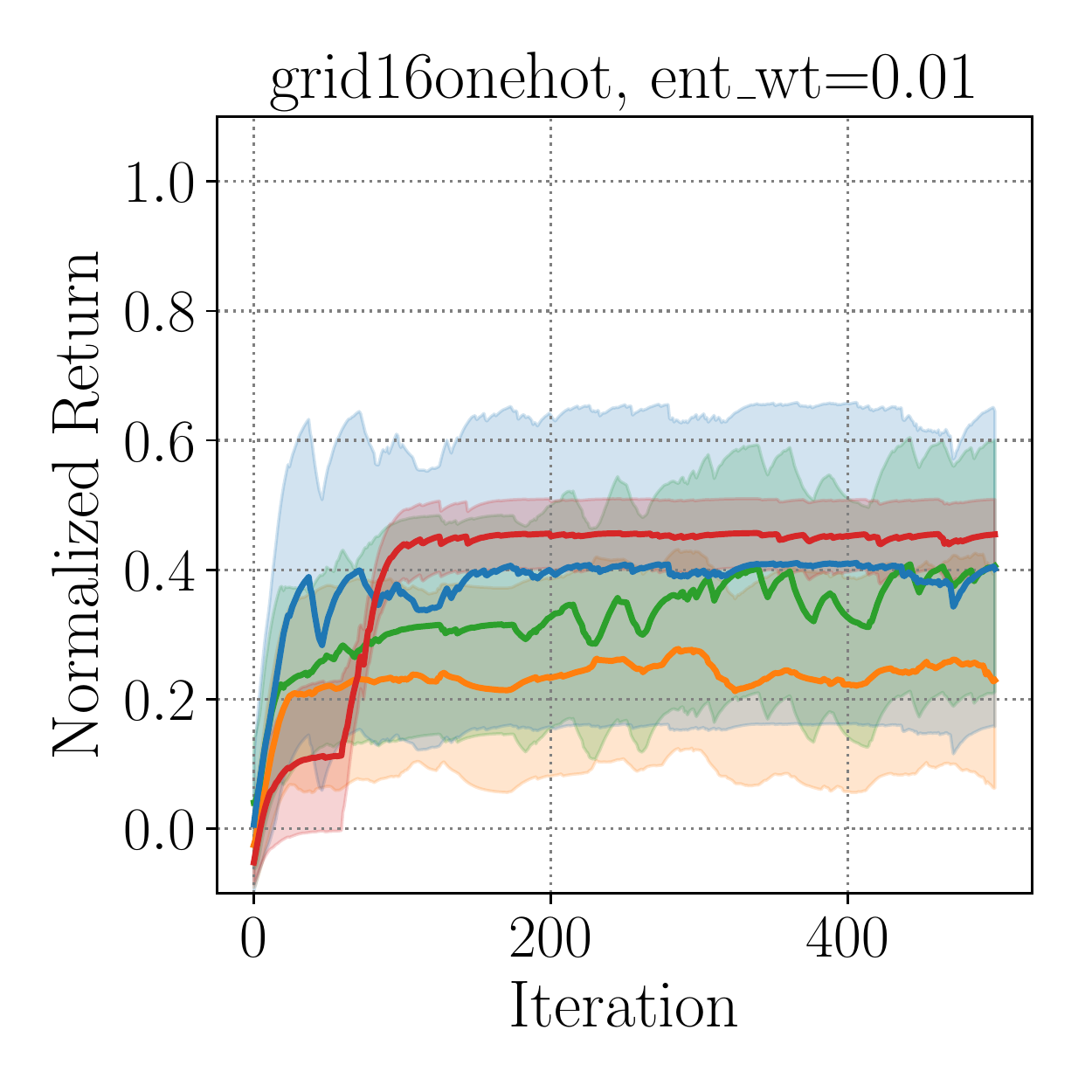}
    \includegraphics[width=0.19\linewidth]{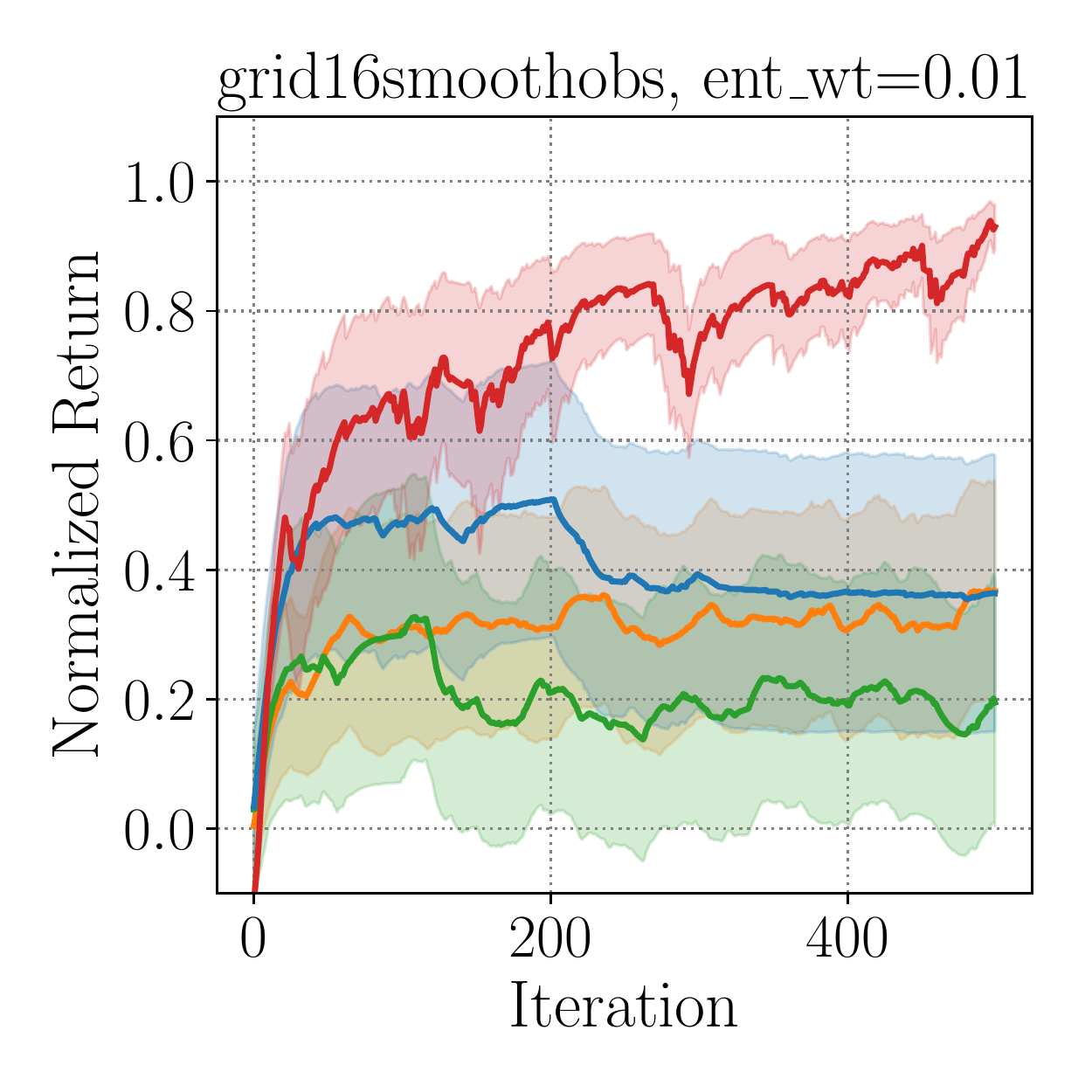}
    \vline
    \includegraphics[width=0.19\linewidth]{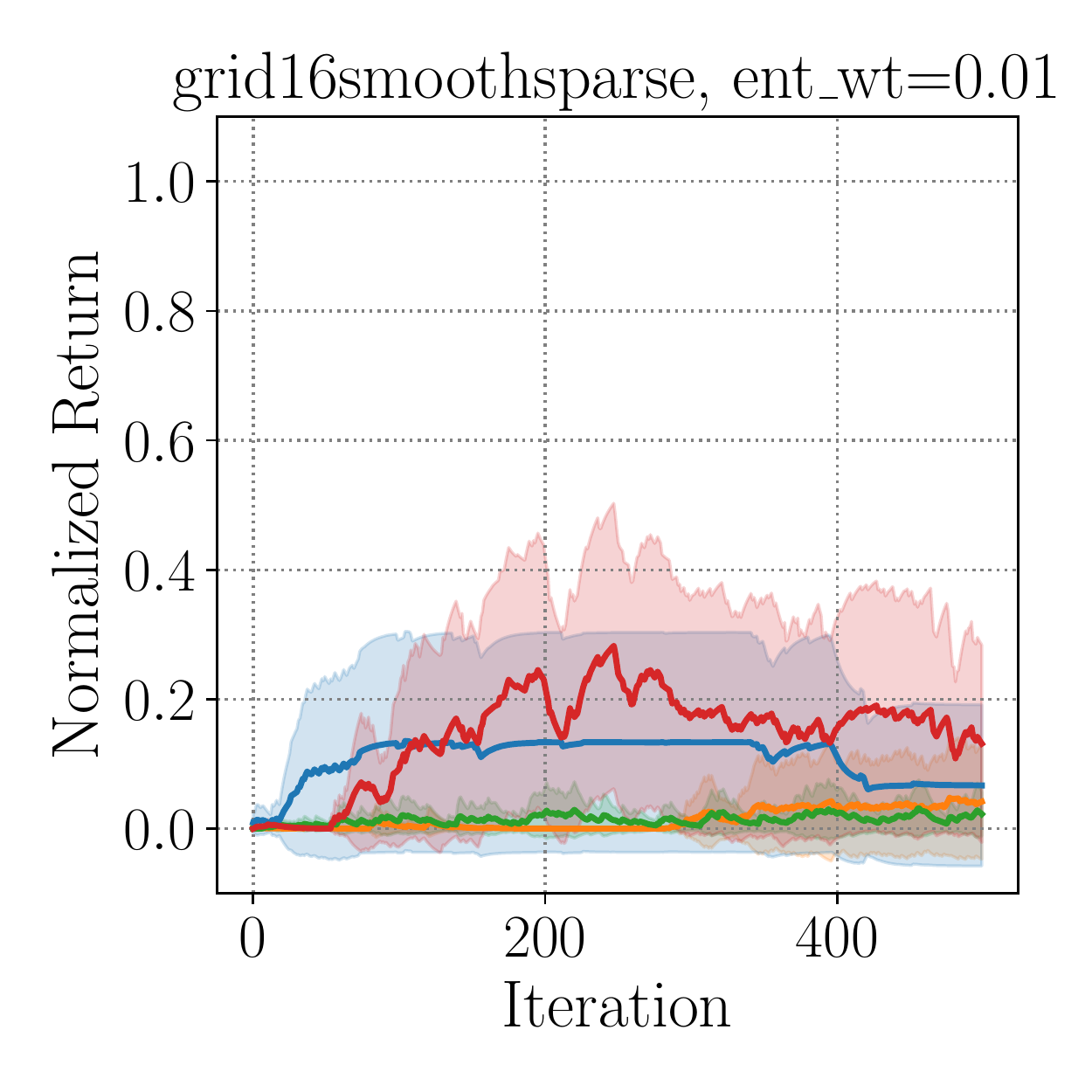}
    \includegraphics[width=0.19\linewidth]{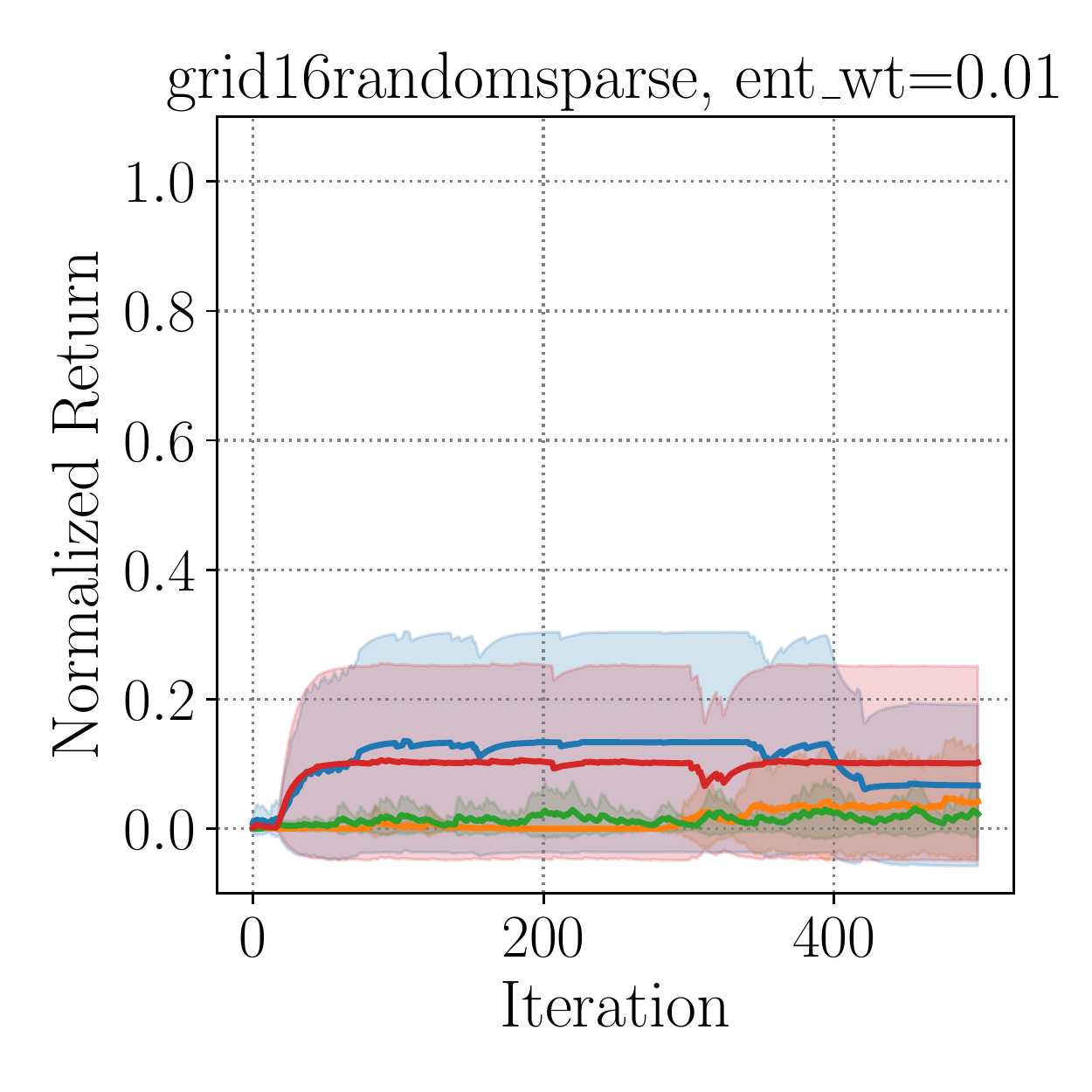}
    \includegraphics[width=0.5\linewidth]{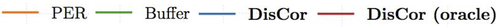}
    \caption{\footnotesize{Learning curves for different algorithms in the \textbf{sampled} setting. Note that DisCor and DisCor (oralce) anre generally the best algorithms in these settings. Replay Buffers (green) help over on-policy (gray) distributions, but may the algorithm may still fail to reach optimal return. Prioritizing for high Bellman error (PER) may fail to learn in sparse-reward tasks as is evident from the curves for sparse reward environments (\textbf{right} of the vertical line).}}
    \label{fig:sampled_fqi_runs}
\end{figure*}

\vspace{-10pt}
\subsection{MetaWorld Tasks}
\label{sec:app_exps_metaworld}
In this section, we first provide a pictorial description of the six hard tasks we tested on from meta-world, where SAC usually does not perform very well. Figure \ref{fig:metaworld_tasks} shows these tasks. We provide the trends for average return achieved during evaluation (not the success rate as shown in Figure \ref{fig:success_rates_metaworld} in Section \ref{sec:experiments}) for each of the six tasks. Note that DisCor clearly outperforms both the baseline SAC and the prior method PER in all six cases, achieving nearly \textbf{50\%} more than the returns achieved by SAC. 

\begin{figure}[ht]
    \centering
    \includegraphics[width=0.6\linewidth]{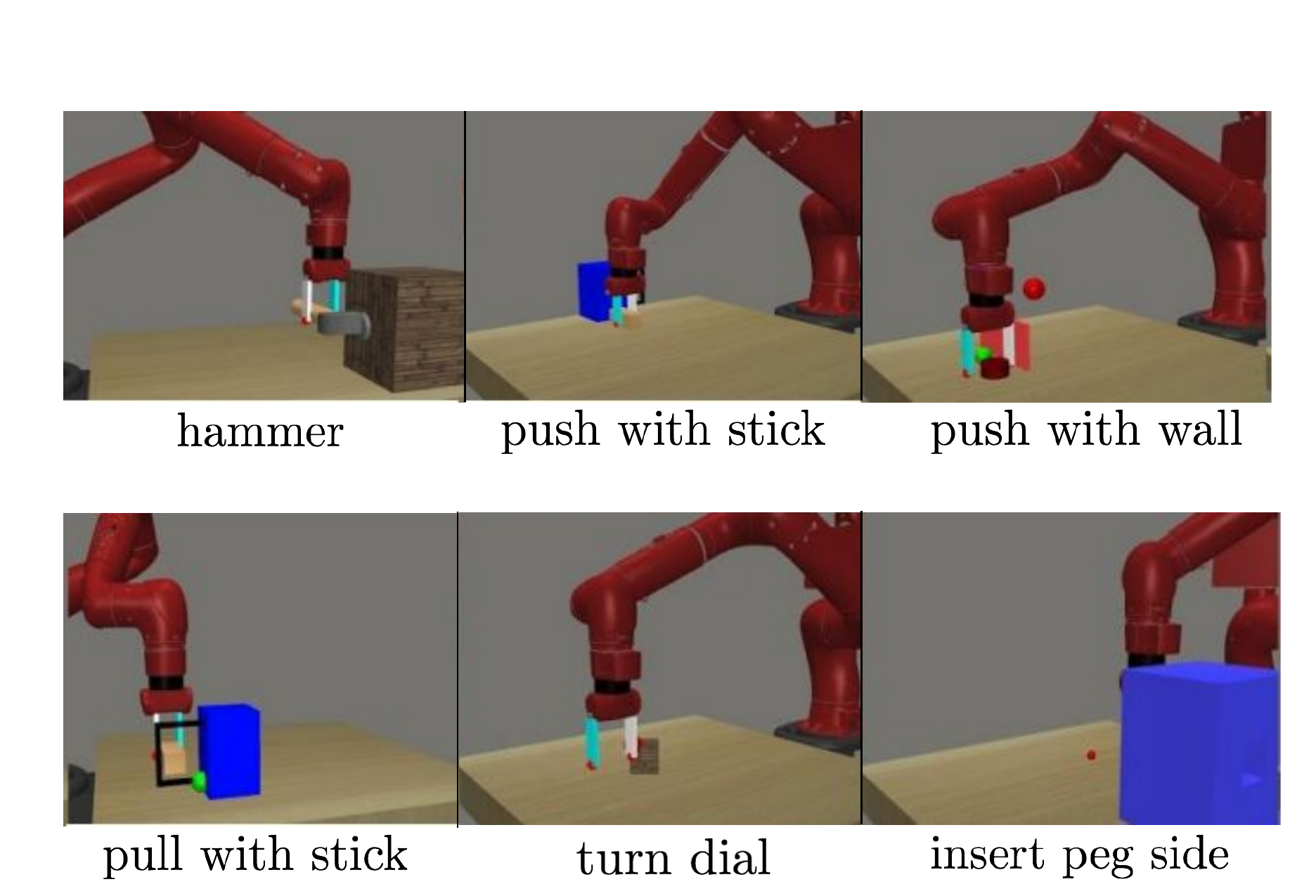}
    \caption{\footnotesize{Visual description of the six MetaWorld tasks used in our experiments in Section~\ref{sec:experiments}. Figures taken from \cite{yu2019meta}.}}
    \label{fig:metaworld_tasks}
\end{figure}

\begin{figure*}
    \centering
    \includegraphics[width=0.25\linewidth]{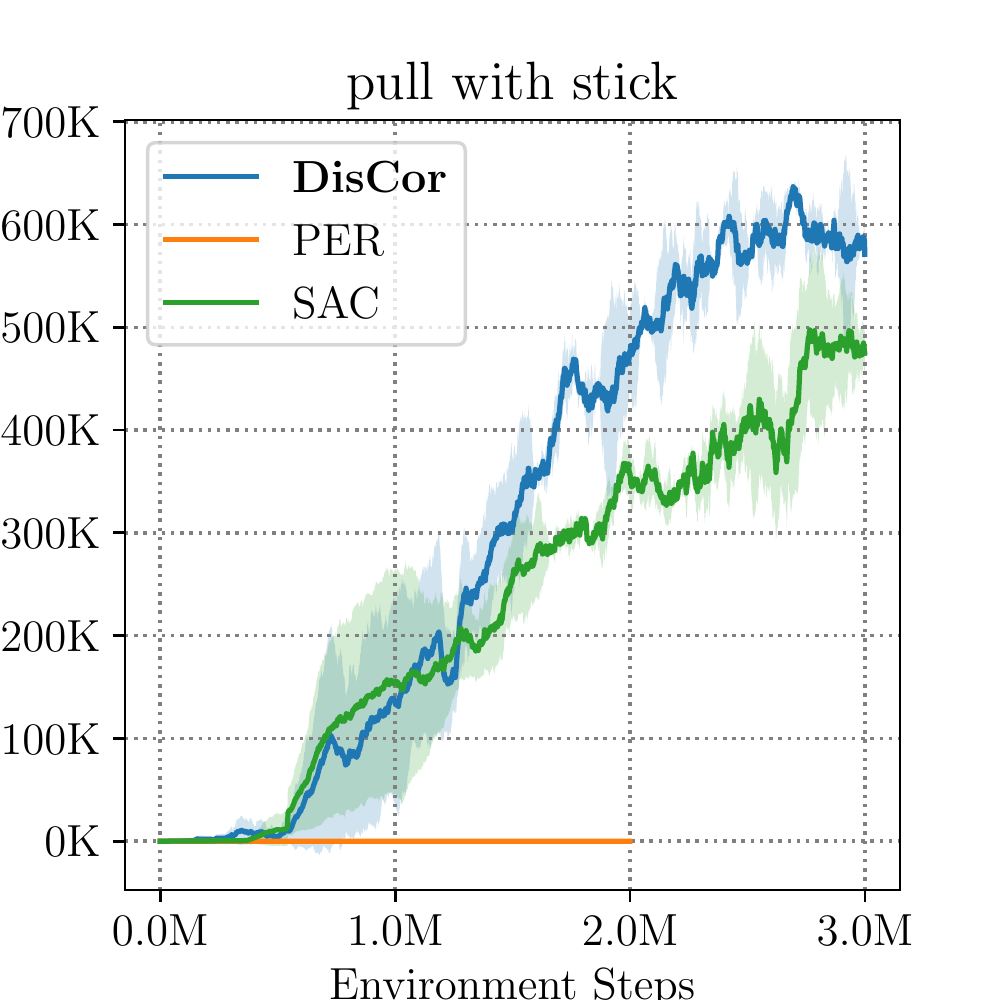}
    \includegraphics[width=0.25\linewidth]{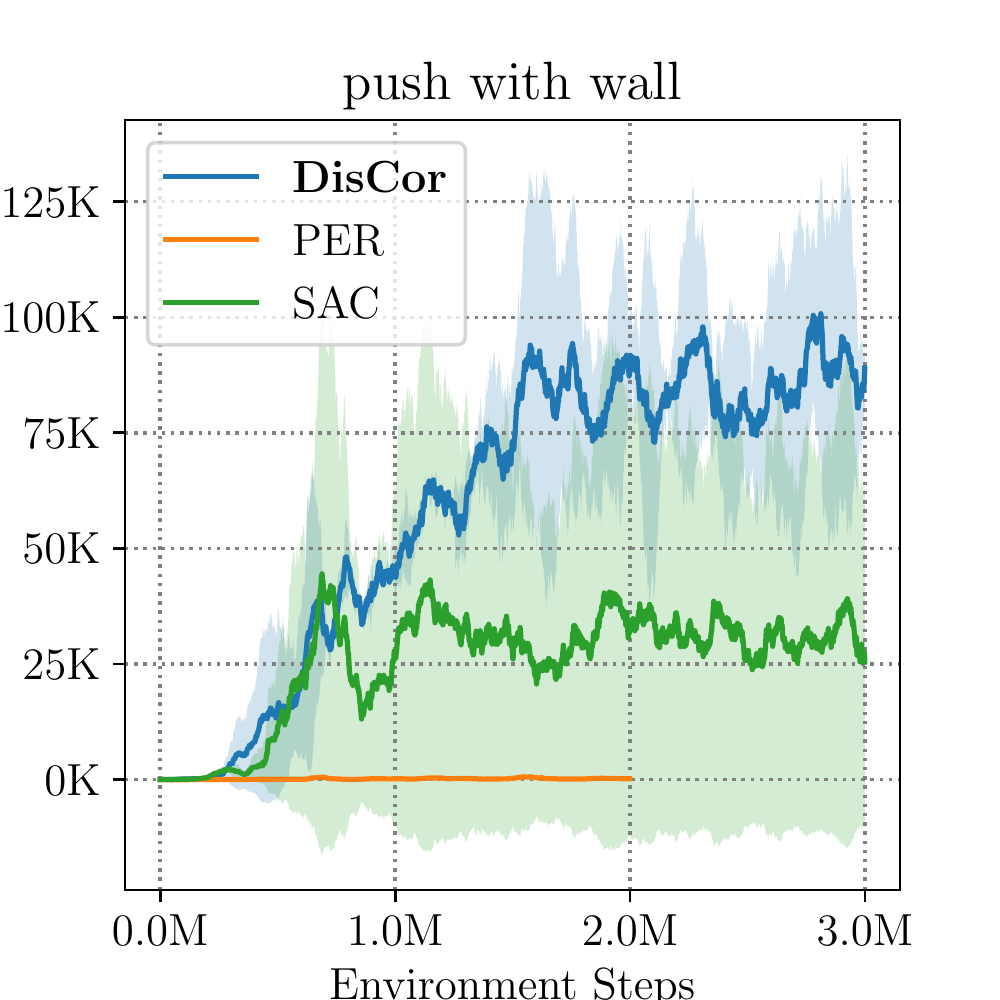}
    \includegraphics[width=0.25\linewidth]{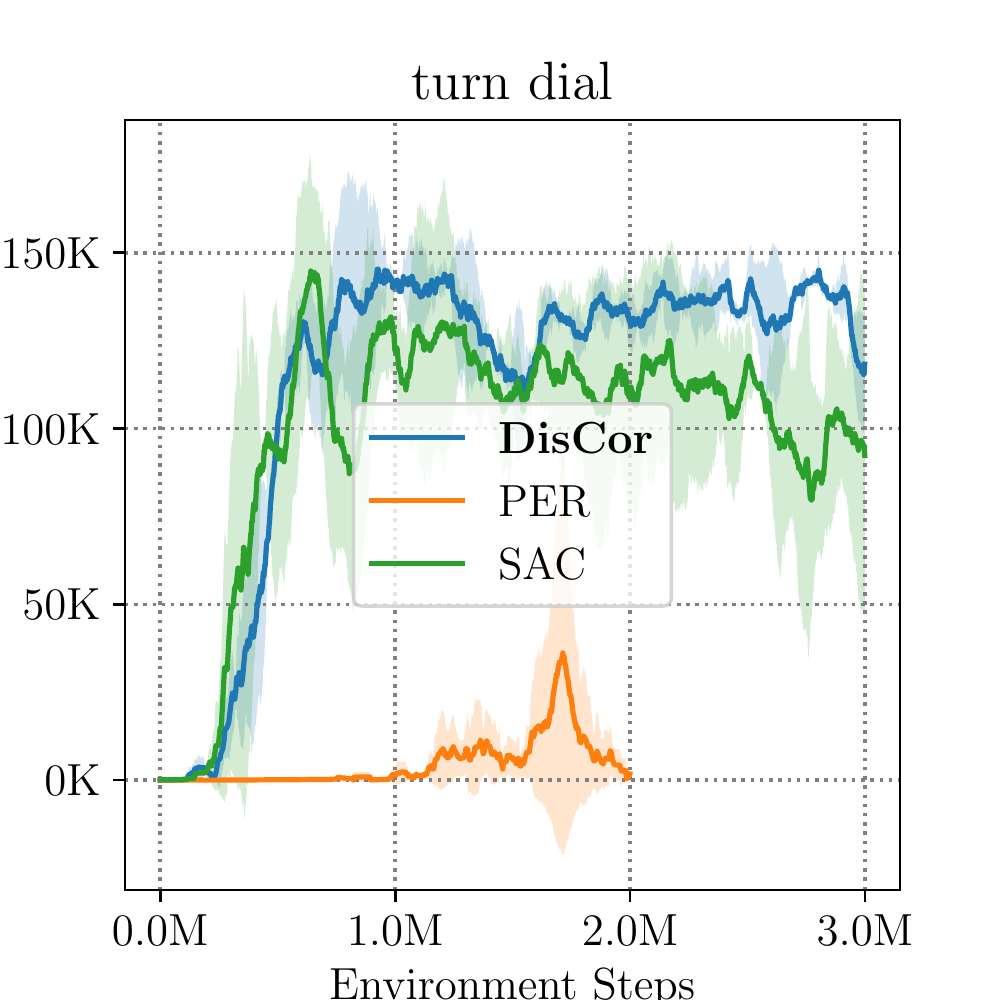} \\
    \includegraphics[width=0.25\linewidth]{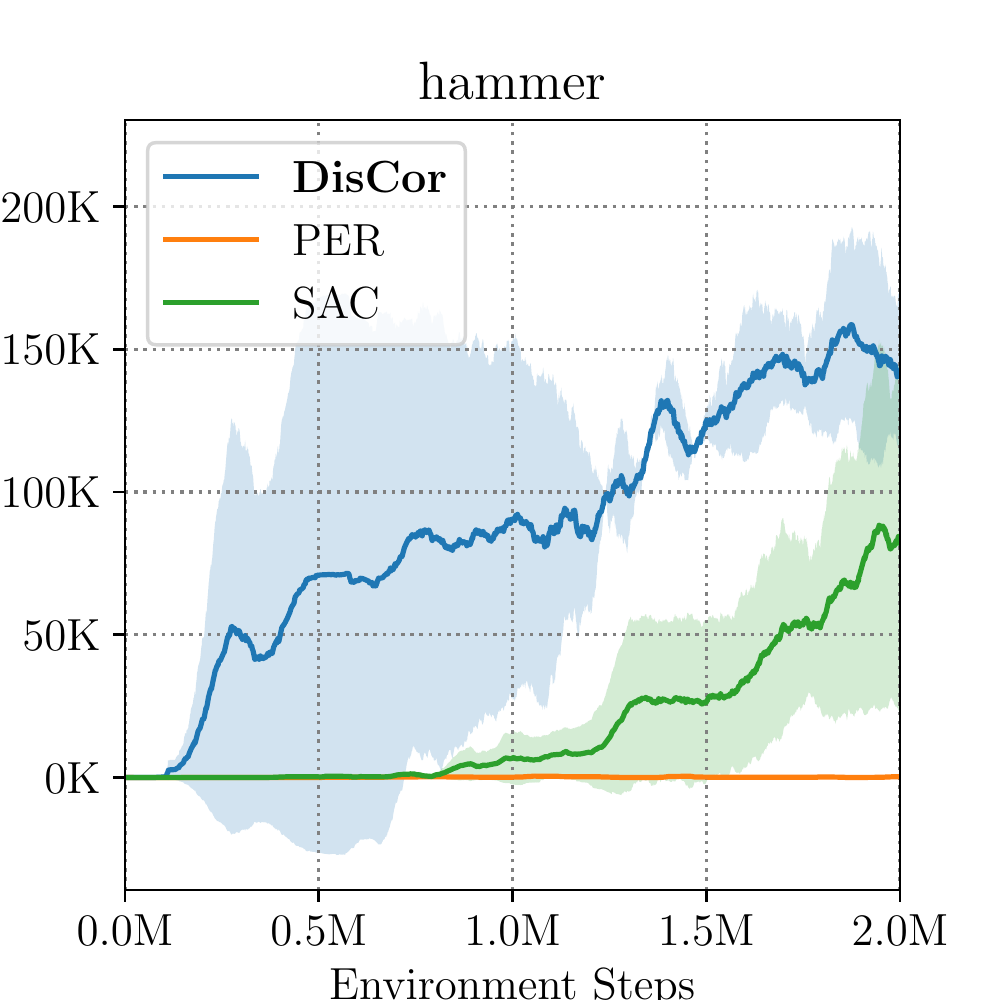}
    \includegraphics[width=0.25\linewidth]{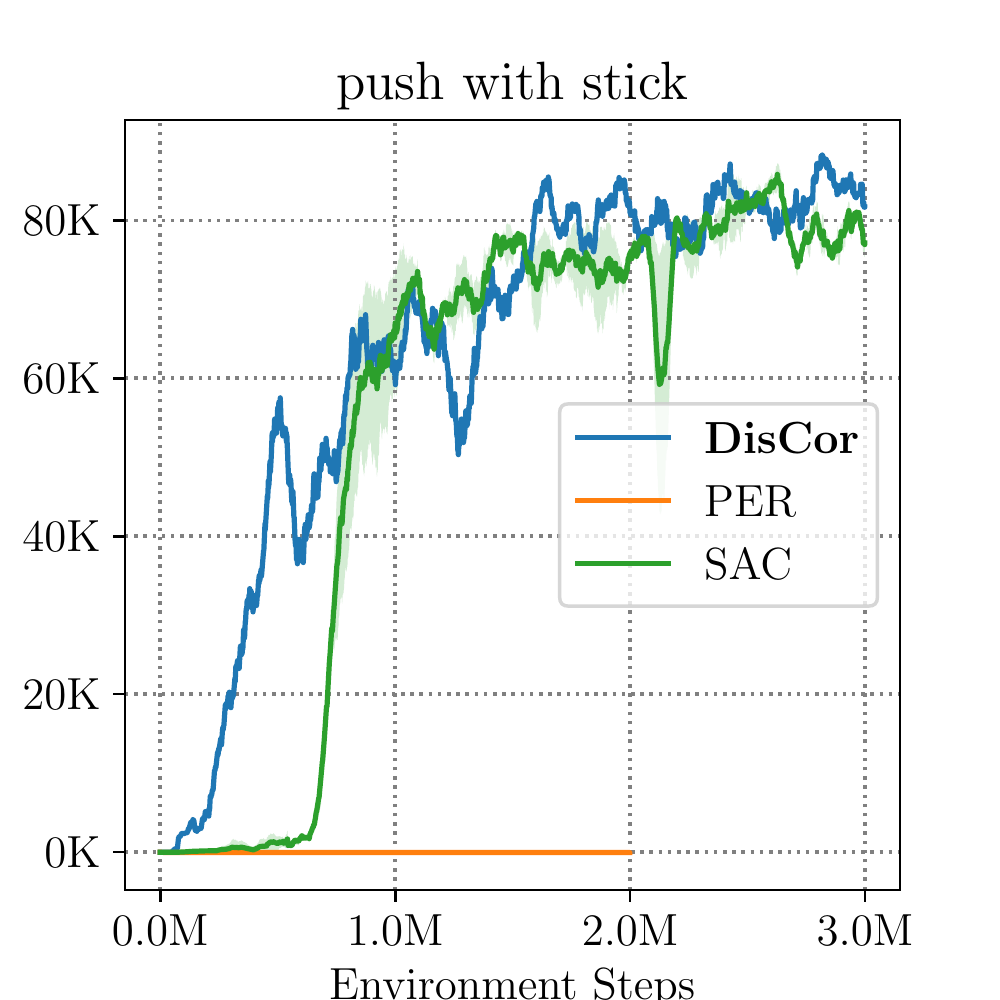}
    \includegraphics[width=0.25\linewidth]{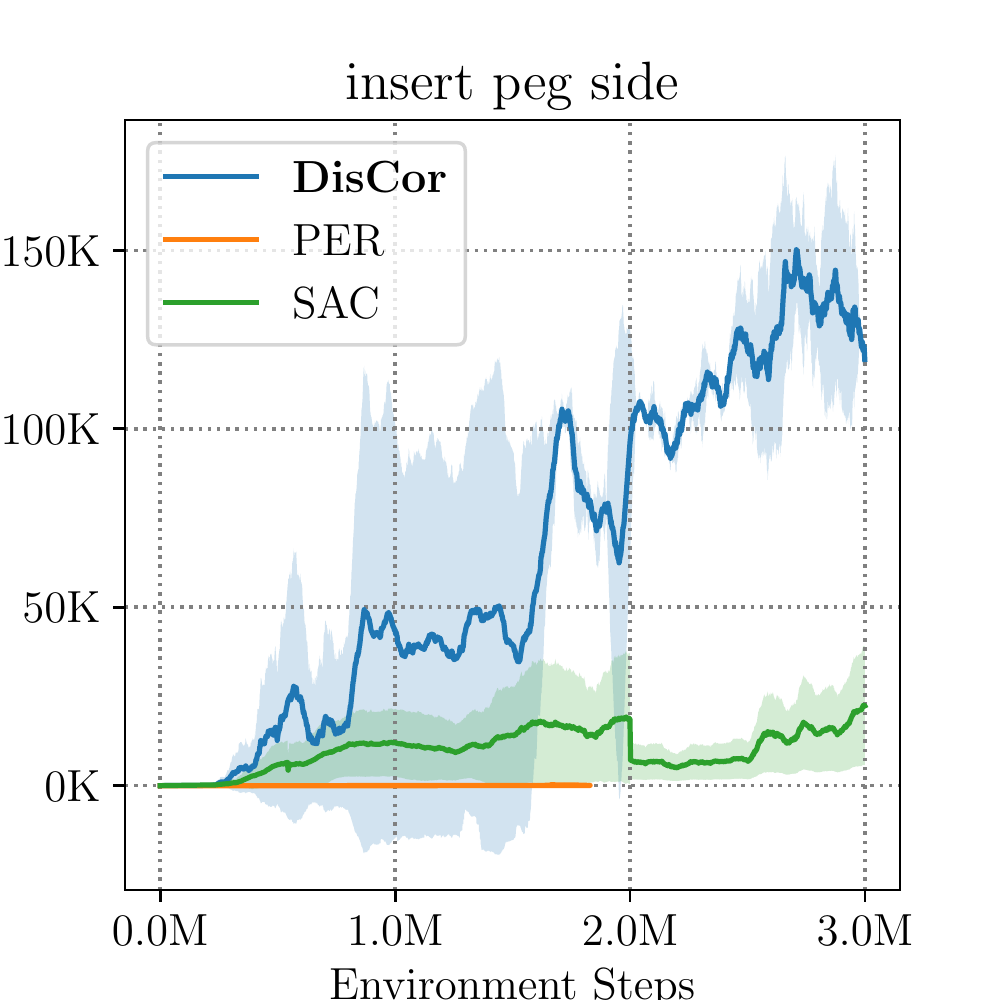}
    \caption{\footnotesize{Evaluation average return achieved by DisCor (blue), SAC (green) and PER (orange) on six Metaworld benchmarks. From left to right: pull stick, push with wall, push with stick, turn dial, hammer and insert peg side tasks. Note that DisCor clearly achieves better returns or learns faster in most of the tasks.}}
    \label{fig:app_returns_metaworld}
\end{figure*}

\subsection{OpenAI Gym Benchmarks}
\label{sec:app_mujoco_benchmarks}
Here we present an evaluation on the standard OpenAI continuous control gym benchmark environments. Modern ADP algorithms such as SAC can already solve these tasks easily, without any issues, since these algorithms have been tuned on these tasks. A comparison of the three algorithms DisCor, SAC and PER, on three of these benchmark tasks is shown in Figure~\ref{fig:app_mujoco_results}. We note that in this case, all the algorithms are roughly comparable to each other. For instance, DisCor performs better than SAC and PER on Walker2d, however, is outperformed by SAC on Ant. 

\paragraph{Stochastic reward signals.} That said, we also performed an experiment to verify the impact of stochasticity, such as noise in the reward signal, on the DisCor algorithm as compared to other baseline algorithms like SAC and PER. Analogous the diagnostic tabular experiments on low signal-to-noise ratio environments, such as those with sparse reward, we would expect a baseline ADP method to be impacted more due to an absence of corrective feedback in tasks with stochastic reward noise, since a noisy reward effectively reduces the signal-to-noise ratio. We would also expect a method that ensures corrective feedback to perform better. 

In order to test this hypothesis, we created stochastic reward tasks out of the OpenAI gym benchmarks. We modified the reward function $r(s, a)$ in these gym tasks to be equal to:
\begin{equation}
    r'(s, a) = r(s, a) + z, ~~~ z \sim \mathcal{N}(0, 1)
\end{equation}
and the agent is only provided these noisy rewards during training. However, we only report the deterministic ground-truth reward during evaluation. 
We present the results in Figure~\ref{fig:sac_noisy_results}. Observe that in this scenario, DisCor emerges as the best performing algorithm on these tasks, and outperforms other baselines SAC and PER both in terms of asymptotic performance (example, HalfCheetah) and sample efficiency (example, Ant).

\begin{figure}[H]
    \centering
    \includegraphics[width=0.25\linewidth]{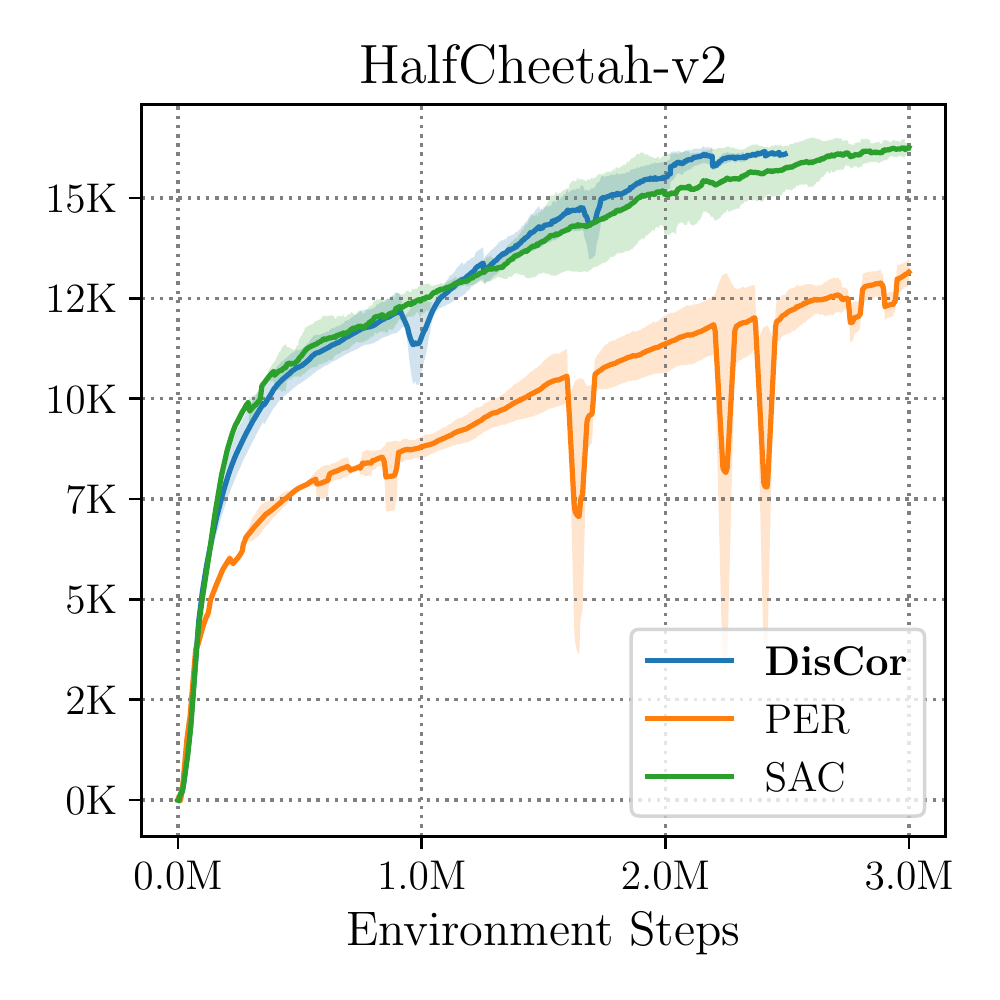}
    \includegraphics[width=0.25\linewidth]{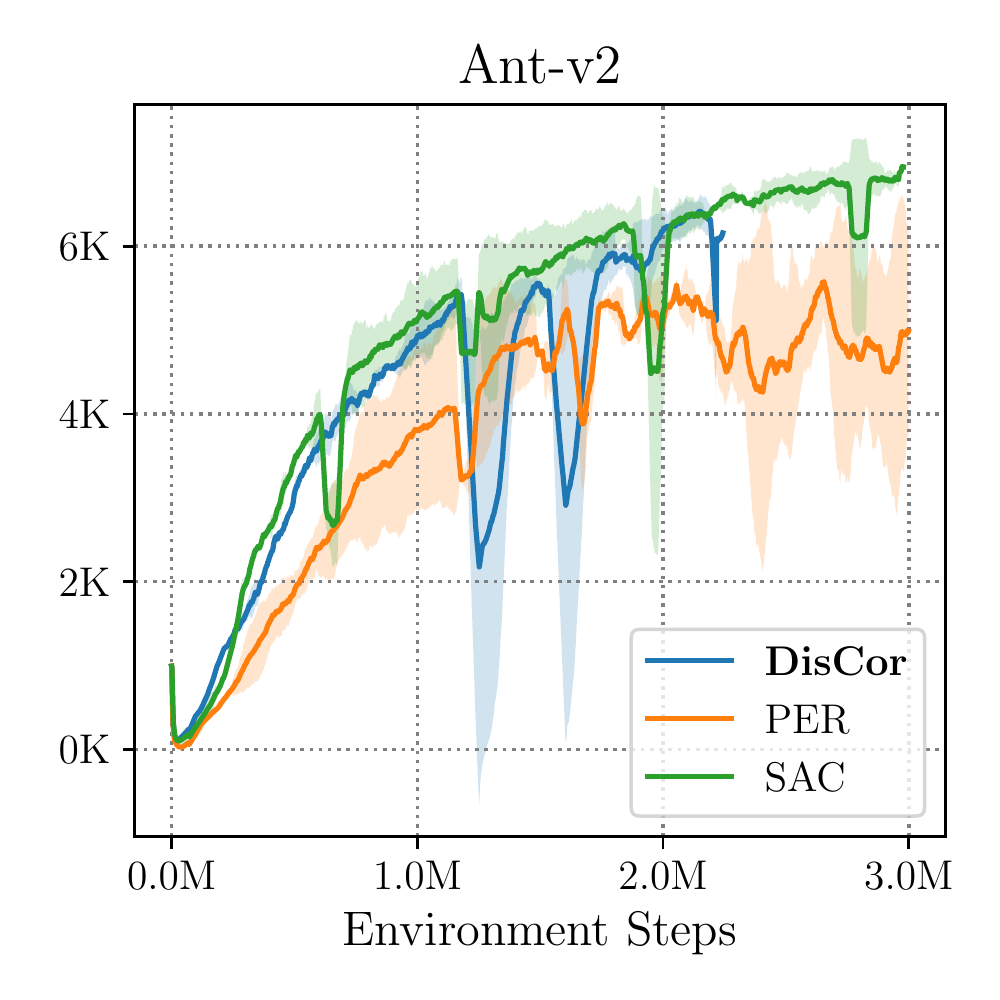}
    \includegraphics[width=0.25\linewidth]{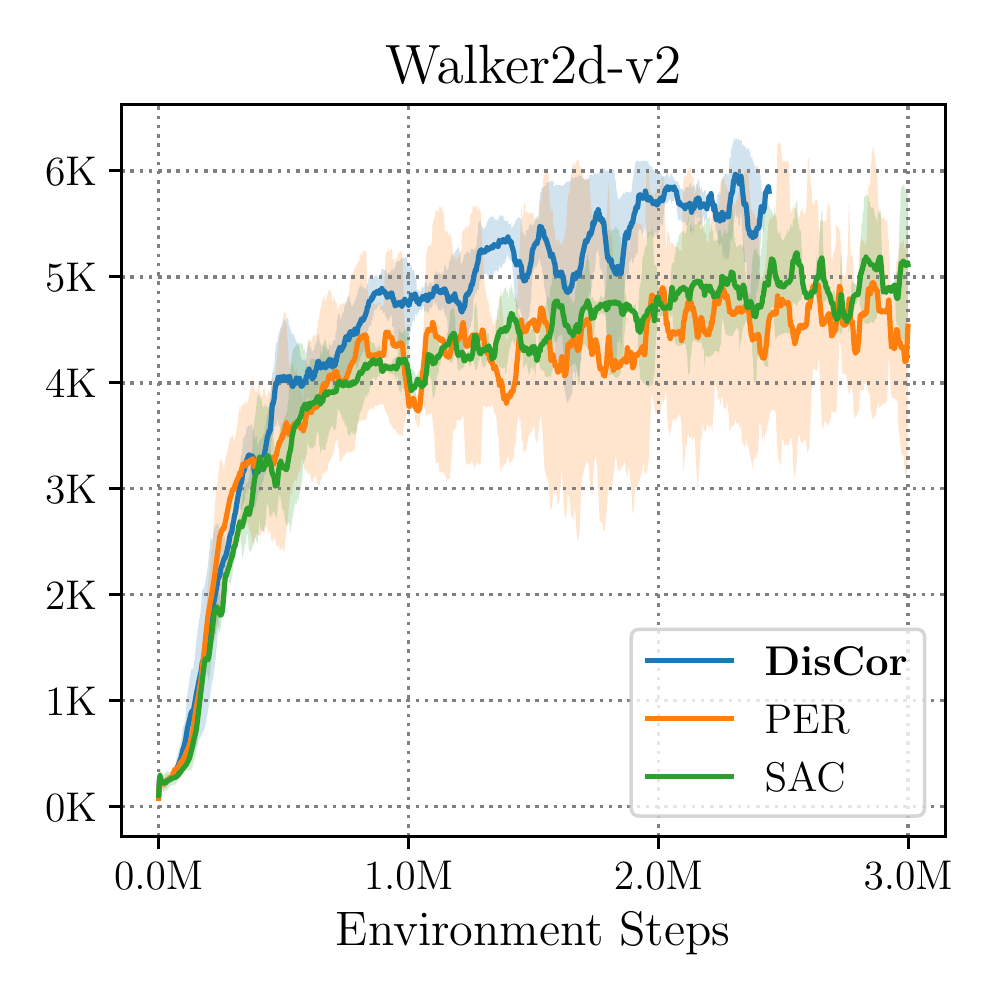}
    \caption{\footnotesize{Peformance of DisCor, SAC and PER on gym benchmarks. On an average, all methods perform roughly similarly on these settings.}}
    \label{fig:app_mujoco_results}
\end{figure}

\begin{figure}[H]
    \centering
        \includegraphics[width=0.25\linewidth]{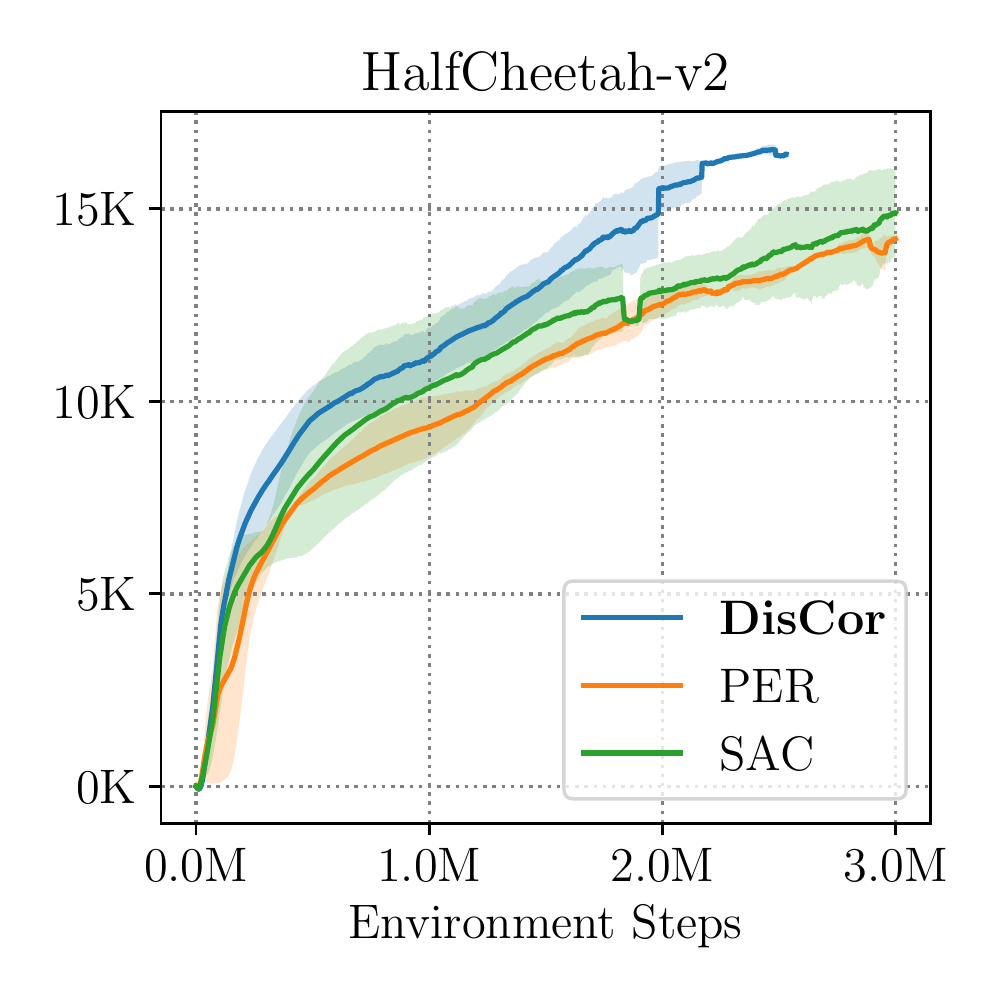}
        \includegraphics[width=0.25\linewidth]{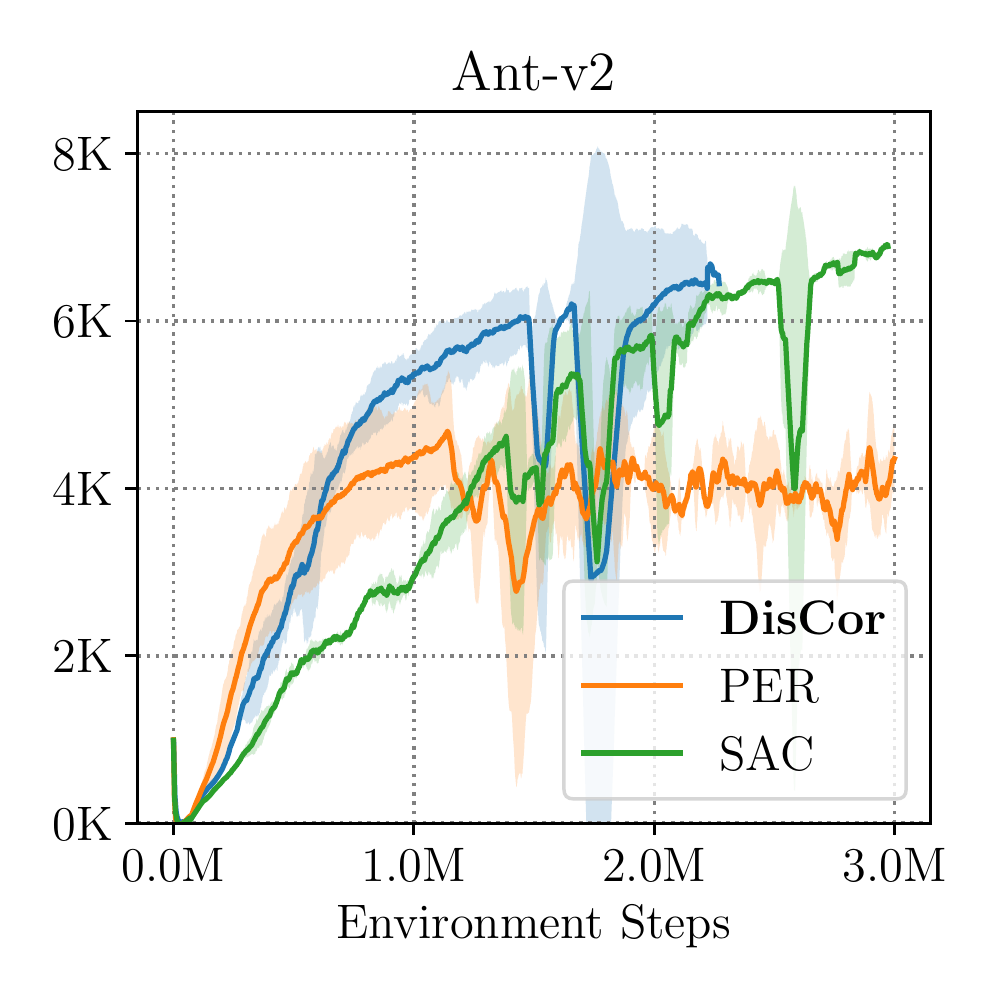}
        \includegraphics[width=0.25\linewidth]{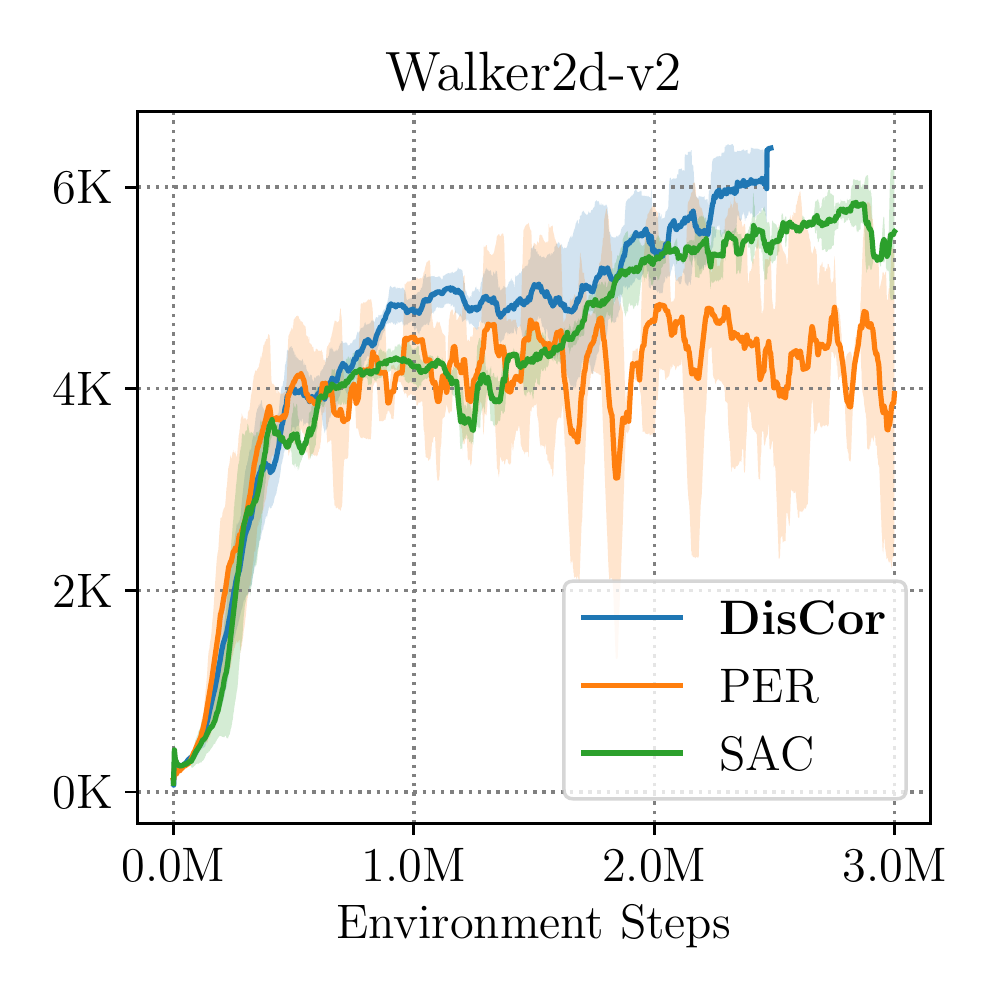}
    \caption{\footnotesize{Performance of DisCor, SAC and PER on continuous control gym benchmarks \textit{with stochastic reward noise}. Note that DisCor learns slightly faster and performs better than SAC and PER on these stochastic problems.}}
    \label{fig:sac_noisy_results}
\end{figure}

\subsection{MT10 Multi-Task Experiments}
\label{sec:app_multi_task}
In this section, we present the trend of returns, as a learning curve and as a comparative histogram (at 1M environment steps of training) for the multi-task MT10 benchmark, extending the results shown in Section~\ref{sec:multi-task}, Figure~\ref{fig:mt_results}. These plots are shown in Figure~\ref{fig:app_mt10}. Observe that DisCor achieves more than \textbf{30\%} of the return of SAC, and obtains an individually higher value of return on more tasks. 

\begin{figure}[H]
    \centering
    \begin{subfigure}[t]{.44\linewidth}
        \centering
        \includegraphics[width=0.65\linewidth]{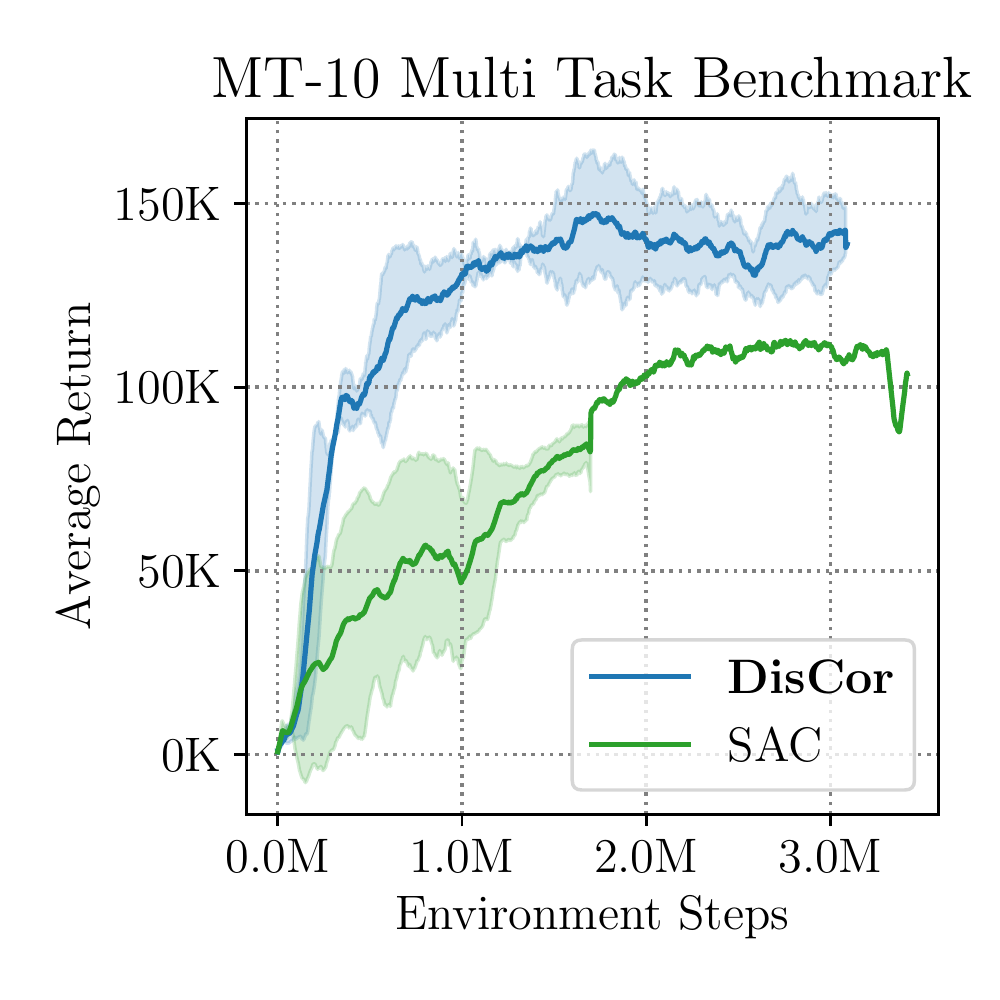}
        \caption{\footnotesize{Average task return}}
    \end{subfigure}
    ~
    \begin{subfigure}[t]{.47\linewidth}
        \centering
        \includegraphics[width=0.7\linewidth]{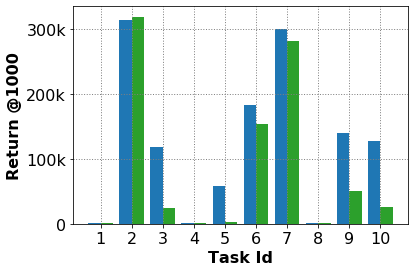}
        \caption{\footnotesize{Per-task return at 1M steps}} 
    \end{subfigure}
    \caption{\footnotesize{Performance of DisCor and SAC on the MT10 benchmark. Returns for DisCor are higher than SAC by around \textbf{30\%}; (2) DisCor achieves a non-trivial return on \textbf{7/10} tasks after 1000k steps, as compared to \textbf{3/10} for unweighted SAC, similar to the trend at 500k steps shown in Figure~\ref{fig:mt_results}.}}
    \label{fig:app_mt10}
\end{figure}

\subsection{MT50 Multi-Task Experiments}
\label{sec:app_multi_task_ablation}
We further evaluated the performance of DisCor on the multi-task MT50 benchmark~\citep{yu2019meta}. This is an extremely challenging benchmark where the task is to learn a single policy that can solve 50 tasks together, with the same evaluation protocol as previously used in the MT10 experiments (Section~\ref{sec:multi-task} and Appendix~\ref{sec:app_multi_task}). We present the results (average task return and average success rate) in Figures~\ref{fig:app_mt50}. Note that while SAC tends to saturate/plateau in between 4M - 8M steps, accounting for corrective feedback via the DisCor algorithm makes the algorithm continue learning in that scenario too.    

\begin{figure}[H]
    \centering
    \begin{subfigure}[t]{.47\linewidth}
        \centering
        \includegraphics[width=0.65\linewidth]{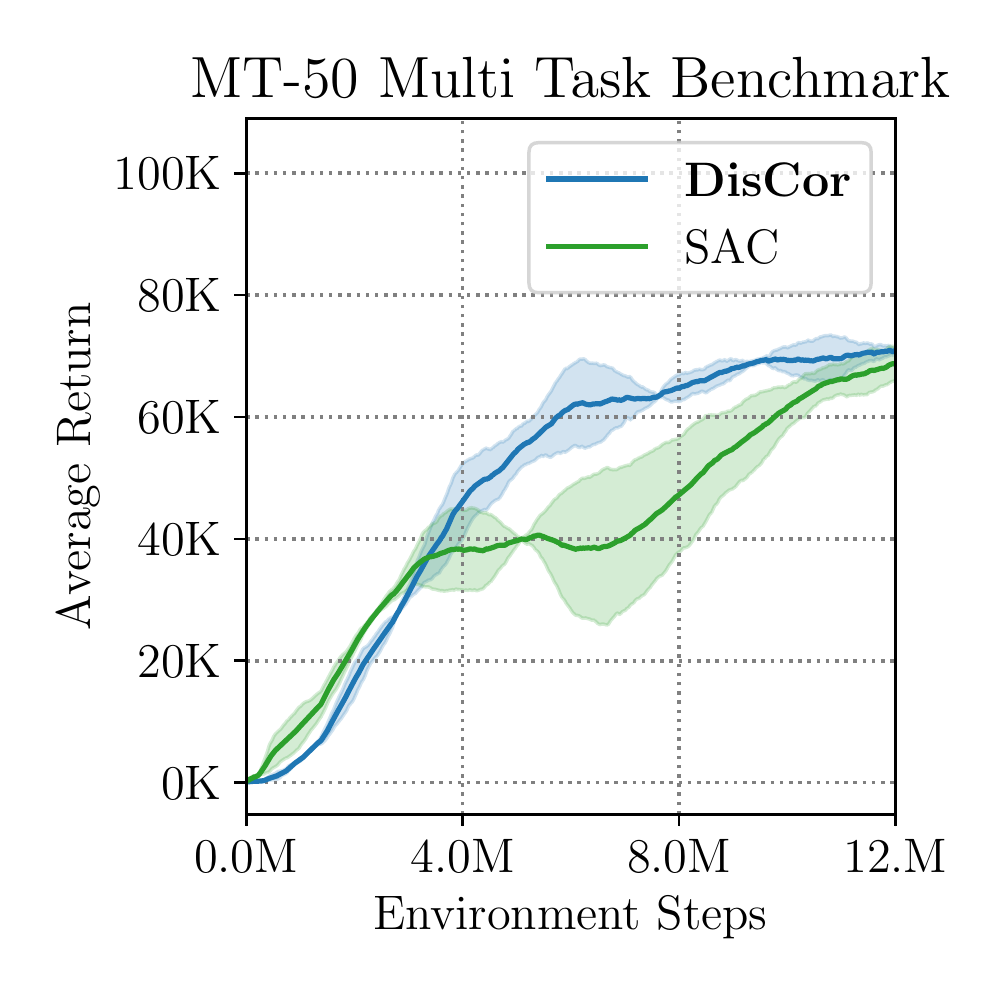}
        \caption{\footnotesize{Average task return}}
    \end{subfigure}
    ~
    \begin{subfigure}[t]{.47\linewidth}
        \centering
        \includegraphics[width=0.65\linewidth]{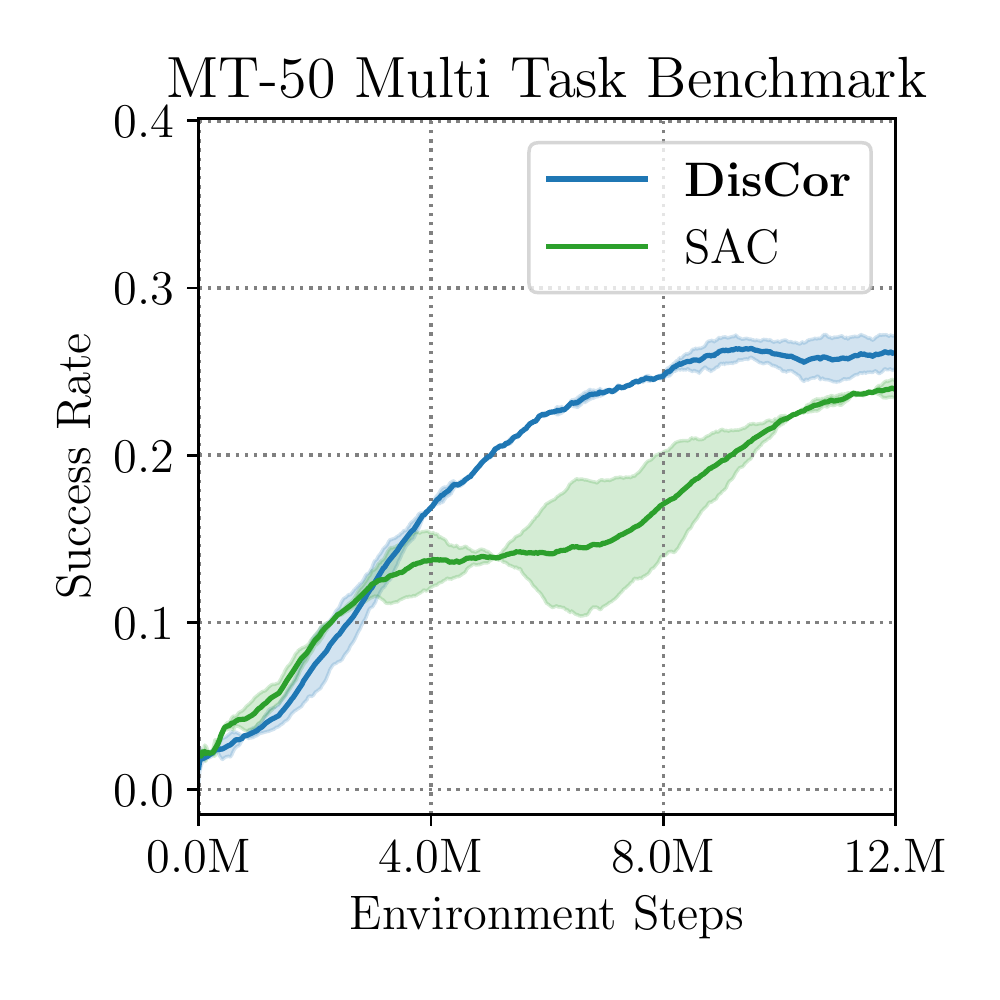}
        \caption{\footnotesize{Average success rate}} 
    \end{subfigure}
    \caption{\footnotesize{Performance of DisCor and SAC on the MT50 benchmark. Note that, DisCor clearly keeps learning unlike SAC which tends to plateau for about 3M steps in the middle (the stretch between 4M and 7M steps on the x-axis, where SAC exhibits a small gradient in the learning progress, whereas DisCor continuously keeps learning).}}
    \label{fig:app_mt50}
\end{figure}
\subsection{Comparison with AFM}
\label{sec:afm_comparison}
In this section, we present a comparison of DisCor and AFM~\citep{fu19diagnosing}, a prior method similar to prioritized experience replay on the MuJoCo gym benchmarks. We find that DisCor clearly outperforms AFM in these scenarios. We present these results in Figure~\ref{fig:app_discor_vs_afm} where the top row presents results in the case of regular gym benchmarks, and the bottom row presents results in the case of gym benchmarks with stochastic reward noise.

\subsection{DQN with multi-step returns}
\label{sec:dqn_multi_step}
N-step returns with DQN are hypothesized to stabilize learning since updates to the Q-function now depends on reward values spanning multiple steps, and the coefficient of the bootstrapped Q-value is $\gamma^T$, which is exponentially smaller than $\gamma$ used conventionally in Bellman backups, implying that the error accumulation process due to incorrect targets is reduced. Thus, we perform a comparison of DisCor and DQN with n-step backups, where $n$ was chosen to be $3$, $n=3$, in accordance with commonly used multi-step return settings for Atari games. We present the average return obtained by DisCor and DQN (+n-step), with sticky actions, in Table~\ref{table:dqn_discor_nstep}. We clearly observe that DisCor outperforms DQN with 3-step returns in all three games evaluated on. We also observe that n-step returns applied with DisCor also outperform n-step returns applied with DQN, indicating the benefits of using DisCor even when other techniques, such as n-step returns are used.

\begin{table}[t]
\centering
\label{table:atari_multistep}
\begin{tabular}{c r r r}
\hline
\textbf{Game} & \textbf{n-step DQN} & \textbf{DisCor} & \textbf{n-step DisCor}\\
 & $(n=3)$ & (Regular) & $(n=3)$ \\
\hline
\hline
\textbf{Pong} & 17 & {17} & \colorbox{blue!30}{\textbf{19}} \\
\textbf{Breakout} & 37 & \textbf{175} & \colorbox{blue!30}{47} \\
\hline
\end{tabular}
\caption{\footnotesize{Average Performance of DQN + 3-step returns, DisCor and Discor + 3-step returns on Pong and Breakout at 60M steps into training, rounded off to the nearest integer. Note that DisCor clearly outperforms DQN with multi-step returns. We also find that adding n-step returns to DisCor can hurt, for instance, on Breakout, where the same hurts with DQN as well (for comparison, see Figure~\ref{fig:atari_results} in the main paper), however, we still observe that DisCor, when applied with multi-step returns performs better than DQN with multi-step returns as well, indicating the benefits of DisCor even when methods such as multi-step returns are used.}} 
\label{table:dqn_discor_nstep}
\end{table}

\begin{figure}[H]
    \centering
        \includegraphics[width=0.25\linewidth]{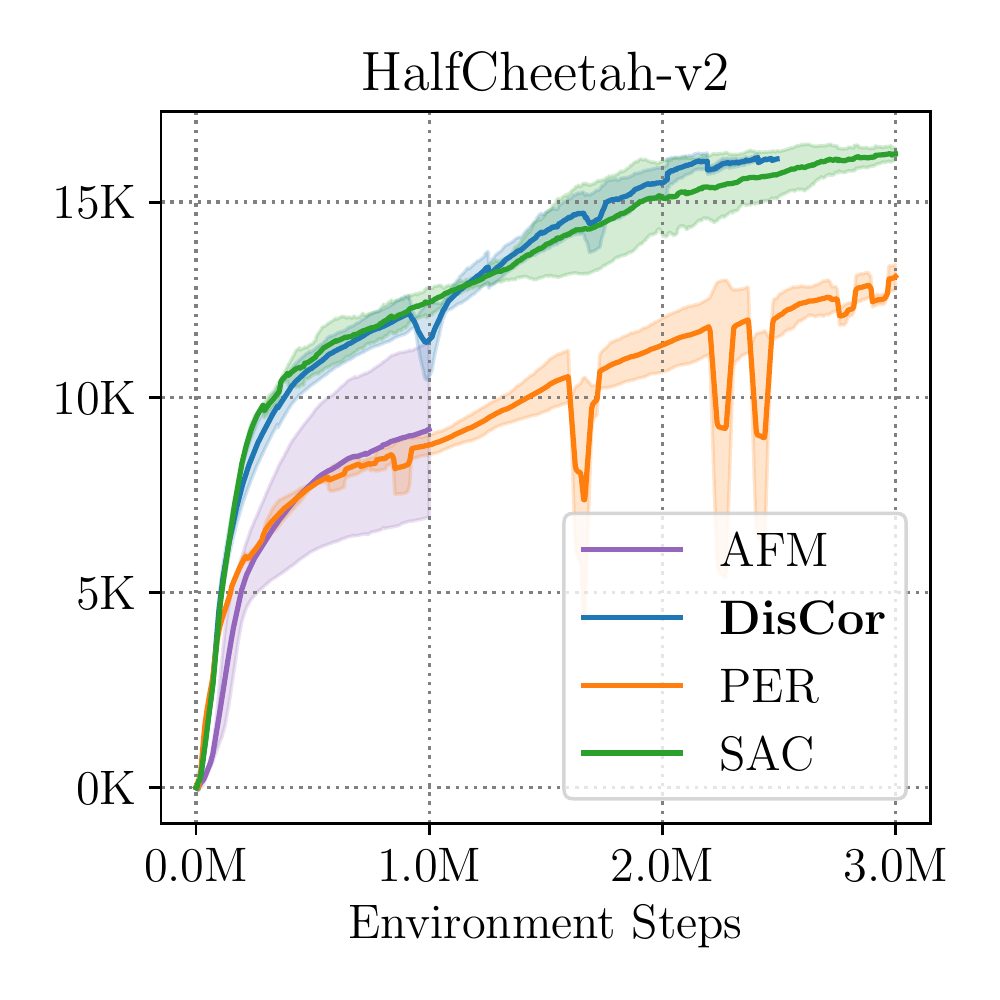}
        \includegraphics[width=0.25\linewidth]{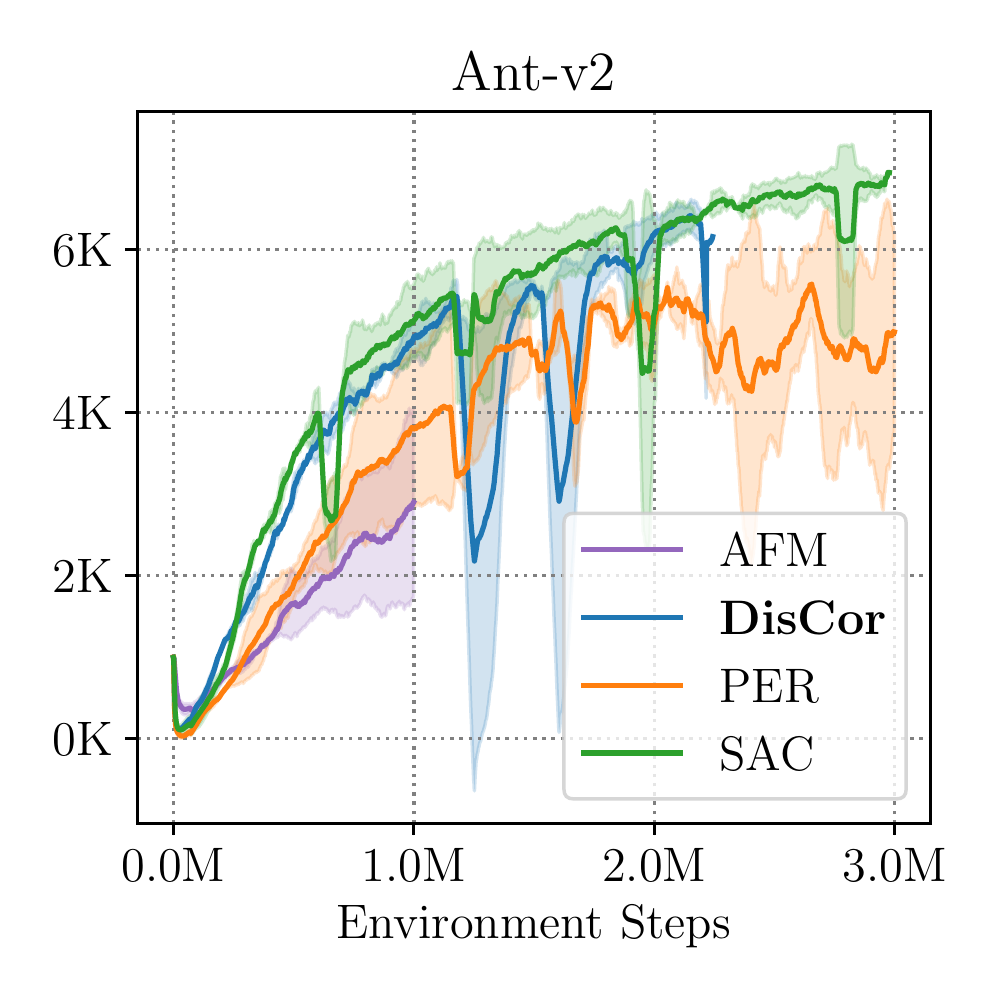}
        \includegraphics[width=0.25\linewidth]{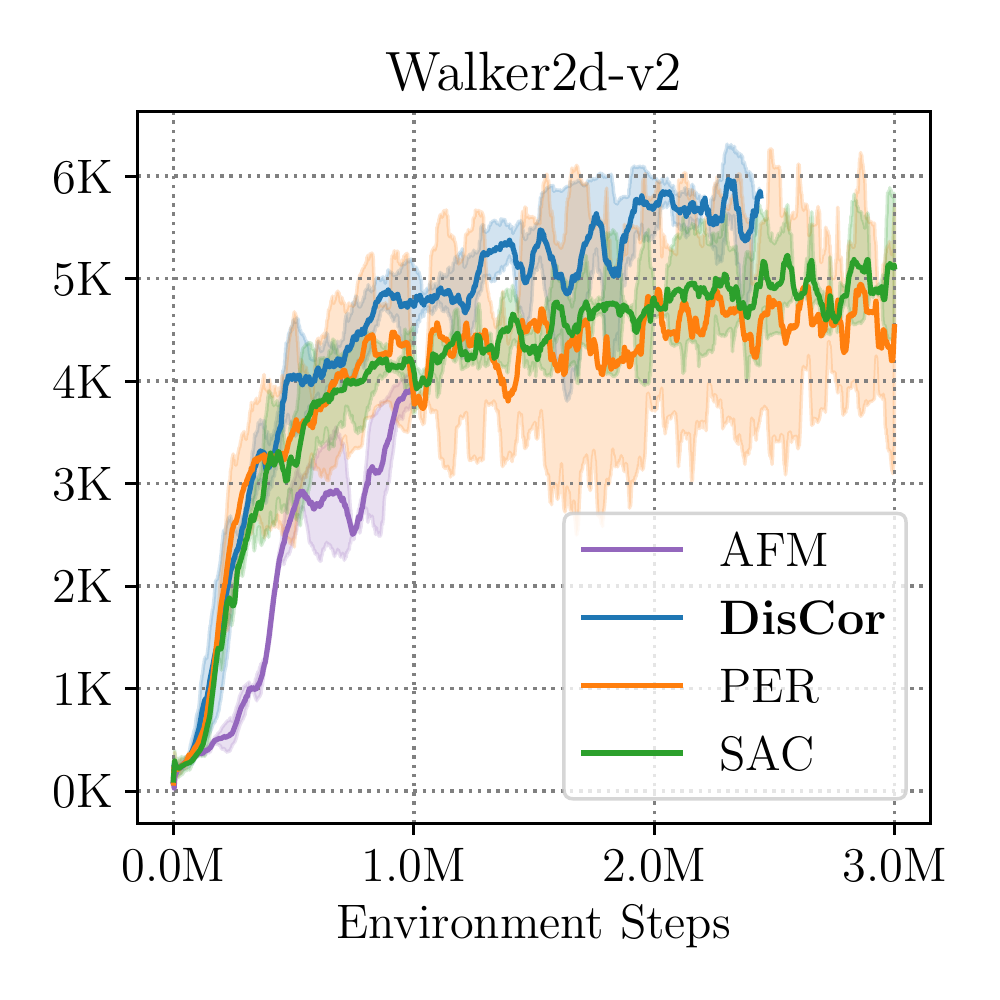}\\
        \vspace{4pt}
        \includegraphics[width=0.25\linewidth]{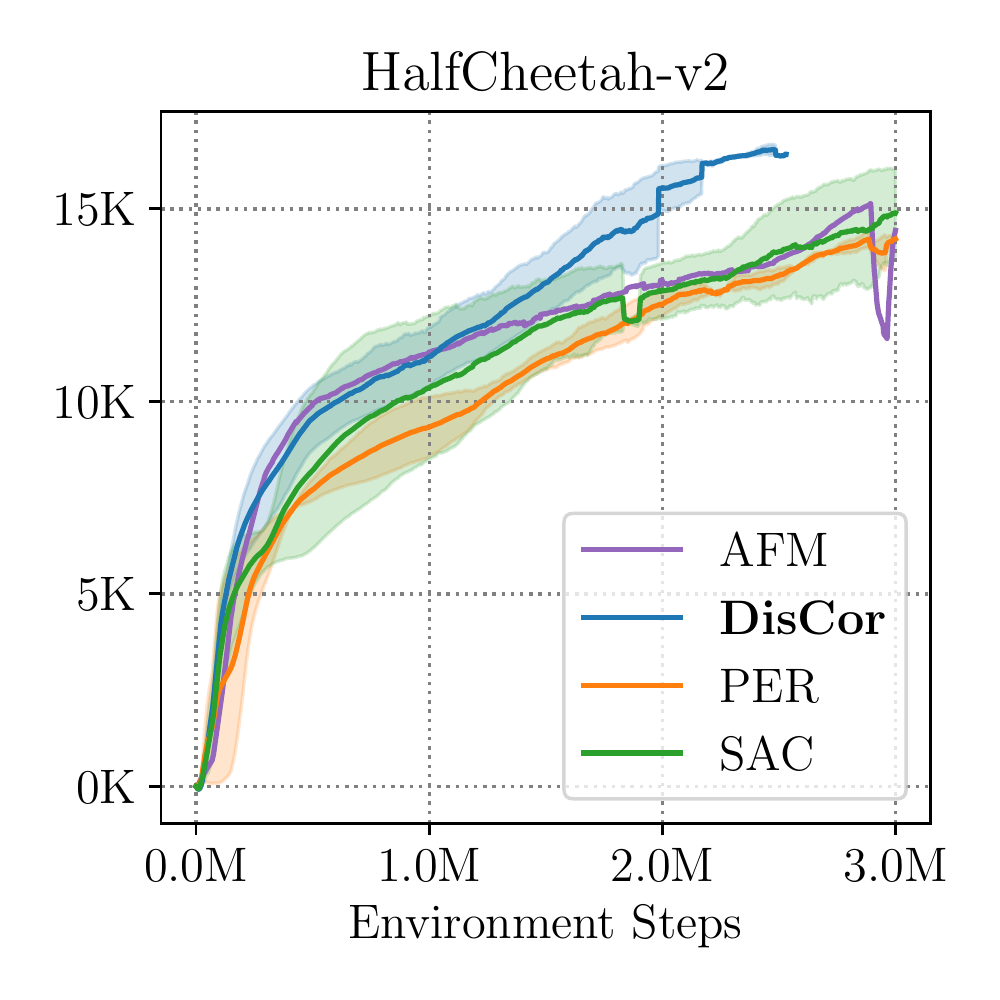}
        \includegraphics[width=0.25\linewidth]{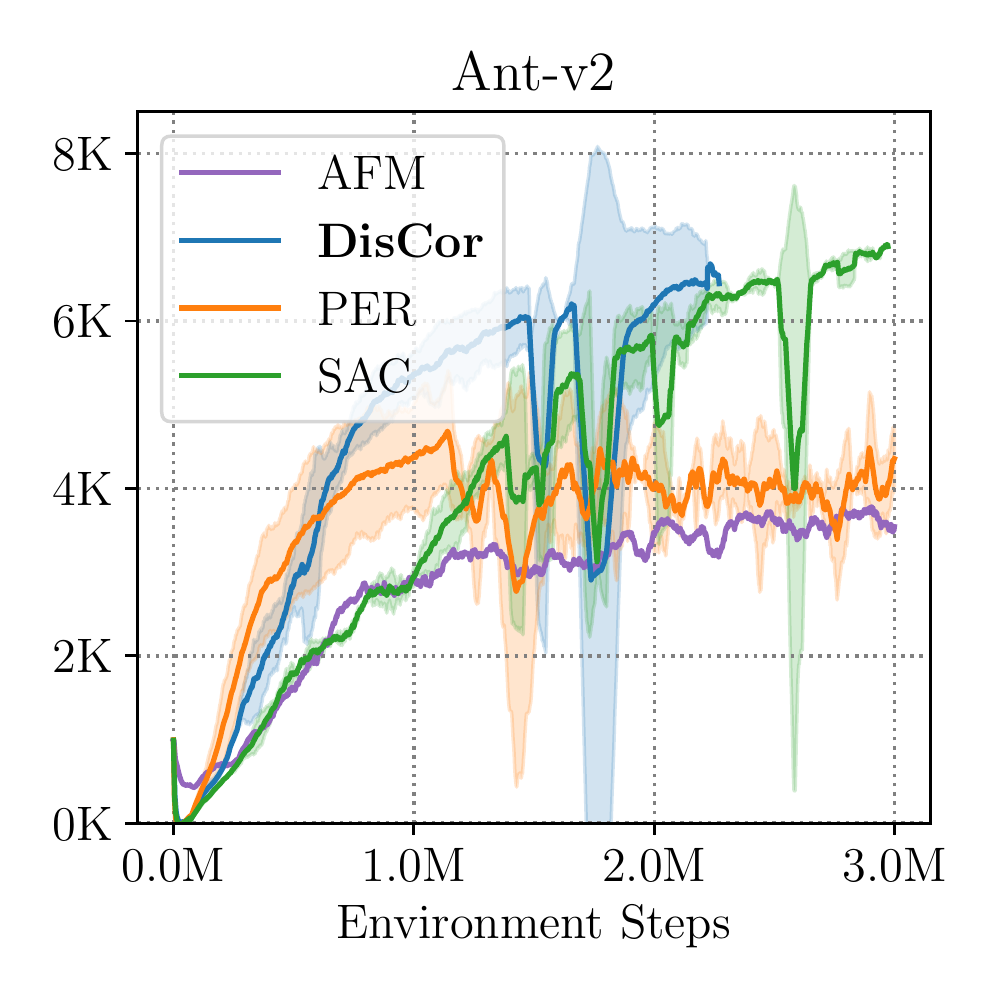}
        \includegraphics[width=0.25\linewidth]{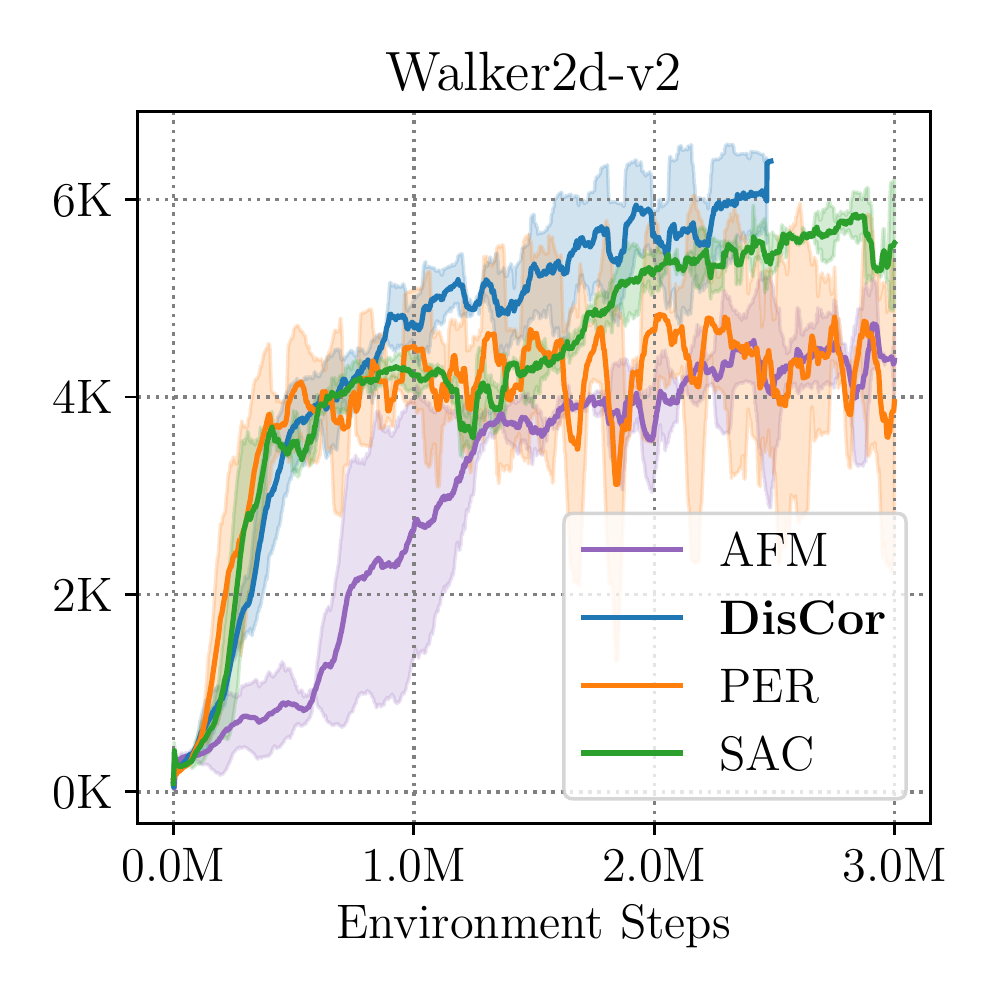}
    \caption{\footnotesize{Performance of DisCor, SAC, PER and AFM on continuous control gym benchmarks (top row) and gym benchmarks \textit{with stochastic reward noise} (bottom row). Note that DisCor clearly out-performs AFM in both scenarios, on all three benchmarks tested on.}}
    \label{fig:app_discor_vs_afm}
\end{figure}

\begin{figure*}[t!]
\centering
    \begin{lstlisting}[language=Python]
    def _init_critic_update_with_dist(self):
        """Update critic with distribution weighting, 
           and update \delta_\phi using recursive update. """
        next_actions = self._policy.actions([self._next_observations_ph])
        
        ## Compute errors at next state, and an action from the policy
        qf_pred_errs = self._error_fns([self._next_observations_ph, next_actions])
            
        ## error_model_tau_ph: moving mean of the error values over batches
        err_logits = -tf.stop_gradient(
            self._discount * qf_pred_errs / self._error_model_tau_ph)

        Q_target = tf.stop_gradient(self._get_Q_target())
        Q_values = self._Q([self._observations_ph, self._actions_ph])
        
        ## Compute importance sampled loss, also perform self-normalized sampling
        loss, weights = importance_sampled_loss(
            labels=Q_target, predictions=Q_values,
            weights=err_logits, weight_options='self_normalized')
        
        ## Train Q-function
        Q_training_ops = tf.contrib.layers.optimize_loss(loss, learning_rate=self._Q_lr,
            optimizer=self._Q_optimizer, variables=self._Q.trainable_variables)
        training_ops.update({'Q': tf.group(Q_training_ops)})
        
        ## Training the error function
        err_values = self._error_fns([self._observations_ph, self._actions_ph])
        
        ## Mean Bellman error used to compute target values for error
        bellman_errors = tf.abs(Q_values - Q_target)
        err_targets = tf.stop_gradient(self._get_error_target(bellman_errors))
        
        ## This is used to update the moving mean, self._error_model_tau_ph
        self._mean_error_values = tf.reduce_mean(err_values)

        ## Simple mean squared error loss for \delta_\phi
        err_losses =  tf.losses.mean_squared_error(
            labels=err_targets, predictions=err_values, weights=0.5)
        
        ## Update error function: \delta_\phi
        err_training_ops = tf.contrib.layers.optimize_loss(err_losses,
            learning_rate=self._dist_lr, 
            optimizer=self._err_optimizer, variables=self._error_fns.trainable_variables)
        training_ops.update({'Error': tf.group(err_training_ops)})
    \end{lstlisting}
    \caption{\footnotesize{Code for training the error function $\Delta_\phi$, and modified training for the Q-function $Q(s, a)$ using $\Delta_\phi$ to get weights $w(s, a)$ for training. Code written in convention with regular Tensorflow guidelines, in the same style as the official SAC implementation~\citep{haarnoja2018sacapps}.}}
\label{fig:code_discor}
\end{figure*}

\subsection{Code for the Method}
\label{sec:code}
The code is shown in Figure~\ref{fig:code_discor}. It is a simplified version of the code from our implementation of DisCor on top of the official SAC repository~\citep{haarnoja2018sacapps}.

\end{document}